\documentclass[twoside,11pt]{article}

\usepackage[margin=1.05in]{geometry}
\usepackage{microtype}

\usepackage[round]{natbib}

\newcommand{\revision}[1]{#1}

\renewcommand{\cite}{\citep}  %

\usepackage{amsmath,amssymb,pifont}
\usepackage{multicol}
\usepackage{amstext}
\usepackage{amsthm}
\usepackage{multirow}
\usepackage{booktabs}
\usepackage{adjustbox}
\usepackage[skip=0pt]{subcaption}
\usepackage{lipsum}
\usepackage[shortlabels]{enumitem}
\usepackage{cancel}
\usepackage{wrapfig}
\usepackage{array}
\usepackage{siunitx}
\usepackage{csvsimple}
\usepackage[multidot]{grffile}
\usepackage{bbm}
\usepackage{dblfloatfix}
\usepackage[colorlinks = true,
            linkcolor = blue,
            urlcolor  = blue,
            citecolor = blue,
            anchorcolor = blue]{hyperref}
\usepackage{makecell}
\usepackage{bbm, dsfont}
\usepackage{mathtools}
\usepackage{comment}

\usepackage{geometry}
\usepackage{enumitem}

\usepackage{minitoc} %

\usepackage[multiple]{footmisc}
\usepackage{mathrsfs}
\usepackage{todonotes}
\usepackage{tikz}
\usepackage{cleveref}

\usepackage[noend]{algpseudocode}
\usepackage{algorithm,algorithmicx}
\algrenewcommand\algorithmicrequire{\textbf{Input:}}
\algrenewcommand\algorithmicensure{\textbf{Return}}
\usepackage{xfrac}

\usepackage{graphicx} %
\usepackage{color}

\usepackage{array}
\usepackage{amssymb}
\usepackage{amsmath}
\usepackage{xspace}
\usepackage{fancyhdr}
\usepackage{comment}
\usepackage{bm}

\newcommand{\No}{-}

\newcommand{\noisysgd}{Noisy-SGD\xspace}
\newcommand{\noisyftrl}{Noisy-FTRL\xspace}
\newcommand{\ournoisyftrl}{$\nu$-Noisy-FTRL\xspace}

\newcommand{\kappafunc}{{\psi}}

\newcommand{\optimallti}{\ourprivftrl}  %

\newcommand{\dpftrl}{DP-FTRL\xspace}  %
\newcommand{\dpmf}{\dpftrl}  %
\newcommand{\hed}{Half-Expo Decay\xspace}
\newcommand{\failprob}{p}
\newcommand{\Fobj}{F_{\infty}}

\newcommand{\smooth}{L}
\newcommand{\strong}{\mu}
\newcommand{\algname}{\dpftrl}

\newcommand{\privftrl}{DP-FTRL\xspace}
\newcommand{\ourprivftrl}{$\nu$-DP-FTRL\xspace}

\newcommand{\privsgd}{DP-SGD\xspace}

\newcommand{\Exps}{\mathbb{E}}
\newcommand{\ExP}[2]{\Exps_{#1}\left[{#2}\right]}
\newcommand{\data}{{\sf {data}}}
\newcommand{\Pdata}{\mathbb{P}_\data}
\newcommand{\herm}[1]{{#1}^\ast}
\newcommand{\br}[1]{\left({#1}\right)}

\newcommand{\tran}[1]{{#1}^\top}
\newcommand{\T}{^\top}
\newcommand{\multH}{\Lambda}
\newcommand{\multh}{\lambda}

\newcommand{\note}[1]{\textcolor{red}{#1}}

\newcommand{\reg}{r}
\newcommand{\strongconvex}{\mu}
\newcommand{\smoothconvex}{L}

\newcommand\numberthis{\addtocounter{equation}{1}\tag{\theequation}}
\hypersetup{final}

\newcommand{\eps}{\ensuremath{\varepsilon}}

\newcommand{\boldzero}{\ensuremath{\boldsymbol{0}}}

\newcommand{\lambdat}{\tilde{\lambda}}
\newcommand{\Ssgd}{S_{\text{sgd}}}
\newcommand{\id}[1]{\bfI_{#1}}
\newcommand{\bfmulth}{{\bm{\multh}}}
\newcommand{\bfA}{{\bm{A}}}
\newcommand{\bfB}{{\bm{B}}}
\newcommand{\bfC}{{\bm{C}}}

\newcommand{\bfG}{{\bm{G}}}
\newcommand{\bfH}{{\bm{H}}}
\newcommand{\bfI}{{\bm{I}}}

\newcommand{\bfM}{{\bm{M}}}

\newcommand{\bfP}{{\bm{P}}}
\newcommand{\bfQ}{{\bm{Q}}}
\newcommand{\bfR}{{\bm{R}}}
\newcommand{\bfS}{{\bm{S}}}
\newcommand{\bfT}{{\bm{T}}}
\newcommand{\bfU}{{\bm{U}}}
\newcommand{\bfV}{{\bm{V}}}
\newcommand{\bfW}{{\bm{W}}}
\newcommand{\bfX}{{\bm{X}}}
\newcommand{\bfY}{{\bm{Y}}}
\newcommand{\bfZ}{{\bm{Z}}}

\newcommand{\bfe}{{\bm{e}}}

\newcommand{\bfg}{{\bm{g}}}
\newcommand{\bfh}{{\bm{h}}}
\newcommand{\bfi}{{\bm{i}}}

\newcommand{\bfu}{{\bm{u}}}
\newcommand{\bfv}{{\bm{v}}}
\newcommand{\bfw}{{\bm{w}}}
\newcommand{\bfx}{{\bm{x}}}
\newcommand{\bfy}{{\bm{y}}}
\newcommand{\bfz}{{\bm{z}}}

\newcommand{\calA}{\ensuremath{\mathcal{A}}}

\newcommand{\calD}{\ensuremath{\mathcal{D}}}
\newcommand{\calE}{\ensuremath{\mathcal{E}}}
\newcommand{\calF}{\ensuremath{\mathcal{F}}}

\newcommand{\calI}{\ensuremath{\mathcal{I}}}

\newcommand{\calM}{\ensuremath{\mathcal{M}}}
\newcommand{\calN}{\ensuremath{\mathcal{N}}}

\newcommand{\calP}{\ensuremath{\mathcal{P}}}

\newcommand{\calT}{\ensuremath{\mathcal{T}}}

\newcommand{\calX}{\ensuremath{\mathcal{X}}}
\newcommand{\calY}{\ensuremath{\mathcal{Y}}}

\newcommand{\bfSigma}{\ensuremath{\bm{\varSigma}}}
\newcommand{\bftheta}{\ensuremath{\bm{\theta}}}

\newcommand{\bfdelta}{\ensuremath{\bm{\delta}}}

\newcommand{\bfbeta}{\ensuremath{{\bm{\beta}}}}
\newcommand{\bfLambda}{\ensuremath{\bm{\varLambda}}}

\newcommand{\inner}[2]{\left\langle {#1}, {#2}\right\rangle}

\newcommand{\bfphi}{\ensuremath{\bm{\phi}}}
\newcommand{\bfxi}{\ensuremath{\bm{\xi}}}
\newcommand{\bfzeta}{\ensuremath{\bm{\zeta}}}

\newcommand{\polylog}[1]{\ensuremath{{\rm polylog}\left(#1\right)}}

\newcommand{\R}{\mathbb{R}}
\newcommand{\reals}{\mathbb{R}}

\newcommand{\C}{\mathbb{C}}
\newcommand{\conj}[1]{\overline{#1}}
\newcommand{\real}[1]{\mathbbm{Re}\br{#1}}

\newtheorem{lem}{Lemma}[section]

\newtheorem{thm}[lem]{Theorem}

\newtheorem{cor}[lem]{Corollary}

\newtheorem{defn}[lem]{Definition}

\newtheorem{prop}[lem]{Proposition}
\newtheorem{property}[lem]{Property}
\newtheorem{assumption}[lem]{Assumption}

\newtheorem{remark}[lem]{Remark}

\crefalias{asmp}{enumi}
\crefname{asmp}{assumption}{assumptions}
\Crefname{asmp}{Assumption}{Assumptions}

\DeclareMathOperator*{\argmin}{arg\,min}

\makeatletter
\newcommand{\vast}{\bBigg@{4}}
\newcommand{\Vast}{\bBigg@{5}}
\makeatother

\newcommand{\ex}[2]{{\ifx&#1& \mathbb{E} \else
\underset{#1}{\mathbb{E}} \fi \left[#2\right]}}
\newcommand{\pr}[2]{{\ifx&#1& \mathbb{P} \else
\underset{#1}{\mathbb{P}} \fi \left[#2\right]}}

\newcommand{\indi}[1]{\mathbbm{1}\left(#1\right)}

\newcommand{\grad}{\nabla}

\usepackage{ulem}

\DeclarePairedDelimiterX{\infdivx}[2]{(}{)}{%
  #1\;\delimsize\|\;#2%
}

\newcommand{\mypar}[1]{\smallskip
	\noindent{\textbf{{#1}:}}}
	
\renewcommand{\epsilon}{\varepsilon}

\renewcommand{\tilde}{\widetilde}

\newcommand{\DP}{{\sf{dp}}}
\newcommand{\SGD}{{\sf{sgd}}}

\newcommand{\D}{{\mathrm{d}}}  %
\newcommand{\I}{i}  %
\newcommand{\sigmadp}{\sigma_\DP}
\newcommand{\sigmasgd}{\sigma_\SGD}

\newcommand{\kurt}{C_{\sf kurt}}
\newcommand{\thetadp}{\tilde\bftheta^{\,\DP}}
\newcommand{\thetasgd}{\tilde\bftheta^{\,\SGD}}

\newcommand{\edim}{d_{\sf eff}}
\newcommand{\srank}{{\sf srank}}

\newcommand{\clip}[2]{{\sf clip}\left(#1,#2\right)}

\newcommand{\tr}[1]{{\sf Tr}\left[#1\right]}
\newcommand{\diag}{{\sf diag}}

\newcommand{\norm}[1]{{\left\Vert {#1} \right\Vert}}

\newcommand \pow [1]{^{(#1)}}
\DeclarePairedDelimiterX{\inp}[2]{\langle}{\rangle}{#1, #2} %

\setlist{nolistsep}
\setlist[itemize]{noitemsep, topsep=0pt}

\setlist{nolistsep}
\setlist[itemize]{noitemsep, topsep=0pt}

\newtheorem{theorem}[lem]{Theorem}

\newcommand{\expect}{\mathbb{E}}

\newcommand{\toeplitz}{\mathrm{Toeplitz}}

\newcommand\blfootnote[1]{%
  \begingroup
  \renewcommand\thefootnote{}\footnote{#1}%
  \addtocounter{footnote}{-1}%
  \endgroup
} 

\newcounter{arxiv}
\title{
\textbf{Correlated Noise Provably Beats Independent Noise} \\
\textbf{for Differentially Private Learning}
}

\author{
\renewcommand{\arraystretch}{1.2}
\begin{tabular}{cc}
Christopher A. Choquette-Choo$^*$ 
&
Krishnamurthy (Dj) Dvijotham$^*$  
\\
Krishna Pillutla$^*$
&
Arun Ganesh
\\
Thomas Steinke
&
Abhradeep Guha Thakurta
\end{tabular}
 \vspace{0.5em}
\\
  Google
}

\date{\vspace{-2em}}

\begin{document}
\normalem  %

\maketitle

\blfootnote{
$^*$Equal contribution; alphabetical ordering.  
}

\doparttoc %
\faketableofcontents %

\begin{abstract}
Differentially private (DP) learning algorithms inject noise into the learning process. While the most common private learning algorithm, DP-SGD, adds independent Gaussian noise in each iteration, recent work on matrix factorization mechanisms has shown empirically that introducing correlations in the noise can greatly improve their utility. We characterize the asymptotic learning utility for any choice of the correlation function, giving precise analytical bounds for linear regression and as the solution to a convex program for general convex functions. We show, using these bounds, how correlated noise provably improves upon vanilla DP-SGD as a function of problem parameters such as the effective dimension and condition number. Moreover, our analytical expression for the near-optimal correlation function circumvents the cubic complexity of the semi-definite program used to optimize the noise correlation matrix in previous work. We validate our theory with experiments on private deep learning. Our work matches or outperforms prior work while being efficient both in terms of compute and memory. \end{abstract}

\section{Introduction}
\label{sec:intro}
The broad adoption of deep learning using sensitive data has led to the increasing popularity of rigorous frameworks for privacy preservation, such as differential privacy~\cite{DMNS}. The workhorse of private learning, a differentially private variant of stochastic gradient descent called \privsgd~\cite{song2013stochastic,BST14, DP-DL}, clips per-example gradients to some $\ell_2$ norm and adds \textit{independent} Gaussian noise.
\privsgd has been used in a range of applications from learning with medical images~\cite{adnan2022federated} to finetuning large language models with $O(100B)$ parameters~\cite{he2023exploring}.

A recent line of work instead proposes to add \textit{correlated} Gaussian noise to each clipped gradient~\cite{thakurta2013nearly,kairouz2021practical,denisov2022improved,choquette2023multi}.
This class of algorithms called \dpftrl, has been used for private federated learning at industrial scale~\cite{xu2023federated}. By solving an expensive semi-definite program to find the noise correlations,
\citet{choquette2023amplified} demonstrated \textit{empirically} that \dpftrl is never worse and often \textit{much better} than \privsgd in its privacy-utility tradeoff across multiple modalities like images and text.

However, several questions remain open. Does \dpftrl \textbf{provably improve} over \privsgd in its expected utility? Further, can we design a more \textbf{computationally efficient} procedure to find the noise correlations for \dpftrl without significantly worsening the privacy-utility tradeoff?

We answer both questions affirmatively by (1) providing a sharp theoretical characterization of the noisy training dynamics of \dpftrl, and (2) leveraging these analytical tools to circumvent the semi-definite program required in past work.

\subsection{Problem Setup and Background}\label{sec:setting}
\begin{algorithm}[t]
\caption{The \algname/\noisyftrl algorithms with a noise coefficient matrix $\bfB \in \R^{T \times T}$}
\label{alg:dpmf}
\small
\begin{algorithmic}[1]
\Require $\bfB \in \reals^{T\times T}$, initial iterate $\bftheta_0 \in \reals^d$, $\ell_2$ clip norm $G$, noise multiplier $\sigmadp$, learning rate $\eta$, dataset $\calD$
\For{$t=0, \ldots, T-1$}
\State 
Obtain the next datapoint $\bfz_t$ and compute $\bfg_t =
\begin{cases}
\grad f(\bftheta_{t}; \bfz_t) + \grad \reg(\bftheta) & \text{ for \noisyftrl},  \\
\clip{\grad f(\bftheta_{t}; \bfz_t) }{G}+ \grad \reg(\bftheta) & \text{ for  \dpftrl}
\end{cases}$ 
\State Sample noise $\bfw_t \sim \calN(0, \sigmadp^2 G^2 \bfI_d)$ and calculate the correlated noise $\tilde{\bfw}_t = \sum_{\tau=0}^{t} \bfB_{t, \tau} \bfw_{\tau}$
\label{line:noise-add}
\State Update $\bftheta_{t+1} = \bftheta_{t} - \eta\tilde\bfg_t$ for the noisy gradient $\tilde  \bfg_t = \bfg_t + \tilde \bfw_t$
\EndFor
\Ensure $\bftheta_T$
\end{algorithmic}
\end{algorithm}

Let $\calD = \{\bfz_0, \ldots, \bfz_{T-1}\}$ be a dataset of $T$ datapoints, where each datapoint is sampled i.i.d. from an underlying distribution $\Pdata$.
Our learning objective is to minimize:
\begin{align}
\label{eq:obj-general}
    F\br{\bftheta} = \ExP{\bfz \sim \Pdata}{f\br{\bftheta; \bfz}} 
    + \reg\br{\bftheta} 
    \,,
\end{align}
where 
$f(\bftheta; \bfz)$ is the loss incurred by model parameters $\bftheta \in \reals^d$ on a datapoint $\bfz$, and
$\reg(\cdot)$ is data-independent regularization. We aim to minimize $F$ while satisfying differential privacy with respect to the dataset $\calD$. We assume that $F$ has a unique minimizer denoted $\bftheta_\star$.

We focus on variants of stochastic gradient descent with a batch size of $1$ for data arriving in a stream. The learning algorithms we study are presented in Algorithm \ref{alg:dpmf}; we assume throughout that the dataset $\calD$ is randomly shuffled before running the algorithm so that each datapoint $\bfz_t$ is an i.i.d. sample from $\Pdata$.
\privftrl with a noise coefficient matrix $\bfB \in \R^{T \times T}$ (which is lower triangular w.l.o.g.) performs the updates\footnote{
    Matrices (e.g. $\bfB = [\bfB_{t, \tau}]_{t, \tau \ge 0}$) and vectors (e.g. $\bfbeta = (\beta_0, \beta_1, \ldots)$) are zero-indexed and bold-faced.
}
\begin{align} \label{eq:dpftrl:intro}
\textstyle
    \bftheta_{t+1} = \bftheta_t - \eta \left( \clip{\grad f(\bftheta_t; \bfz_t)}{G} + \grad r(\bftheta_t) + \sum_{\tau=0}^{t} \bfB_{t,\tau} \bfw_\tau \right)
\end{align}
for Gaussian noise $\bfw_t \sim \calN(\boldzero, \sigmadp^2 G^2 \bfI_d)$, where $\clip{\cdot\,}{G}$ denotes projection onto an $\ell_2$ ball of radius $G$. We define \noisyftrl to be \privftrl without the gradient clipping operation. Taking $\bfB=\bfI$ as the identity matrix recovers DP-SGD (with clipping) and \noisysgd (without clipping), and other choices give rise to alternate algorithms.
The actual noise injected into the learning process $\tilde{\bfw}_t = \sum_{\tau=0}^{t} \bfB_{t, \tau} \bfw_{\tau}$ is thus correlated across iterations when $\bfB \neq \bfI$.

We restate a result from prior work showing that  \privftrl is differentially private for any choice of the noise coefficient matrix $\bfB$, provided the noise multiplier is scaled up appropriately.

\begin{theorem}[\citet{denisov2022improved, bun2016concentrated}]
\label{thm:dp}
\dpftrl (\Cref{alg:dpmf} with the clipping enabled) satisfies $\rho$-zero concentrated differential privacy (zCDP) if the noise multiplier is taken as $\sigmadp^2 = {\gamma_T^2(\bfB)}/{(2\rho)}$ where
$\gamma_T\br{\bfB}= \max_{t < T} {\| (\bfB^{-1})_{:,t}\|_2}$ is the sensitivity of $\bfB^{-1}$.\footnote{
    We give DP guarantees w.r.t. the ``zero-out'' notion of neighborhood~\cite{kairouz2021practical}; see \Cref{sec:a:background} for a review.
Further, a $\rho$-zCDP guarantee can be readily translated into $(\epsilon, \delta)$-DP~\cite[Prop. 1.3]{bun2016concentrated}.
}
\end{theorem}

\begin{remark}
Although \noisyftrl is \emph{not} differentially private, it lets us analyze the noise dynamics of \dpftrl without technicalities associated with clipping. 
We sharply characterize the asymptotic utility of \noisyftrl for linear regression and
show later that this analysis extends to \privftrl under appropriate assumptions. For mean estimation and learning with Lipschitz convex losses, we directly analyze \privftrl.
\end{remark}

\subsection{Motivation}
This work is motivated by two open questions in particular.

\mypar{Provable separation between \privsgd and \dpftrl}
The best-known separation between \privsgd and \dpftrl in the literature is due to \citet{kairouz2021practical}.
For $G$-Lipschitz convex losses, \dpftrl at a privacy level of $\rho$-zCDP achieves a suboptimality of $O(G d^{1/4} / \sqrt{\rho T})$ compared to \privsgd's $O(G d^{1/4} / \sqrt{\rho^2 T})$. The only improvement here is in terms of the privacy parameter $\rho$. 
More recently, \citet{koloskova2023convergence} analyze \noisyftrl but \textit{without} normalizing for the sensitivity $\gamma_T(\bfB)$ as required by \Cref{thm:dp}.
Thus, the existing theory fails to reflect the large margin by which \dpftrl empirically outperforms \privsgd across the board~\cite{choquette2023amplified}, and a precise characterization is missing.

\mypar{Computationally efficient \dpftrl}
Prior work on \dpftrl utilizes the noise coefficient matrix $\bfB$ that minimizes the squared error in the \textit{gradient prefix sums}~\citep{kairouz2021practical,denisov2022improved}:
\begin{align} \label{eq:prefix-error}
\textstyle
	\varphi(\bfB) =
  	\sum_{t=0}^{T-1} \expect \norm{\sum_{\tau=0}^t \tilde{\bfg}_t - \sum_{\tau=0}^t \bfg_t}_2^2
\end{align}
where $\bfg_t$ is the clipped gradient applied in iteration $t$ and $\tilde \bfg_t$ is its noisy counterpart, with the noise being correlated by the rows of the coefficient matrix $\bfB$ as in \Cref{alg:dpmf}.
This surrogate objective was, in turn, obtained as an upper bound on the regret in an adversarial online learning setting~\cite[Thm. C.1]{kairouz2021practical}.
The most potent algorithm from the previous work selected the coefficient $\bfB$ as the solution of a semidefinite program with matrix variables of size $O(T^2)$, requiring $O(T^3)$ time~\cite[Eq. 4]{denisov2022improved}.
This cost is prohibitive for large learning problems.
Moreover, there is a mismatch between the objective \eqref{eq:prefix-error} used to find the noise coefficients and the final learning objective $F(\bftheta_T)$. In particular, there exist matrices $\bfB_1, \bfB_2$ 
with equal squared error $\varphi(\bfB_1) = \varphi(\bfB_2)$ and equal sensitivities $\gamma_T(\bfB_1) = \gamma_T(\bfB_2)$ such that \dpftrl with $\bfB_1$ diverges while \dpftrl with $\bfB_2$ converges~\cite{koloskova2023convergence}.

\mypar{Our approach}
We study the suboptimality in the final objective $\expect[F(\bftheta_T) - F(\bftheta_\star)]$. We work in the asymptotic $T\to \infty$ regime to allow the use of analytic tools, but also to derive results that apply regardless of the dataset size.%
\footnote{
    Note that the DP noise multiplier $\sigmadp$ remains finite in the asymptotic $T \to \infty$ regime as we consider the streaming setting: each example is processed once and the number of examples also grows to infinity.
}
Second, we restrict the noise coefficient matrix $\bfB$ to be \textit{Toeplitz}, i.e., it satisfies $\bfB_{t, \tau} = \beta_{t-\tau}$ for a sequence $\bfbeta=(\beta_0, \beta_1, \ldots)$ of reals. Toeplitz noise coefficients have the advantageous property of being usable {\bf anytime}, i.e., they do not be recomputed for each value of $T$ and easily apply as $T \to \infty$. Toeplitz noise coefficient matrices $\bfB$ were previously considered for their computational efficiency in learning~\cite{choquette2023multi} and their near-optimal rates in linear counting queries~\cite{henzinger2023unifying}.

Thus, our goal is to characterize the \textbf{asymptotic suboptimality}
\begin{align} \label{eq:fobj}
 \Fobj(\bfbeta) := \lim_{T \to \infty}    \ExP{}{F\br{\bftheta_T} - F(\bftheta_\star) } 
\end{align}
for $\bftheta_T$ produced by \noisyftrl or \dpftrl under noise coefficients $\bfbeta$ where $\bftheta_\star = \argmin F$ is assumed unique. This limit turns out to be well-defined and finite for the settings we consider as long as $\norm{\bfbeta}_2$ is finite.

We analyze $\Fobj$ in the frequency domain using the \textbf{discrete-time Fourier transform} 
$B(\omega) = \sum_{t=0}^\infty \beta_t \exp(\I \omega t)$, with $\I$ denoting the imaginary unit. This transformation is invertible, so we use the noise coefficients $\bfbeta$ interchangeably with its Fourier representation $B$. Further, we define the limiting sensitivity associated with the (Fourier representation of the) noise coefficients $B$ as the limiting value of the sensitivity $\gamma_T$ over $T \to \infty$ iterations:
\begin{align} \label{eq:sens-fourier-main}
\gamma_\infty\br{B} := \lim_{T \to \infty} \gamma_T\br{B} = 
{
\textstyle
\left( \frac{1}{2\pi} \int_{-\pi}^\pi |B\br{\omega}|^{-2} \,\, \D\omega \right)^{1/2}\,,
}
\end{align}
where the last equality 
follows from standard tools in Fourier analysis.

\subsection{Our Contributions}

The concrete contributions of this work are as follows.

\mypar{\ourprivftrl: Analytically optimal \privftrl for mean estimation} 
We give analytical expressions for the asymptotic suboptimality $\Fobj$ for mean estimation and the noise coefficients $\bfbeta$ that minimize $\Fobj$ as a function of the learning rate $\eta$ (\S\ref{sec:mean_estimation}).
\revision{
We find that the optimal noise is \emph{anti}-correlated, so the algorithm \emph{cancels out} previously added noise.
}
Inspired by the analytical expression for the optimal noise coefficients $\bfbeta_\star$ for mean estimation, we propose a single-parameter family of choices for the noise coefficients $\bfbeta$; we call this variant \ourprivftrl. We show its favorable theoretical and empirical properties for a broader range of problems.

\begin{table}[tb]
\caption{\small
Asymptotic suboptimality of \noisysgd/\noisyftrl for linear regression with Gaussian inputs $\bfx \sim \calN(\boldzero, \bfH)$ and noise multiplier $\sigmadp^2 = \gamma_\infty(\bfbeta)^2 / (2\rho)$ based on the limiting sensitivity \eqref{eq:sens-fourier-main}. We give the bounds in terms of the fixed learning rate $\eta > 0$, dimension $d$, the \textbf{effective dimension} $\edim = {\tr{\bfH}} / {\norm{\bfH}_2}$ of the problem, and the noise variance $\rho^{-1}$ representing the privacy level. Without loss of generality, we take $G=1$ and $\norm{\bfH}_2 = 1$ (thus, $\eta \le 1$ is required for convergence). We only show the term depending on $\rho$ as it captures the effect of the correlated noise.
Since $1 \le \edim \le d$, \noisyftrl is significantly better than \noisysgd at smaller learning rate $\eta$ or
when the efficient dimension $\edim$ is small (e.g., when the input covariance $\bfH$ is close to low rank).
}\label{tab:rates-noisy}
\centering
\small
\renewcommand{\arraystretch}{2.0}
\adjustbox{max width=0.95\textwidth}{%
\begin{tabular}{cccc}
\toprule
\textbf{Algorithm} & 
{\renewcommand{\arraystretch}{1}
\begin{tabular}{c} \textbf{Asymptotic} \\ \textbf{Suboptimality $F_\infty$} \end{tabular}} &
{\renewcommand{\arraystretch}{1}
\begin{tabular}{c} 
\textbf{Ratio w/} \\ \textbf{Lower Bound} \end{tabular}} &
\textbf{Remark}
\\
\midrule
Lower Bound &
$\Omega\br{\eta^2 \rho^{-1} \edim}$ &
$1$ &
{\renewcommand{\arraystretch}{1}
\begin{tabular}{c}
for all noise coefficients $\bfbeta$ \\ with finite $\norm{\bfbeta}_1$ \end{tabular}}
\\ 
\noisysgd &
$\Theta\br{\eta \rho^{-1} d}$ &
$\dfrac{d}{\eta \edim}$ &
{\renewcommand{\arraystretch}{1}
\begin{tabular}{c} $\Theta(\cdot)$ denotes matching upper \& lower bounds \\ (up to absolute constants)  \end{tabular}}
\\ 
\ournoisyftrl &
$O\br{\eta^2 \rho^{-1} \edim \log^2\frac{1}{\eta\mu}}$
& 
$\log^2\frac{1}{\eta\mu}$ &
{\renewcommand{\arraystretch}{1}
\begin{tabular}{c}
    $\mu = \lambda_{\min}\br{\bfH}$ and we use the noise \\ coefficients $\bfbeta$ from \eqref{eq:optimal_beta}
\end{tabular}}
 \\
\bottomrule
\end{tabular}
} %
\end{table}

\mypar{Strict separation for linear regression}
We establish sharp bounds on the asymptotic suboptimality of \noisyftrl (i.e., \dpftrl without gradient clipping) for linear regression.
Summarized in \Cref{tab:rates-noisy} and stated formally in \S\ref{sec:linear-regr}, we show:
\begin{enumerate}[label=(\alph*), nosep,leftmargin=1.7em]
    \item \ournoisyftrl, with analytical closed-form noise coefficients, matches (up to log factors) the lower bound we establish on the asymptotic suboptimality for any possible noise coefficients.
    Both of these bounds scale with the effective dimension $\edim$ of the problem, which is no greater than the dimension $d$ but can be much smaller when the data is approximately low rank.
    \item \ournoisyftrl is provably better than \noisysgd by a factor that can be as large as $d / \log{d}$ (when $\edim$ is a constant). This shows an exponential separation between \noisyftrl and \noisysgd.
\end{enumerate}

\noindent
Our bounds quantitatively show how the anti-correlations of \ournoisyftrl help prevent noise accumulation along eigen-directions of the Hessian with small eigenvalues. The gradients have a weak signal along these directions and are unable to undo the effect of the previous noise and move the iterates back toward the minimizer. The cancellation of the noise is essential to obtain the near-optimal asymptotic suboptimality.
We also leverage these asymptotic results to give bounds on the utility of \privsgd and \ourprivftrl for finite $T$; these bounds demonstrate a similar improvement from the dimension to the effective dimension.

\mypar{Numerical separation for general strongly convex functions}
We bound the asymptotic suboptimality $\Fobj$ for any noise coefficients $\bfbeta$ as the optimal value of a convex program. We use this to show that \privftrl achieves a tighter bound particularly when the condition number is large (Figure \ref{fig:IQCbound} in \S\ref{sec:iqc}).

\mypar{Experiments with private deep learning}
We show the proposed \ourprivftrl outperforms other efficient differentially private algorithms on image and text classification tasks. Our approach is competitive even with inefficient approaches that require $O(T^3)$ computation and $O(T^2)$ memory to compute the noise coefficient matrix $\bfB$.
 
\section{Analysis for Quadratic Objectives}
For quadratic objective functions, \Cref{alg:dpmf} (with no clipping) corresponds to a linear dynamical system~\citep{gray2004introduction}, which allows the application of analytical tools. This enables an exact analysis of \privftrl for mean estimation and \noisyftrl for linear regression. The analysis of \noisyftrl also lets us derive guarantees for \privftrl for linear regression. We do not aim to achieve the best possible rates in these stylized models. Rather, our goal is to understand the noise dynamics of \dpftrl and show a separation with \privsgd.

\subsection{Conceptual Overview: Private Mean Estimation in One Dimension}
\label{sec:mean_estimation}
We begin with a simple objective function, namely the squared error for a mean estimation problem on the real line. This setting captures the core intuition and ideas used to derive further results.

Consider a distribution $\Pdata$ supported on $[-1, 1]$ with $|z - \expect[z]| \le \sigmasgd$. We consider estimating the mean privately by minimizing the following squared error with \privsgd or \privftrl:
\begin{align} \label{eq:mean-est}
\textstyle
    F\br{\theta} = \frac{1}{2}\, \expect_{z \sim \Pdata}{(\theta-z)^2}  \,.
\end{align}
This is a special case of the learning problem in Eq.~\eqref{eq:obj-general} with
\begin{align*}
\textstyle
        f\br{\theta; z} = \frac{z^2}{2}-z\theta,
        \quad \text{and} \quad
        r\br{\theta}=\frac{\theta^2}{2}  \,.
\end{align*}
We show a strict separation between \dpftrl and \privsgd for this simple minimization problem.

\begin{figure}[t]
    \centering
    \includegraphics[width=0.97\linewidth]{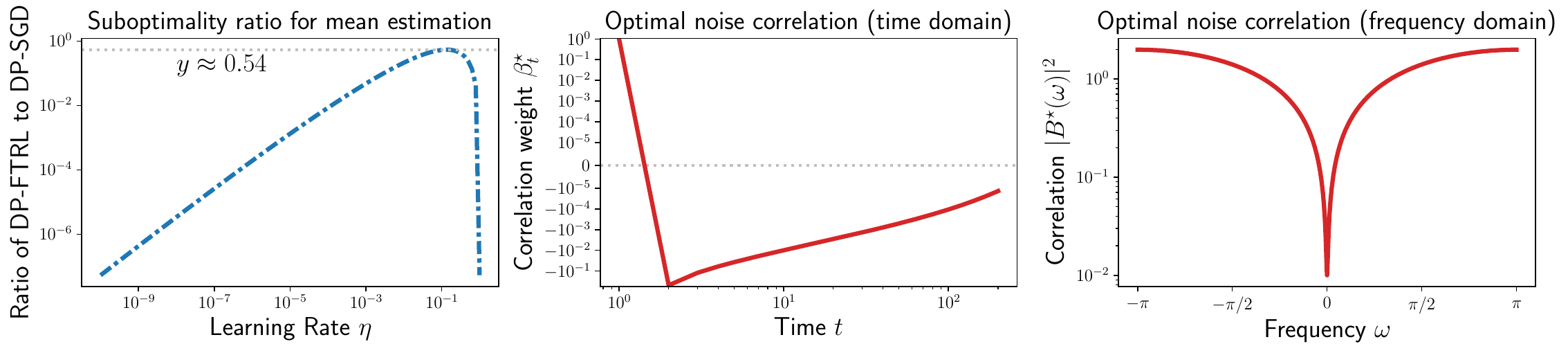}
    \caption{\small
    \textbf{Left}: The ratio of the asymptotic suboptimalities of \dpftrl to \privsgd for mean estimation vs. the learning rate $\eta$. \dpftrl is never worse but is orders of magnitude better at $\eta \to 0$ or $\eta\to 1$.
    \textbf{Middle \& Right}: Time- and frequency-domain descriptions of the optimal noise coefficients for mean estimation (defined in \Cref{thm:mean}).
    }
    \label{fig:mean}
\end{figure}

\begin{theorem} \label{thm:mean}
    Consider the setting above with learning rate $\eta \le 1$, a clip norm $G = 1$, and a (squared) noise multiplier $\sigmadp^2=\frac{\gamma_{\infty}\br{\bfbeta}^2}{2\rho}$ selected to ensure that the output sequence $(\theta_t)_{t=0}^\infty$ of \dpftrl with noise coefficients $\bfbeta$ is $\rho$-zCDP.
    Then, the asymptotic suboptimality of DP-SGD with noise coefficients $\bfbeta^\SGD = (1, 0, 0, \ldots)$ is 
    \[
    F_\infty(\bfbeta^\SGD) = \Theta(\eta\rho^{-1} + \eta \sigmasgd^2)\,.
    \]
    Further, the smallest asymptotic suboptimality of any $\rho$-zCDP sequence $(\theta_t)_{t=0}^\infty$ from \dpftrl is
    \[
    \ifnum \value{arxiv}= 0 {\textstyle} \else {} \fi  %
        \inf_\bfbeta F_\infty(\bfbeta) = F_\infty(\bfbeta^\star) = \Theta\left( \eta^2 \rho^{-1} \log^2(1/\eta) + \eta \sigmasgd^2 \right)
        \,.
    \]
    The infimum above is attained by the noise coefficients
    $\beta^\star_t = (-1)^t {1/2 \choose t} (1 - \eta)^t$, where we denote the fractional binomial coefficient ${1/2 \choose t} = \prod_{k=0}^{t-1} \frac{1/2 - k}{t - k}$.
\end{theorem}
\begin{proof}[Proof Sketch]
Using tools from frequency-domain analysis of linear time-invariant systems~\cite{oppenheim1997signals}, we show that the asymptotic suboptimality of \dpftrl with noise coefficients $B(\cdot)$ in the Fourier domain is (for some absolute constant $C$):
\[
    F_\infty(B) = C \, \eta^2 \rho^{-1} \, \gamma_\infty^2(B) \,\, \int_{-\pi}^\pi \frac{|B(\omega)|^2 \, \, \D \omega}{|1 - \eta - \exp(\I \omega)|^2}   +  \eta \sigmasgd^2 \,.
\]
The result for \privsgd can be obtained by plugging in $B(\omega) \equiv 1$ and evaluating the integral. Next, we turn to the best possible error from \dpftrl.
By plugging in the sensitivity $\gamma_\infty(B)$ from \eqref{eq:sens-fourier-main} and ignoring the terms independent of $B(\cdot)$, we find that the asymptotic suboptimality $F_\infty(B)$ is a product of two integrals: 
\[
    \left( \int_{-\pi}^\pi \frac{\D \omega}{|B(\omega)|^2}  \right)
    \,\,
    \left( \int_{-\pi}^\pi \frac{|B(\omega)|^2 \,\, \D\omega}{|1 - \eta - \exp(\I \omega)|^2} \right)  \,.
\]
This product is minimized (with respect to the choice of $B$) with $|B^\star(\omega)|^2=|1 - \eta - \exp(\I \omega)|$ (see Fig.~\ref{fig:mean}, right for a plot). This can be seen, for instance, from the Cauchy-Schwarz inequality. The corresponding coefficients $\bfbeta^\star$ in the time-domain can be obtained via an inverse Fourier transform  (Fig.~\ref{fig:mean}, center). We give the full proof in \S\ref{sec:a:mean}.
\end{proof}

We make several remarks about this result.
First, \Cref{thm:mean} demonstrates a clear gap between \privsgd and \privftrl:
the optimal $\rho^{-1}$ coefficient $\eta^2 \log^2(1/\eta)$ is always better than \privsgd's $\eta$, and is significantly better when the learning rate $\eta\to 0$; see the left plot of \Cref{fig:mean}.
Second, the optimal noise coefficients satisfy
\[
    \beta_t^\star = 
    \begin{cases}
    1 \,, & \text{ if } t = 0\,, \\
    -\Theta(t^{-3/2} (1 - \eta)^t) \,, & \text{ else} \,.
    \end{cases}
\]
Importantly, note that $\beta_t^\star < 0$ for $t \ge 1$ (see also the middle plot of \Cref{fig:mean}).
Thus, \dpftrl helps by \emph{subtracting out} or \emph{canceling} the previously injected noise.
Moreover, the actual noise $(\tilde \bfw_t)_{t=0}^\infty$ injected into the learning process (as defined in line~\ref{line:noise-add} of \Cref{alg:dpmf}) is also \textit{anti-correlated}, i.e., $\expect\inp{\tilde \bfw_t}{\tilde \bfw_\tau} < 0$ for $t \neq \tau$. 

Finally, we also recover the noise coefficients of \citet{fichtenberger2023constant} with $\eta = 0$. These coefficients were shown to be near-optimal for linear counting queries~\cite{henzinger2023unifying} and were later shown to be optimal in the class of Toeplitz noise coefficients for this problem~\cite{mcmahan2024efficient}. The additional exponential $(1-\eta)^t$ term in our noise coefficients compared to that of \citep{fichtenberger2023constant} is necessary for optimality in mean estimation because gradient descent is contractive on strongly convex learning problems.

\mypar{\ourprivftrl/\ournoisyftrl}
\Cref{thm:mean} gives an analytical expression for the optimal noise coefficients for \dpftrl for the simplified setting of mean estimation. 
We adapt these coefficients for more general problems by parameterizing these coefficients. Specifically, given a parameter $0 < \nu < 1$, we define
\begin{align} \label{eq:optimal_beta}
\textstyle
    \hat \beta_t^\nu := (-1)^t {1/2 \choose t} (1 - \nu)^t \,.
\end{align}
We analyze this choice theoretically for the setting of linear regression and demonstrate near-optimality for appropriate $\nu$. Later, for our experiments with \privftrl, we tune $\nu$ as a hyperparameter to tune. We call this approach (with clipping) \ourprivftrl and (without clipping) \ournoisyftrl.
 
\subsection{Asymptotic Suboptimality for Linear Regression}
\label{sec:linear-regr}

We now give a precise analysis of the asymptotic suboptimality $\Fobj$ for linear regression with \ournoisyftrl. We will use this to derive non-asymptotic privacy-utility bounds for \dpftrl at the end of this section.

We consider (unregularized) linear regression with the squared loss
$f\br{\bftheta; \br{\bfx, y}} = \frac{1}{2}\br{y-\inner{\bftheta}{\bfx}}^2$ so that our objective is
\begin{align} \label{eq:linear-regression:objective}
\textstyle
    F(\bftheta) = \frac{1}{2} \,\, \expect_{(\bfx, y) \sim \Pdata} \left(y -  \inp{\bftheta}{\bfx} \right)^2  \,.
\end{align}
We assume $d$-dimensional Gaussian covariates $\bfx \sim \calN(\boldzero, \bfH)$ and a well-specified linear model with Gaussian residuals $y - \inp{\bftheta_\star}{\bfx} \sim \calN(0, \sigmasgd^2)$ where $\bftheta_\star = \argmin F$. We make these assumptions for ease of presentation; we state and prove our results under weaker assumptions in the supplement (e.g. that $\bfx$ has bounded fourth moments or is sub-Gaussian). Further, we assume that the objective $F$ is $L$-smooth and $\mu$-strongly convex. This is equivalent to assuming that $\mu \bfI \preceq \bfH \preceq L \bfI$, since the input covariance $\bfH$ is also the Hessian of the quadratic objective $F$.
We express the bounds on $F_\infty$ in terms of the problem parameters $\rho, G$ which, for \dpftrl, denote the target privacy level and the gradient clip norm respectively.
The full proofs from this section are given in \S\ref{sec:a_asymptotic_lr}. Our main result is the following.

\begin{figure*}[t]
    \includegraphics[width=0.95\linewidth]{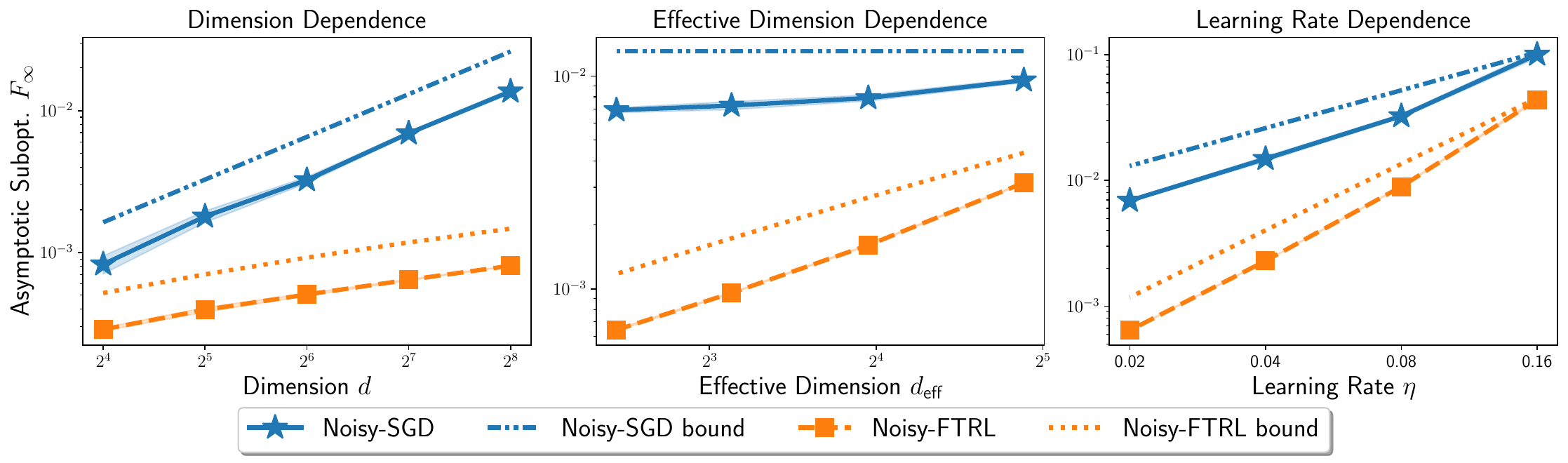}
    \ifnum \value{arxiv} = 0 {\vspace{-1em} } \else {} \fi
    \caption{\small \textbf{Linear regression simulations}: We plot the empirically observed asymptotic suboptimality of \ournoisyftrl/\noisysgd and their theoretical bounds with $d=128$ (varied in the left plot) where the Hessian $\bfH$ has eigenvalues $\lambda_k = 1/k$ (varied as $k^{-\alpha}$ for $\alpha \in [0.4, 1]$ in the middle plot), and learning rate $\eta = 0.02$ (varied in the right plot). The \textbf{slope} of the corresponding empirical and theoretical lines are nearly equal, showing the \textbf{tightness of the theory}. In particular, we observe that \noisysgd has a linear dependence on the dimension (slope $1.00$) and is nearly constant w.r.t. the effective dimension (slope $0.18$) while \noisyftrl has a near-linear dependence on the effective dimension (slope $0.94$). \noisyftrl (slope $2.03$) also has a better dependence on the learning rate than \noisysgd (slope $1.27$).
    }
    \label{fig:noisy_linear_regression}
\end{figure*}

\begin{theorem} \label{thm:lr}
Let $c, C_1, C_2$ denote universal constants and consider the linear regression setting above. Consider the sequence $(\bftheta_t)_{t=0}^\infty$ produced by \noisyftrl with a constant learning rate $0 < \eta \le c / \tr{\bfH}$ and a (squared) noise multiplier $\sigmadp^2 = \gamma_\infty^2(\bfbeta) / (2\rho)$ for noise coefficients $\bfbeta$. Then, we have the following results:
\begin{align*}
    \textbf{(\noisysgd)} 
    &&
    F_\infty(\bfbeta^\SGD) = \Theta\left( \eta d G^2{\rho^{-1}} + \eta \sigmasgd^2 \tr{\bfH}\right) 
    &\quad
    \text{ with } \bfbeta^\SGD = (1, 0, \ldots)\,,
    \\
    \textbf{(\ournoisyftrl)}
    &&
    F_\infty(\hat{\bfbeta^\nu}) \leq C_1  \left(\eta^2 G^2\rho^{-1} \log^2\tfrac{1}{\nu} + \eta \sigmasgd^2 \right) \tr{\bfH}
    &\quad \text{ with } \nu \le \eta \mu, \text{ and} 
    \\
    \textbf{(Lower bound)}
    &&
    F_\infty(\bfbeta) \ge C_2 \left(\eta^2 G^2{\rho^{-1}} + \eta \sigmasgd^2\right) \tr{\bfH}
    &\quad \text{ for all } \bfbeta \text{ with } \norm{\bfbeta}_1 < \infty \,.
\end{align*}
This shows the near-optimality of \ournoisyftrl and a provable gap between \noisyftrl and \noisysgd.
\end{theorem}
We prove the bound on \noisysgd in \S\ref{sec:stationary:upper-bound}, the lower bound in \S\ref{sec:stationary:lower-bounds-proofs}, and the bound on \ournoisyftrl in \S\ref{sec:tuned_dp_ftrl_proof}.
Observe that our bounds separate the contributions arising from correlated noise ($\rho^{-1}$ term) and those from the inherent noise in the linear model ($\sigmasgd^2$ term). We focus on the effect of correlation because the effect of the latter noise is the same across all choices of the noise coefficients $\bfbeta$. We plot the bounds as well as numerical values of $F_\infty$ from simulations in \Cref{fig:noisy_linear_regression}.
The slopes of the bounds and the observed numerical suboptimality are nearly the same,\footnote{
    The curve $y=c x^\alpha$ appears in a log-log plot as a straight line with a slope $\alpha$.
} 
indicating the tightness of the theory with respect to the problem parameters.

\mypar{Exponential separation between \noisysgd and \noisyftrl}
\noisysgd's stationary error depends on the ambient dimension $d$, while the lower bound depends on the \textit{effective dimension} $\edim = \tr{\bfH} / \norm{\bfH}_2$ of the covariance $\bfH$. We have, $\edim \le d$ with equality when all the eigenvalues of $\bfH$ are equal. However, we can have $\edim \ll d$ when the eigenvalues of $\bfH$ decay rapidly or it is nearly low rank. This is true particularly for overparameterized models where the features may be highly correlated resulting in an approximately low-rank covariance.
For instance, if the eigenvalues of $\bfH$ are $(1, 1/d, \ldots, 1/d)$, then $\edim \le 2$. Then, \noisyftrl's error of $O(\eta^2 \rho^{-1} \log^2(d/\eta))$ is exponentially better than \noisysgd's $\Theta(\eta \rho^{-1} d)$.
A similar advantage also holds when eigenvalues of $\bfH$ decay at various rates; see \Cref{tab:rates-noisy:decays} in \S\ref{sec:a_asymptotic_lr}.
The learning rate dependence of \noisysgd is also suboptimal, similar to \S\ref{sec:mean_estimation}. This observation is also corroborated empirically in \Cref{fig:noisy_linear_regression} (right).

\mypar{Effective dimension and stable rank}
The stable rank of a matrix is defined as the squared ratio of its Frobenius norm to its largest singular value~\cite{rudelson2007sampling}. Thus, we have that $\edim = \srank(\bfH^{1/2})$ is the stable rank of the square root matrix $\bfH^{1/2}$.
It is generally desirable for numerical algorithms to depend on the stable rank of their matrix inputs rather than the true rank since the former is a continuous function while the latter is discontinuous~\cite{cohen2016optimal,martinsson2020randomized}. Thus, \ournoisyftrl exhibits this desirable property for linear regression, while \noisysgd does not. We refer to \S\ref{sec:a:edim} for a further discussion.

\mypar{Improvement in low signal directions}
The improvement from the dimension $d$ for \noisysgd to the effective dimension $\edim$ for \noisyftrl comes from reducing the error in low signal eigen-directions of the covariance $\bfH$.
Assume $\norm{\bfH}_2 = 1$ and consider the contribution of the $j$\textsuperscript{th} eigen-direction of the covariance to the asymptotic suboptimality.
We show that this contribution is $\Theta(1)$ for \noisysgd, while it scales with
the corresponding eigenvalue $\lambda_j$ for \ourprivftrl.
For the former, the low signal in the gradients in tail eigen-directions is insufficient to prevent the accumulation of noise. On the other hand, the anti-correlated noise of \ourprivftrl allows the cancellation of the past noise, leading to a significant improvement in such directions.
We refer to \Cref{remark:dim_vs_edim} of \Cref{sec:a_asymptotic_lr} for details on these calculations and how noise cancellation can help.

\mypar{Analysis of other noise coefficients}
The proof of \Cref{thm:lr} proceeds by bounding the asymptotic suboptimality of \noisyftrl with any noise coefficient $\bfbeta$ with finite $\norm{\bfbeta}_2$. This bound can be instantiated for other choices of the noise coefficients. One such example corresponds to anti-correlated perturbed gradient descent (anti-PGD), which was proposed in a context unrelated to privacy by \citet{orvieto2022anticorrelated} to improve generalization. As highlighted in \Cref{tab:ComparisonNoisy} and proved in \S\ref{sec:two-step-noise}, we show that a variant of anti-PGD interpolates between the rates of \noisysgd and \ournoisyftrl (in fact, it is their geometric mean, ignoring log factors).

\begin{table}[t]
\renewcommand{\arraystretch}{1.2}
\small
\centering
\caption{\small
\textbf{Comparison to prior work}:
    We apply our theory to compute $\Fobj$ for linear regression given choices of $\bfB$ used in prior work. 
    Though certain choices of the noise coefficients $\bfbeta$ may be optimal for finite linear counting queries~\cite{fichtenberger2023constant, mcmahan2024efficient}, our results show that they have $\Fobj=\infty$  because the sensitivity diverges as $T \to \infty$. $\nu$-\noisyftrl effectively introduces an additional damping term $(1-\nu)^t$ in the correlations of \cite{fichtenberger2023constant} to achieve near-optimality for linear regression. Damping similarly helps for anti-PGD~\cite{orvieto2022anticorrelated}, where the resulting error is the geometric mean of the lower bound and the bound of \noisysgd from \Cref{thm:lr}.
}
\label{tab:ComparisonNoisy}
\adjustbox{max width=0.97\linewidth}{
\begin{tabular}{cccc}
\toprule
    \textbf{Algorithm} & 
    \begin{tabular}{c}
    \textbf{Noise Coefficients} $\bfbeta$
    \end{tabular} &
    \begin{tabular}{c}
    \textbf{Sensitivity in $T$ steps} \\ $\gamma_T(\bfbeta)^2$ \end{tabular}
    &
    \begin{tabular}{c}
    \textbf{Asymptotic} \textbf{Suboptimality} \\ $F_\infty(\bfbeta)$ 
    \end{tabular} \\
    \midrule
    \cite{fichtenberger2023constant} &
    Eq.~\eqref{eq:optimal_beta} with $\nu = 0$ &
    $\log T$ &
    $\infty$
    \\
    $\nu$-\noisyftrl (Ours) &
    Eq.~\eqref{eq:optimal_beta} with $0 < \nu \le \eta\mu$ &
    $\log(1/\nu)$ &
    $\eta^2 G^2 \rho^{-1} \tr{\bfH} \log^2(1/\nu)$ \\
    Anti-PGD
    \cite{orvieto2022anticorrelated} & 
    $(1, -1, 0, \ldots)$
    & $T$ & $\infty$
    \\
    Anti-PGD + Damping
    & 
    $(1, -(1-\nu), 0, \ldots)$ &
    $1/\nu$ &
    $\eta^{3/2} G^2 \rho^{-1} \sqrt{d \, \tr{\bfH}}$ \\
    \bottomrule
\end{tabular}
}%
\end{table}

\subsection{Finite-time Privacy-Utility Bounds for Linear Regression}\label{sec:finite_lr}
\noisyftrl, which we analyzed so far, is not differentially private.
Differential privacy requires gradient clipping which significantly complicates the analysis due to the bias it introduces~\cite{koloskova2023revisiting}. However, for a finite time horizon $T$, we can argue using concentration that $\nabla f\br{\bftheta;\bfz}$ is bounded with high probability, and clipping can be avoided.
Formal statements and proofs for the finite-time analysis are given in \S\ref{app:finite_lr}.

Consider \ourprivftrl with noise coefficients $\hat\bfbeta^\nu$ from \eqref{eq:optimal_beta} with $\nu = \eta\mu$
and gradients clipped to a $\ell_2$-norm $G$ to be determined later.
As mentioned in \S\ref{sec:setting}, the outputs $(\bftheta_{1}, \ldots, \bftheta_T)$ of \dpftrl are $\rho$-zCDP for any choice of the clip norm $G$. 
For an appropriate choice of $\eta$, we give utility bounds in terms of the effective dimension $\edim$ and the condition number $\kappa=L/\mu$:
\begin{enumerate}[label=(\alph*), nosep,leftmargin=1.5em,topsep=-0.3em]
    \item \label{bullet:dp:a}
    For $\eta$ small enough, we have with probability at least $1-p$ that the stochastic gradient norm is uniformly bounded as
    \begin{align} \label{eq:grad-bound}
    \ifnum \value{arxiv}= 0 {\textstyle} \else {} \fi  %
        \max_{t<T} \norm{\bfg_t}_2 \le c \max \left\{ \tr{\bfH} \norm{\bftheta_0 - \bftheta_\star}_2, \sigmasgd \sqrt{\tr{\bfH}} \right\} \polylog{T/p} =: \tilde G \,.
    \end{align}
    We then take the clip norm as $G = \tilde G$ as defined in \eqref{eq:grad-bound}.
    When the event $\mathcal{E} := \{ \max_{t < T} \norm{\bfg_t}_2 \le \tilde G \}$ holds, then
    no gradients are clipped and \dpftrl coincides with \noisyftrl. The bounds we prove are meaningful only when this high-probability event holds.
    \item  
    For $T \ge \tilde \Omega(\kappa^2 \edim^2 d / \rho)$, we have the utility bound (omitting log factors and $o(1/T^2)$ terms and taking $\norm{\bfH}_2 = 1$):
    \[
    \ifnum \value{arxiv}= 0 {\textstyle} \else {} \fi  %
        \expect\left[(F(\bftheta_t) - F(\bftheta_\star)) \cdot \indi{\calE}\right]
        \lesssim
        \begin{cases}
            \kappa \, \edim  \left(
            \frac{d \edim \norm{\bftheta_0 - \bftheta_\star}_2^2}{\rho T}
             + \frac{ d \sigmasgd^2}{\rho T} + \frac{\sigmasgd^2}{T}
        \right)  & \text{ for \privsgd}, \\
            \kappa \edim  \left(
            \frac{\kappa \edim^2  \norm{\bftheta_0 - \bftheta_\star}_2^2}{\rho T^2}
             + \frac{\kappa \edim \sigmasgd^2 }{\rho T^2} + \frac{\sigmasgd^2}{T}
        \right) & \text{ for \ourprivftrl}.
        \end{cases}
    \]
\end{enumerate}
Thus, the dimension $d$ in \privsgd's bound effectively becomes $\kappa \edim / T$ for \dpftrl, leading to a better dimension dependence.
While faster $1/(\rho T^2)$ rates are known for \privsgd-style algorithms for linear regression~\cite{varshney2022nearly,liu2023near}, such algorithms require sophisticated adaptive clipping strategies. We analyze algorithms that use a fixed clipping norm $G = \tilde G$ and a fixed noise multiplier $\sigmadp$ independent of $T$; the bounds presented above are, to the best of our knowledge, the best known in the literature for \privsgd in this setting. We leave the exploration of combining adaptive clipping with correlated noise for future work.

\section{Asymptotic Suboptimality for General Strongly Convex Functions}
\label{sec:iqc}

We now generalize \S\ref{sec:linear-regr} to general strongly convex problems. Here, we bound the asymptotic suboptimality of \dpftrl and \privsgd by the value of a convex program.
\begin{theorem}\label{thm:iqc}
Suppose $f(\,\cdot\,; \bfz)$ is $G$-Lipschitz, 
and the stochastic gradients are uniformly bounded as $\norm{\nabla_{\theta} f\br{\bftheta;\bfz} - \ExP{\bfz^\prime \sim \Pdata}{\nabla_{\theta} f\br{\bftheta;\bfz^\prime}}}_2 \leq \sigmasgd$. Then, if $F$ is $\mu$-strongly convex and $L$-smooth, the asymptotic suboptimality $\Fobj$ is bounded for any noise coefficients $B\br{\omega}$ in the frequency domain by:
\begin{align}
    \inf \left\{
     \tfrac{ Ld }{2\pi} 
     \int_{-\pi}^\pi \br{{G^2}{\rho^{-1}}|B\br{\omega}|^2\gamma_\infty(B)^2 + \sigmasgd^2} \kappafunc(\omega) \, \D \omega  \, \middle\vert\, \kappafunc: [-\pi, \pi] \to \mathbb{R}_+ \,, \,\, \kappafunc \in \mathcal{C}\br{\eta, \smooth, \strong} \right\}\,,
     \label{eq:IQCbound}
    \end{align}
    where $\gamma_\infty(B)$ is the limiting sensitivity from Eq.~\eqref{eq:sens-fourier-main}, and $\mathcal{C}\br{\eta, \strong, \smooth}$ is a convex set (details and proof in \S\ref{app:IQC}).
\end{theorem}
\begin{figure}[t]
        \centering
        \includegraphics[width=.5\textwidth]{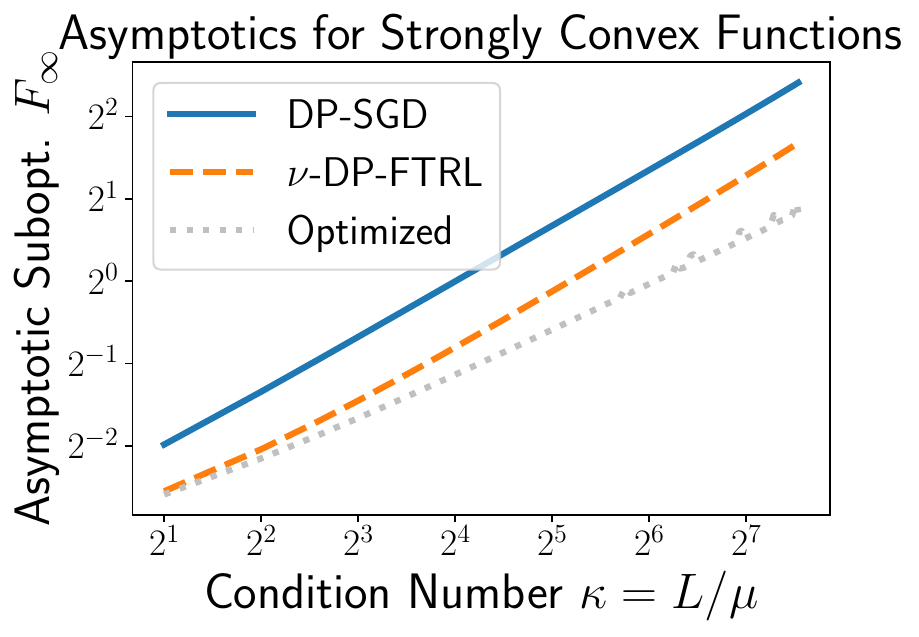}
        \caption{\small 
        \textbf{\dpftrl attains a tighter bound on $\Fobj$} with the growing condition number. Here, ``Optimized'' approximately minimizes \eqref{eq:IQCbound}. The plots hold for smooth and strongly convex functions ($L=1=G, \sigmasgd=0$).
        }
        \label{fig:IQCbound}
    \end{figure}
 
    While technically an infinite-dimensional optimization problem over the function $\kappafunc$, we can approximate the solution by discretizing $\kappafunc$ into $k$ points uniformly over $[-\pi, \pi]$. Further, if we discretize $B$ similarly, we can obtain a \textbf{second-order cone program} with $k$ conic constraints and $O(k)$ decision variables. As $k \to \infty$, the solution approaches the solution to \eqref{eq:IQCbound}. Empirically, we observe that the values stabilize quickly as $k$ increases. We stop the computation when the change in bound as a function of $k$ drops below a threshold --- this gives $k=1000$.

    Further, given the optimal $\kappafunc=\kappafunc^\star$, we can run an alternating minimization where we minimize the objective of \eqref{eq:IQCbound} with respect to $\kappafunc$ for fixed $B$ and with respect to $B$ for fixed $\kappafunc$. This leads to an iteratively improving choice of $B$. We find empirically that this iterative procedure converges quickly and leads to a provable theoretical gap between the upper bounds on $\Fobj$ achievable by \privsgd and \privftrl.
   
    We numerically compare the bound \eqref{eq:IQCbound} for \privsgd and \ourprivftrl.
    \Cref{fig:IQCbound} shows that the gap between \privsgd and \ourprivftrl is multiplicative: the absolute gap grows with the increasing condition number $\kappa = L/\mu$.
    The suboptimality of ``Optimized'' \privftrl (optimized as described above) grows even more slowly with $\kappa$.

    Overall, \ourprivftrl significantly improves upon \privsgd and has only a single tunable parameter $\nu$ and no expensive computation to generate the noise coefficients. We focus on \ourprivftrl for experiments in this paper but leave the possibility of improving results further based on Optimized \privftrl for future work.

\section{Experiments}
\label{sec:expt}

\begin{table}[t]
    \centering
    \small
    \renewcommand{\arraystretch}{1.2}
    \centering
    \adjustbox{max width=\linewidth}{
    \begin{tabular}{cccccc}
    \toprule
        \textbf{\dpftrl Variant}  & \textbf{Citation} &  \textbf{Coeff. matrix $\bfB$ } 
        & \textbf{Anytime?} & \multicolumn{2}{c}{\textbf{Computation Cost}}\\
        & & & & Generation & Training (per step)\\
    \midrule
        {DP-SGD} & \cite{DP-DL} & Identity  
        & \checkmark & $O\br{1}$ & $O\br{1}$\\
        {Honaker/TreeAgg} & \cite{kairouz2021practical} & \makecell{Lower-Triangular (LT)} 
        & \checkmark & $O\br{1}$  &  $O\br{\log T}$\\
        {Optimal CC} &\cite{fichtenberger2023constant} & Toeplitz \& LT  & 
        \checkmark &   $O\br{1}$ & $O(T)$\\
        {\ourprivftrl} & Ours & Toeplitz \& LT &  
        \checkmark  & $O(1)$ & $O(T)$ \\
    \midrule
        {FFT} & \cite{choquette2023multi} & Toeplitz  
        & \No  & $O(1)$ & $O\br{T\log^2 T}$\\
        {Full Honaker} & \cite{honaker_trick} & Arbitrary  
        & \No  & $O(T^2)$ & $O(T^2)$\\
        {Multi-Epoch (ME)} & \cite{choquette2023multi} & Arbitrary 
        & \No & $O\br{T^3}$ &  $O\br{T^2}$ 
        \\
    \bottomrule
    \end{tabular}
    } %
    \caption{\small 
    \textbf{Variants of \privftrl}:
    the noise coefficient matrix $\bfB$ and  whether the coefficient matrix $\bfB$ can be created/optimized agnostic to the time horizon $T$ (denoted as ``\textbf{Anytime}''), and the computation cost.
    }
    \label{tab:exp_algorithms}
    
\end{table}

We demonstrate the practical benefits of \ourprivftrl for deep learning tasks. This approach has a single tunable parameter $\nu$ that can easily be tuned based on minimizing the squared error~\eqref{eq:prefix-error} as in prior work.

\mypar{Comparing Computation~(\Cref{tab:exp_algorithms})} 
While optimized noise coefficient matrices (e.g. ``ME'' in~\Cref{tab:exp_algorithms}) have the state-of-the-art privacy-utility tradeoffs in private learning (without amplification), their  computational cost scales as $O(T^3)$ for $T$ iterations.\footnote{
    Note that in practice we take $T$ to be the number of steps of minibatch gradient descent, effectively doing several epochs over the data which differs from the theoretical setting considered in previous sections.
}
For example, generating the coefficient matrix $\bfB$ for $T=10^4$ takes around $24$ hours~\cite{choquette2023multi}. Moreover, it has a $O(T^2)$ cost per step. We find in this section that \ourprivftrl achieves near state-of-the-art privacy-utility tradeoffs at a much smaller computational cost of $O(T)$ per iteration.\footnote{
    We note that follow-up work~\cite{mcmahan2024efficient} has demonstrated that the $O(T)$ per-iteration cost of \dpftrl with Toeplitz noise coefficient matrices can further be reduced to $O(k)$ for any constant $k$ at an additional $\exp(-\sqrt{k})$ factor in the error.
}

\begin{figure}[t]
  \centering
  \begin{subfigure}[b]{0.49\textwidth}
    \centering
    \includegraphics[width=0.75\linewidth]{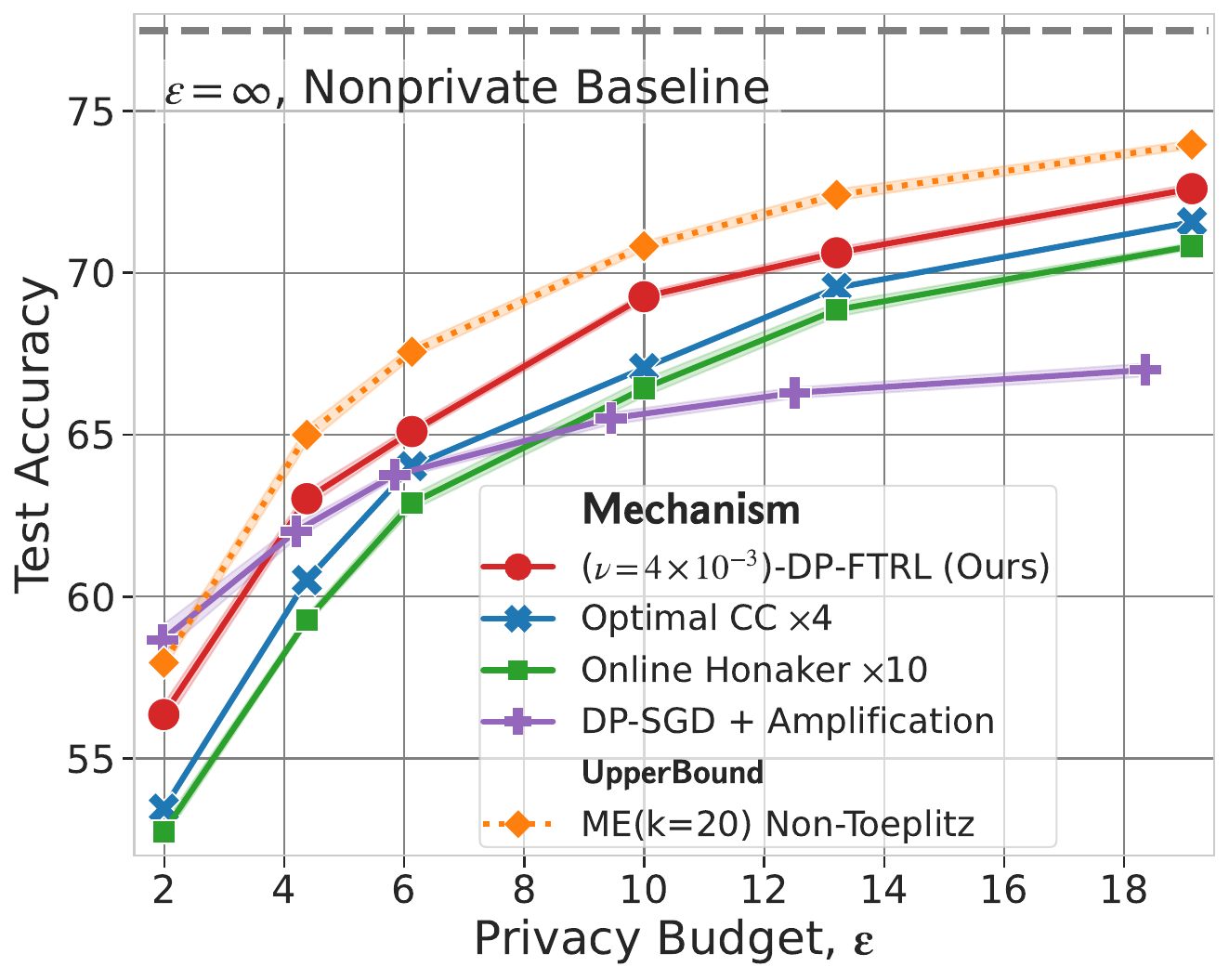}
    \caption{
    \footnotesize Example-level DP on CIFAR-10 (image classification).}
    \label{fig:cifar}
  \end{subfigure}
  \begin{subfigure}[b]{0.5\textwidth}
    \centering
    \includegraphics[width=0.75\linewidth]{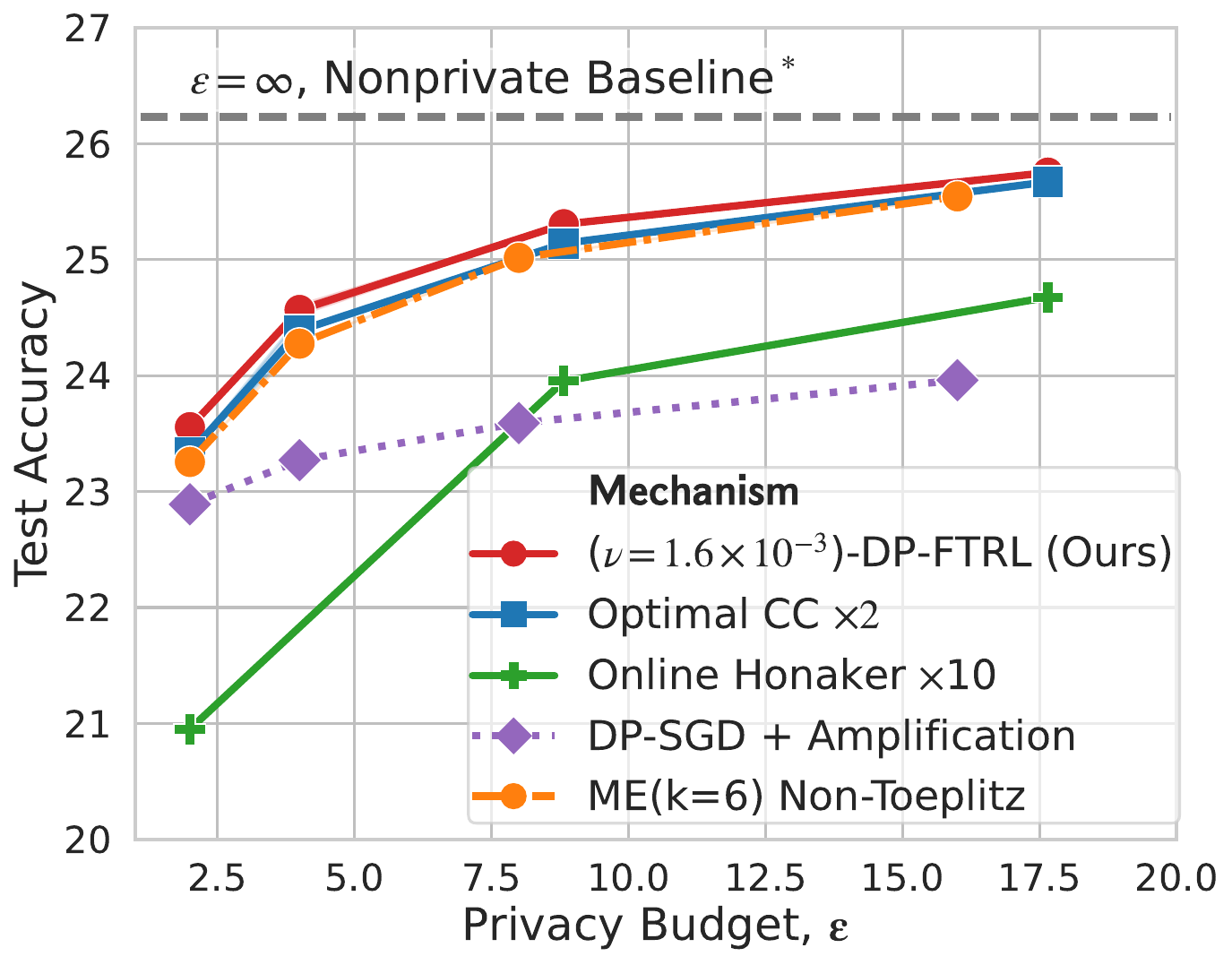}
    \caption{\footnotesize User-level DP on StackOverflow (language modeling).}
    \label{fig:sonwp}
  \end{subfigure}
  \caption{\small
  \textbf{The proposed \optimallti outperforms all other efficient and anytime mechanisms.} It also nearly equals or slightly outperforms the state-of-the-art ``ME'' mechanism that requires significantly more compute (cf. \Cref{tab:exp_algorithms}).
  $^*$The non-private baseline for StackOverflow uses per-user clipping as this improves performance by $\approx0.5\%$ pp.
  }
  \vspace{-0.4cm}
  \label{fig:deep-learning}
\end{figure}

We compare with other \emph{anytime} approaches listed in \Cref{tab:exp_algorithms} for which the noise coefficient matrices $\bfB$ can extended to any time horizon $T$. The practitioner then need not specify $T$ in advance, but rather, can train for as long as necessary to achieve minimal model loss or error. In non-private training, it is common to let algorithms run until certain stopping conditions, like a maximum difference on the train-test loss, are met~\cite{morgan1989generalization}.
Moreover, general matrices $\bfB$ become prohibitive in terms of compute/memory as models scale up~\cite{kaplan2020scaling,anil2023palm}.

The DP-SGD baseline we compare to has the additional benefit of privacy amplification by sampling to make it a stronger baseline. On the other hand, the correlated noise algorithms are considered without amplification.

\mypar{Experiment Setup} 
We use two standard benchmarks: example-level DP for image classification on the CIFAR-10 dataset and user-level DP for language modeling on the StackOverflow dataset.
We use the same setup as \cite{kairouz2021practical}.
We also stamp/restart all baselines as suggested in~\cite{choquette2023multi}. 
This gives the baselines the advantage of an additional tuning parameter (tuned to minimize the squared error \eqref{eq:prefix-error}), but does not affect their per-step training cost. We denote this by the suffix ``$\times S$'' for $S > 1$ in the plot. 
We tune all CIFAR-10 hyperparameters with a grid search, while we use hyperparameters reported from previous works for StackOverflow. 
\Cref{sec:b:empirical-setup} gives the full setup. 

\mypar{Main Results} 
Across both datasets, \optimallti outperforms all existing anytime mechanisms by a significant margin (\Cref{fig:cifar}).
We find an average $3$pp improvement that grows as $\epsilon$ becomes small. 
Indeed, the proposed \ourprivftrl makes up 30-80\% of the gap between previous efficient approaches and the state-of-the-art and computationally intense ME approach. For instance, at $\varepsilon=10$, we have \ourprivftrl at $69.26\%$ nearly matches ME at $70.83\%$.
In particular, \ourprivftrl outperforms Optimal CC~\cite{fichtenberger2023constant}, which is equivalent to \ourprivftrl with $\nu = 0$; this shows the practical importance of the exponential decay parameter $\nu$ in Eq.~\eqref{eq:optimal_beta}.
For StackOverflow, we find that \ourprivftrl outperforms the state-of-the-art ME across all $\epsilon$ (\Cref{fig:sonwp}) by $\approx0.3\%$-points while requiring significantly less computation.

As $\epsilon$ becomes small, DP-SGD can outperform DP-FTRL due to privacy amplification. We find that \ourprivftrl outperforms DP-SGD for $\epsilon\geq4$ on CIFAR-10 ($63.02\%$ vs. $62.02\%$) and around $\varepsilon \approx 2$ for StackOverflow ($23.6\%$ versus $22.6\%$), showing its broad applicability. Finally, we observe that \ourprivftrl nearly matches the non-private baselines on StackOverflow. A model trained via \ourprivftrl gets $25.3\%$ validation accuracy at $\varepsilon=8$, a mere $1\%$-point off from the non-private baseline. 
\section{Conclusion}
\label{sec:conc}
This work shows a clear separation between the noisy training dynamics with uncorrelated (\privsgd) and correlated noise (\privftrl) for linear regression. The matching upper/lower bounds reveal that \privftrl has a better dependence than \privsgd on problem parameters such as the effective dimension and condition number.
Inspired by the theory, we propose \ourprivftrl and validated its empirical performance on two DP tasks spanning image and language modalities. We found it can compete the state-of-the-art while circumventing the need for any expensive computations like the semi-definite programs used in prior work.
This work opens up several exciting directions including leveraging correlated-noise mechanisms for instance-optimal bounds and further improving the computational efficiency to enable large-scale private training.

\subsubsection*{Acknowledgements}
The authors thank H. Brendan McMahan, Fabian Pedregosa, Ian R. Manchester, Keith Rush, and Rahul Kidambi for fruitful discussions and helpful comments.

\bibliographystyle{iclr2024_conference}
\bibliography{reference}

\clearpage
\appendix
\appendix
\addcontentsline{toc}{section}{Appendix} %
\part{Appendix} %
\parttoc %
\clearpage

\section{Further Background on \dpftrl}
\label{sec:a:background}
In this appendix, we give a more detailed background of \dpftrl, and its exact notion of differential privacy.

\subsection{\dpftrl: The Matrix Mechanism for Private Learning}

The \dpftrl algorithm~\cite{kairouz2021practical,denisov2022improved} is obtained by adapting the matrix mechanism, originally designed for linear counting queries~\cite{li2015matrix}, to optimization with a sequence $(\bfg_0, \ldots, \bfg_{T-1})$ of gradient vectors.

\Cref{alg:dpmf} gives a detailed description of \dpftrl.
We give an alternate description of \dpftrl with an invertible lower-triangular noise coefficient matrix $\bfB \in \reals^{T \times T}$.
Denoting $\bfC = \bfB^{-1}$, the iterates of \dpftrl are generated by the update
\begin{align}
\begin{pmatrix}
\bftheta_1 \\ 
\vdots  \\
\bftheta_T
\end{pmatrix} = \begin{pmatrix}
\bftheta_0 \\
\vdots  \\
\bftheta_{T-1}
\end{pmatrix} - \eta 
\bfB\br{\bfC \begin{pmatrix}
\bfg_0 \\
\vdots  \\
\bfg_{T-1}
\end{pmatrix} + 
\begin{pmatrix} 
\bfw_0 \\ 
\vdots \\
\bfw_{T-1}
\end{pmatrix}
}
\label{eq:update}
\end{align}
where $\eta$ is a learning rate and
$\bfw_t \sim \calN(\boldzero, G^2\sigmadp^2 \bfI_d)$ is i.i.d. Gaussian noise with a noise multiplier $\sigmadp$ and $G$ is the $\ell_2$ clip norm.

Following prior work, we also refer to $\bfB$ as the \emph{noise correlation matrix} or \emph{noise coefficient matrix}. This is because the effective noise that is added to the optimization is the i.i.d. noise $(\bfw_0, \ldots, \bfw_{T-1})$ which are linearly correlated by the rows of the matrix $\bfB$.
It is also common in the literature to refer to $\bfC$ as the \emph{encoder}, while $\bfB$ is referred to as the \emph{decoder}.

This privacy of \eqref{eq:update} can be seen as a postprocessing of a single application of the Gaussian mechanism. Let $\bfG, \bfW \in \reals^{T\times d}$ denote the matrix where each row is the gradient $\bfg_t$ (and respectively the noise $\bfw_t$).
Then, \eqref{eq:update} is effectively the postprocessing of one run of the Gaussian mechanism $\bfC \bfG + \bfW$. Under a neighborhood model that can change one row of $\bfG$, it can be seen that the maximum sensitivity of this operation is $\max_t \|\bfC_{:,t}\|_2^2$~\cite{denisov2022improved}. This sensitivity logic also holds for adaptively chosen gradients; we postpone a formal description to \Cref{sec:a:adaptive-dp-review}.

\mypar{Connection to the exposition in prior work}
Prior work introduced \dpftrl differently. Letting $\bfA \in \reals^{T \times T}$ denote the lower triangular matrix of all ones, update \eqref{eq:update} can also be written as
\begin{align} \label{eq:update2}
\begin{pmatrix}
\bftheta_1 - \bftheta_0 \\ 
\vdots  \\
\bftheta_T - \bftheta_0 
\end{pmatrix} = 
- \eta 
\tilde\bfB\br{\bfC \begin{pmatrix}
\bfg_0 \\
\vdots  \\
\bfg_{T-1}
\end{pmatrix} + 
\begin{pmatrix} 
\bfw_0 \\ 
\vdots \\
\bfw_{T-1}
\end{pmatrix}
}\,,
\end{align}
where $\tilde \bfB = \bfA \bfB$.
The equivalence between \eqref{eq:update} and \eqref{eq:update2} can be seen by multiplying \eqref{eq:update} by $\bfA$, which is also equivalent to taking the cumulative sum of the rows of a matrix.
In this notation, the objective from \eqref{eq:prefix-error} used in previous work to find the matrix $\bfB$ can equivalently be written as
\[
    \varphi(\bfB) = \| \tilde\bfB\|_F^2 = \| \bfA \bfB\|_F^2 \,.
\]

\mypar{\dpftrl with Toeplitz matrices}
We focus on the class of lower-triangular and Toeplitz matrices $\bfB$. That is,
$[\bfB]_{t, t'}=\beta_{t-t'}$ for all $t \geq t'$ where $\bfbeta = (\beta_0, \ldots, \beta_{T-1})$ is the first column of $\bfB$.\footnote{This implies that $\bfC = \bfB^{-1}$ is also lower-triangular and Toeplitz~\cite[Prop. 2.2 \& Rem. 2.3]{kucerovsky2016some}.}
In this case, \eqref{eq:update} reduces to this simple update:
\begin{align} \label{eq:dp-ftrl:setup}
        \bftheta_{t+1} = \bftheta_t - \eta \left( \bfg_t + \sum_{\tau=0}^t \beta_{\tau} \bfw_{t-\tau} \right) \,.
\end{align}
This lets us study \dpftrl as a time-invariant stochastic process and characterize its stationary behavior.

\subsection{Differential Privacy in Adaptive Streams}
\label{sec:a:adaptive-dp-review}

\mypar{Neighboring streams}
We consider learning algorithms as operating over streams of gradients $\bfg_0, \bfg_1, \ldots \in \reals^d$. We consider differential privacy (DP) under the ``zero-out'' notion of neighborhood~\cite{kairouz2021practical}. 
Two streams $\bfG = (\bfg_0, \ldots, \bfg_{T-1})$ and $\bfG' = (\bfg_0', \ldots, \bfg_{T-1}')$ of length $T$ are said to be neighbors 
if $\bfg_\tau = \bfg_\tau'$ for all positions
$\tau \le T-1$ except possibly one position $t$ where one of $\bfg_t$ or $\bfg_t'$ is the zero vector.

The zero-out neighborhood is standard in prior works on \dpftrl~\cite[e.g.][]{kairouz2021practical,denisov2022improved}.
For a further discussion of different notions of neighborhood, we refer to \cite[Sec. 2.1.1]{ponomareva2023dp}.
This guide suggests that the semantics of the zero-out neighborhood are roughly the same as that of the usual add/remove notion of neighborhood.

\mypar{DP with adaptive continual release}
It is customary to formalize DP with adaptive streams as a privacy game between a mechanism $\calM$ and a privacy adversary $\calA$. This is known as the \textit{adaptive continual release setting}~\cite{jain2023price}. The game makes a binary choice $b\in\{0, 1\}$ ahead of time --- this remains fixed throughout and is not revealed to either $\calM$ or $\calA$.
Each round $t$ consists of four steps:
\begin{itemize}[nosep,leftmargin=1.5em,topsep=-0.3em]
    \item $\calM$ sends the current model parameters $\bftheta_t$ to the adversary $\calA$;
    \item $\calA$ generates two gradient vectors $\bfg_t, \bfg_t'$ (e.g. as $\grad f(\bftheta_t; \bfz_t)$ for $\bfz_t \sim \Pdata$ or simply the zero vector);
    \item the game accepts these inputs if the partial streams $(\bfg_0, \ldots, \bfg_t)$ and $(\bfg_0', \ldots, \bfg_t')$ are neighbors;
    \item $\calM$ receives $\bfg_t$ if $b=0$ else $\bfg_t'$.
\end{itemize}
DP in this setting requires that the adversary cannot infer the value of $b$, i.e., the distribution of $\bftheta_{0:T} | b=0$ to be ``close'' to that of $\bftheta_{0:T} | b=1$ (where the definition of ``closeness'' depends on the DP variant).
For instance, $(\eps, \delta)$-DP~\cite{DMNS} requires for each $b \in \{0, 1\}$ and any outcome set $S$ that
\begin{align*}
    \mathbb{P}(\bftheta_{0:T} \in S \, | \, b) 
    \le \exp(\eps) \, \mathbb{P}(\bftheta_{0:T} \in S \, | \, 1 - b) + \delta \,.
\end{align*}

Similarly, $\rho$-zCDP~\cite{bun2016concentrated} in this setting requires that the Rényi $\alpha$-divergence between the distribution $P_{0}$ of $\bftheta_{0:T} | b=0$ and the distribution $P_1$ of $\bftheta_{0:T}|b=1$ are close:
\begin{align*}
    D_\alpha(P_0 \Vert P_1) \le \rho \alpha
\end{align*}
for all $\alpha \in (0, \infty)$.
Following standard arguments~\cite[e.g.][]{balle2020hypothesis}, $\rho$-zCDP in this setting implies $(\eps_\delta, \delta)$-DP with
\[
    \eps_\delta \le \inf_{\alpha > 1} \left\{
        \rho \alpha + \frac{1}{\alpha - 1} \log\left(\frac{1}{\alpha\delta}\right) + \log(1 - \alpha^{-1}) \,.
    \right\}
\]

\dpftrl satisfies a zCDP guarantee as described in \Cref{thm:dp} in \S\ref{sec:intro}.
This guarantee is equivalent to the one obtained by interpreting \eqref{eq:update} as the postprocessing of one run of the Gaussian mechanism $\bfC \bfG + \bfW$.
 
\section{Asymptotics of \privftrl for Mean Estimation}
\label{sec:a:mean}
We now prove \Cref{thm:mean} on mean estimation.

\begin{proof}[Proof of \Cref{thm:mean}]
We rewrite the iterates of \dpftrl as a linear time-invariant (LTI) dynamical system, whose stationary variance can be analyzed in the Fourier domain directly.

\mypar{Notation}
Since $|\grad f(\theta ;z)| = |z| \le 1$ and $G \ge 1$, there is no gradient clipping. We consider a mean-adjusted version of the learning dynamics: let $\delta_t = \theta_t - \expect[z]$
and $u_{t} = \frac{z_t - \expect[z]}{\sigmasgd}$. This allows us to reason about the deviation of the parameters $\theta_t$ from the true mean $\expect[z]$; indeed, it turns out that $\lim_{t \to \infty} \expect[\delta_t] = 0$. The objective we optimize for can now be succinctly written as
    $\lim_{t \to \infty} \expect[\delta_t^2]$.

\mypar{LTI System}
Our next step is to write this as an LTI system (see \Cref{sec:a:lti} for a review).
Thus, the sequence $(\delta_t)_{t=0}^\infty$ produced by \eqref{eq:dpftrl:intro} evolves as
\begin{align}
    \delta_{t+1} = (1 - \eta) \delta_t + \eta  \sigmasgd u_{t}  - \eta \sigmadp G \sum_{\tau=0}^t \beta_\tau w_{t-\tau} \, \quad t=0,1, \ldots. \label{eq:dpftrl_dynamics}
\end{align}
This is an LTI system with input $\bfx_t = (u_{ t} ; w_t) \in \reals^2$ and output $\bfy_t = [\delta_t] \in \reals^1$. We can verify its asymptotic stability by examining the dynamics under zero inputs: $u_{t} = 0$ and $w_t = 0$ for all $t$. This gives $\delta_t = (1 - \eta)^t \delta_0 \to 0$ as $t \to \infty$. Thus, this system is asymptotically stable.
Further, we can also get from taking expectations that $\expect[\delta_t] = (1-\eta)^t \delta_0 \to 0$. 
Thus, our objective
$\Fobj\br{B}=\lim_{t \to \infty} \expect[\delta_t^2]$
is the limiting (stationary) variance of $\delta_t$.

To invoke results from the LTI literature, it is convenient to re-index time to start from $t=-\infty$ so that the behavior at $t=0$ describes the stationary behavior. Hence, the dynamics can be replaced by
\begin{align}
    \delta_{t+1} = (1 - \eta) \delta_t + \eta  \sigmasgd u_{t}  - \eta \sigmadp G \sum_{\tau=0}^\infty \beta_\tau w_{t-\tau} \, \quad \forall \, t \in \mathbb{Z} \label{eq:dpftrl_dynamics_inf}
\end{align}
where $\mathbb{Z}$ denotes the set of integers and the objective can be taken to be $F_\infty(B) = \expect[\delta_0^2]$.

\mypar{Transfer function of the LTI system}
The transfer function $\bfG(\omega)$ of the LTI system \eqref{eq:dpftrl_dynamics_inf} is a complex matrix of shape $1 \times 2$ (see \S\ref{sec:a:lti} for definitions), which can be written as
\begin{align} \label{eq:mean-transfer-fn}
    \bfG(\omega) = 
    \begin{pmatrix}
    \frac{-\eta}{1 - \eta - \exp(\I\omega)}
    &
    \frac{\eta\, B(\omega)}{1 - \eta - \exp(\I\omega)}
    \end{pmatrix} \,.
\end{align}
The transfer function has the property that for any input sequences $u_t$ and $w_t$ with DTFT $U(\omega)$ and $Z(\omega)$, the output sequence satisfies $Y(\omega) = \bfG(\omega) \begin{pmatrix} U(\omega) \\ Z(\omega) \end{pmatrix}$.

\mypar{Stationary variance of the LTI system}
The stationary variance $ \lim_{t \to \infty} \expect[\delta_t^2]$ admits a nice closed form expression in the Fourier domain since its inputs are white noise. 
In particular, $u_t$ is i.i.d. in each step and independent of the DP noise $w_t$, so that the power spectral density of the sum of these two noise sources is simply the sum of the power spectral densities of the individual sources; the resulting expression is summarized in \Cref{thm:lti-covariance}.

We first calculate the input covariance is 
\begin{align} 
\label{eq:mean-est:stationary-var}
    \bfSigma = \expect[\bfx_t \otimes \bfx_t] = \begin{pmatrix}
        \sigmasgd^2 & 0 \\
        0 & G^2 \sigmadp^2 
    \end{pmatrix} \,.
\end{align}
We can then use \Cref{thm:lti-covariance} from \S\ref{sec:a:lti} to obtain an expression for the stationary variance $F_\infty(B) = \expect[\delta_0^2]$:
\[
    \Fobj(B) =
    \frac{1}{2\pi} \int_{-\pi}^\pi \bfG(\omega) \bfSigma \bfG(\omega)^*  \, \, \D \omega
    = 
    \frac{ \eta^2}{2\pi} \int_{-\pi}^{\pi} \frac{|B\br{\omega}|^2 \frac{G^2}{2\rho}\gamma_\infty^2\br{B} + \sigmasgd^2}{|1-\eta-\exp\br{\I \omega}|^2} \, \D\omega\,.
\]
Note that above $\bfG(\omega)^*$ denotes the conjugate transpose of the complex matrix $\bfG(\omega)$.

\mypar{Optimizing for the noise coefficients in frequency domain}
The dependence of $\Fobj$ on $B$ is via the first term:
\begin{align*}
& \frac{ \eta^2}{2\pi} \int_{-\pi}^{\pi} \frac{|B\br{\omega}|^2 \frac{G^2}{2\rho}\gamma_\infty^2\br{B}}{|1-\eta-\exp\br{\I \omega}|^2} \, \D\omega\,
 \stackrel{\eqref{eq:sens-fourier-main}}{=} \frac{ \eta^2 \frac{G^2}{2\rho}}{4\pi^2} \br{ \int_{-\pi}^{\pi} \frac{|B\br{\omega}|^2}{|1-\eta-\exp\br{\I \omega}|^2} \, \D\omega} \br{\int_{-\pi}^{\pi} \frac{\D\omega}{|B\br{\omega}|^2}} \,.
 \numberthis \label{eq:optimal-B-mean-est}
\end{align*}

The stationary variance's dependence on $B$ in \eqref{eq:optimal-B-mean-est} is a product of a linear function of $|B|^2$ and $\frac{1}{|B|^2}$. The former comes via the variance and the latter through the sensitivity $\gamma_\infty\br{B}$ via \eqref{eq:sens-fourier-main}. The optimal value of $B$ must balance these two considerations. By the Cauchy-Schwarz inequality, the product is minimized when 
\begin{align} \label{eq:mean-est-bstar}
    \frac{|B^\star\br{\omega}|^2}{|1-\eta-\exp\br{\I \omega}|^2} = \frac{1}{|B^\star\br{\omega}|^2}
    \iff |B^\star\br{\omega}| = |\sqrt{1-\eta - \exp\br{\I \omega}}|\,,
\end{align}
and the minimum value is equal to 
\[
    \frac{\eta^2 G^2 \sigmadp^2}{4 \pi^2} 
    \br{
    \int_{-\pi}^\pi \frac{\D\omega}{|1-\eta - \exp\br{\I \omega}|}}^2 \,.
\]
The proof of the error bound now follows by computing and bounding the integral $\int_{-\pi}^\pi \D\omega / | 1 - \eta - \exp(\I \omega)|$. This can be bounded via reductions to standard integrals whose asymptotics are known (see \Cref{lem:sensitivity-ellip} and \Cref{prop:ellipk:asymp} from \S\ref{sec:a:ellip-def}). Similarly, \Cref{lem:integral-easy} can be used to bound the $\sigmasgd^2$ term in \eqref{eq:mean-est:stationary-var}.

\mypar{Optimal noise coefficients in time-domain}
Next, we derive the time-domain description by taking $B^\star(\omega) = \sqrt{1 - (1 - \eta) \exp(-\I \omega)}$ (which amounts to fixing a phase in \eqref{eq:mean-est-bstar} above). We use the Maclaurin series expansion
$\sqrt{1 + z} = \sum_{t=0}^\infty {1/2 \choose t} z^t$ of the square root function
to get
\[
    B^\star(\omega) = 
    \sum_{t=0}^\infty (-1)^t {1/2 \choose t} (1 - \eta)^t \exp(-\I\omega t) \,.
\]
Comparing this to the definition of the discrete-time Fourier transform $B^\star(\omega) = \sum_{t=0}^\infty \beta^\star_t \exp(-\I \omega t)$ gives the claimed expression for $\bfbeta^\star$. 
\end{proof}

Note that the optimal noise coefficients scale as $|\beta_t^\star| = \Theta(t^{-3/2} \exp(-\eta t))$.

\section{Asymptotics of \privftrl for Linear Regression}
\label{sec:a_asymptotic_lr}

The goal of this section is to prove \Cref{thm:lr}. The proof relies heavily on the following matching upper and lower bounds on the stationary error of \noisyftrl with any noise coefficients $\bfbeta$ in the frequency domain using its discrete-time Fourier transform (DTFT) $B$ as:
\begin{align}
\label{eq:finf:matching}
        F_\infty(B)
        &= \Theta\left( 
            \eta \sigmasgd^2 \tr{\bfH}
            + \eta^2 G^2 \rho^{-1} \gamma_\infty^2(B) \int_{-\pi}^\pi |B(\omega)|^2 h(\omega) \, \D \omega
        \right) \,,
\end{align}
where the function $h:[-\pi, \pi] \to \reals$ depends on the eigenvalues $\lambda_1, \ldots, \lambda_d$ of the input covariance $\bfH$:
\begin{align} \label{eq:def:h}
    h(\omega) = \sum_{j=1}^d \frac{\lambda_j}{|1 - \exp(\I\omega) - \eta\lambda_j|^2}\,.
\end{align}

The outline of the section is
\begin{itemize}
    \item \textbf{\Cref{sec:stationary-setup}}: Setup, including notation, and assumptions.
    \item \textbf{\Cref{sec:stationary:upper-bound}}: Proofs of the upper bound of \eqref{eq:finf:matching}, specifically \Cref{thm:lr-frequency-domain:appendix}  (see also \Cref{thm:lr-time-domain:appendix} for the time-domain description).
    \item \textbf{\Cref{sec:stationary:lower-bounds-proofs}}: Proofs of the lower bound of \eqref{eq:finf:matching}, specifically \Cref{thm:lr-lower-bound}.
    \item \textbf{\Cref{sec:tuned_dp_ftrl_proof}}: Asymptotics of \ournoisyftrl.
    \item \textbf{\Cref{sec:two-step-noise}}: Asymptotics of anti-PGD (see \Cref{tab:ComparisonNoisy}).
    \item \textbf{\Cref{sec:a:edim}}: Effective Dimension and its Connection to the Stable Rank.
    \item \textbf{\Cref{sec:additional-proofs}}: Proofs of intermediate technical results.
\end{itemize}

The separation between \noisysgd and \ournoisyftrl is further illustrated in \Cref{tab:rates-noisy:decays}. Following common practice~\cite[e.g.][]{caponnetto2007optimal}, we compare the rates for various regimes of eigenvalue decays for $\bfH$.

\begin{table}[t]
\caption{\small
Asymptotic suboptimality of \noisysgd and \noisyftrl for linear regression with Gaussian inputs based on the eigenvalues $\lambda_k$ of the Hessian $\bfH$. We give the bounds in terms of the learning rate $\eta$, dimension $d$, the effective dimension $\edim = \tr{\bfH}/\norm{\bfH}_2$, and the noise variance $\rho^{-1}$ representing the privacy level. We take $G=1$ and $\norm{\bfH}_2 = 1$ w.l.o.g.
\noisyftrl is always better at large dimension $d$ or small learning rate $\eta$.
}\label{tab:rates-noisy:decays}
\centering
\small
\renewcommand{\arraystretch}{2}
\adjustbox{max width=\textwidth}{%
\begin{tabular}{ccccc}
\toprule
Eigenvalues of $\bfH$ & Effective dim. $\edim$ & \noisysgd & \noisyftrl & Ratio of $\frac{\text{\noisyftrl}}{\text{\noisysgd}}$ \\
\midrule
$\lambda_k = 1$ & $d$ &  $\eta d \rho^{-1}$ & $\eta^2 d \rho^{-1} \log^2(\frac1\eta)$ & $\eta \log^2(\frac1\eta)$ \\
\midrule
$\lambda_k = 1/\sqrt{k}$
& $\sqrt{d}$ &
$\eta d \rho^{-1}$ &
$\eta^2 \sqrt{d} \rho^{-1} \log^2(\frac{d}{\eta})$ & 
$\frac{\eta}{\sqrt{d}} \log^2(\frac{d}{\eta})$ \\
$\lambda_k = k^{-a}\,\, (a < 1)$
& $\frac{d^{1-a}}{1-a}$ &
$\eta d \rho^{-1}$ &
$(1-a)^{-1}\eta^2 d^{1-a} \rho^{-1} \log^2(d/\eta)$  &
$\frac{\eta}{(1 - a) d^a} \log^2\left( \frac{d}{\eta}\right)$
\\
\midrule
$\lambda_k = 1/k$ &
$\log d$ &
$\eta d \rho^{-1}$ &
$\eta^2 \rho^{-1} \log^3(\frac{d}{\eta})$ & 
$\frac{\eta}{d} \log^3(\frac{d}{\eta})$
\\
\midrule
$\lambda_k = 1/k^2$ &
constant &
$\eta d \rho^{-1}$ &
$\eta^2 \rho^{-1} \log^2(\frac{d}{\eta})$ &
$\frac{\eta}{d} \log^3(\frac{d}{\eta})$
\\
$\lambda_k = k^{-a} \,\, (a > 1)$ &
$\frac{a}{a-1}$ &
$\eta d \rho^{-1}$ &
$\left(\frac{a^2}{a-1}\right) \eta^2 \rho^{-1} \log^2 \left(\frac{d}{\eta}\right)$ &
$\left(\frac{a^2}{a-1}\right) \frac{\eta}{d} \log^2\left(\frac{d}{\eta} \right)$ \\
\bottomrule
\end{tabular}
} %
\end{table}

\subsection{Setup, Assumptions, and Notation}
\label{sec:stationary-setup}

\subsubsection{Setup}
Recall that we wish to minimize the objective
\begin{align} \label{eq:linear-regression}
    F(\bftheta) = \expect_{(\bfx, y) \sim \Pdata} \left[ 
        (y - \inp{\bftheta}{\bfx})^2
    \right] \,.
\end{align}

\mypar{Stochastic gradients}
Given $(\bfx, y) \sim \Pdata$, the vector 
\[
    \bfg := (\bfx \otimes \bfx) \, \bftheta - y \bfx = (\bfx \otimes \bfx)(\bftheta - \bftheta_\star) - \xi \bfx 
\]
is a stochastic gradient of $F$ at $\bftheta$, i.e., $\expect[\bfg] = \nabla F(\bftheta)$.

\mypar{\noisyftrl Iterations}
We specialize the \noisyftrl algorithm with Toeplitz noise coefficients.
Let $T$ denote the number of iterations and $\bfbeta_{:T} = (\beta_0, \ldots, \beta_{T-1})$ denote the first column of the Toeplitz matrix $\bfB = \toeplitz(\bfbeta_{:T})  \in \reals^{T \times T}$. Starting from a given $\bftheta_0 \in \reals^d$, \noisyftrl samples a fresh input-output pair $(\bfx_t, y_t)\sim \Pdata$ and noise $\bfw_t$ to set 
\begin{align} \label{eq:dp-ftrl-finite}
    \bftheta_{t+1} = \bftheta_t - \eta\left( (\bfx_t \otimes \bfx_t) \bftheta_t - y_t \bfx_t)\right) - \eta \sum_{\tau=0}^t \beta_\tau \bfw_{t-\tau} \,.
\end{align}
Recall that the sensitivity $\gamma_T(\bfbeta)$ equals to the maximum columns norm of $\bfB^{-1} = (\toeplitz(\bfbeta))^{-1}$:
\begin{align} \label{eq:sensitivity:finite-time}
    \gamma_T(\bfbeta) = \max_{\tau = 0, \ldots, T-1} \norm{\bfB^{-1} \bfe_\tau }_2  \,,
\end{align}
where $\bfe_\tau = \big(\mathbb{I}(j = \tau)\big)_{\tau = 0}^{T-1} \in \reals^T$ is a standard basis vector.
Note that the submatrix $[\bfB^{-1}]_{0:m, 0:m}$ of the first $m$ rows and columns of $\bfB^{-1}$
equals $\br{\toeplitz(\beta_0, \ldots, \beta_{m-1})}^{-1}$. Thus, the sensitivity $\gamma_t(\bfbeta)$ is an increasing function of $t$ always.

\mypar{Infinite-time limit of \noisyftrl}
We study the \noisyftrl error under the limit $T \to \infty$ with an infinite sequence $\bfbeta = (\beta_0, \beta_1, \ldots)$ of weights. 

It is also convenient to re-index time to start from $t=-\infty$ and consider the sequence $(\bftheta)_{t=-\infty}^\infty$ produced by analogue of \Cref{eq:dp-ftrl-finite}, which reads
\begin{align} \label{eq:dp-ftrl-inf-original}
    \bftheta_{t+1} = \bftheta_t - \eta\left( (\bfx_t \otimes \bfx_t) \bftheta_t - y_t \bfx_t)\right) - \eta \sum_{\tau=0}^\infty \beta_\tau \bfw_{t-\tau} \,.
\end{align}
Note that this includes a summation over all previous DP noise $(\bfw_\tau)_{\tau=-\infty}^t$.
For this sum to have finite variance, we require $\sum_{\tau=0}^\infty \beta_\tau^2 < \infty$ or that $\bfbeta \in \ell^2$, the space of all square-summable infinite sequences. We will assume this holds throughout.

\mypar{Sensitivity in the infinite limit}
We define the sensitivity $\gamma_\infty(\bfbeta)$ by considering the linear operator $\bfB = \toeplitz(\bfbeta)$ as the convolution operator $[\bfB \bfw]_t = \sum_{\tau=0}^\infty \beta_\tau \bfw_{t-\tau}$ on input $\bfw = (\bfw_\tau)_{\tau=-\infty}^\infty$.
Let $\bfB^{-1}$ be the inverse operator to $\bfB$, assuming it exists. Note that the column norms $\norm{\bfB^{-1} \bfe_\tau}_2$ from \eqref{eq:sensitivity:finite-time} become equal for all $\tau$ as $T \to \infty$. Thus, we get that the limiting sensitivity in the infinite time limit equals
\begin{align} \label{eq:sensitivity-inf-time}
    \gamma_\infty(\bfbeta) = \norm{\bfB^{-1} \bfe_0}_2 
\end{align}
for $\bfB = \toeplitz(\bfbeta)$ and $\bfe_0 = (\indi{\tau=0})_{\tau=0}^\infty \in \ell^2$. If $\bfe_0 \notin \text{Range}(\bfB)$, then we take $\gamma_\infty(\bfbeta) = \infty$.

\mypar{Frequency-domain description}
Our analysis relies on the frequency-domain representation $B:[-\pi, \pi] \to \C$ of $\bfbeta$ obtained via a discrete-time Fourier transform (DTFT) and defined as 
\begin{align} \label{eq:beta-dtft}
    B(\omega) = \sum_{t=0}^\infty \beta_t \, \exp(\I \omega t) \,.
\end{align}
The sequence $\bfbeta$ can be recovered from $B(\omega)$ using the inverse Fourier transform. 
Note that $\beta \in \ell^2$ is equivalent to $B \in L^2$, the space of square-integrable functions, by Parseval's theorem.
The sensitivity \eqref{eq:sensitivity-inf-time} can be defined in the Fourier domain as follows.

\begin{property} \label{prop:sensitivity-fourier}
    Let $B(\omega)$ denote the DTFT of $\bfbeta \in \ell^2$. Then, we have
    \begin{align}
        \gamma_\infty^2(\bfbeta)
        = \gamma_\infty^2(B) 
        := \frac{1}{2\pi} \int_{-\pi}^\pi \frac{\D\omega}{|B(\omega)|^2} \,.
    \end{align}
\end{property}
\begin{proof}
    Let $\bfz = \bfB^{-1} \bfe_0$ be the solution of the linear system $\bfB \bfz = \bfe_0$. Let $Z(\omega)$ denote the DTFT of $\bfz$. Since the linear operator $\bfB$ is a convolution with the weights of $\bfbeta$, this system can be expressed in the Fourier domain as
    \[
        B(\omega) Z(\omega) = \sum_{\tau=0}^\infty [\bfe_0]_\tau \exp(- \I \omega \tau) = 1 \,.
    \]
    Thus, $Z(\omega) = 1/B(\omega)$. We complete the proof with Parseval's theorem: $\|\bfz\|_2^2 = \frac{1}{2\pi} \int_{-\pi}^\pi |Z(\omega)|^2 \, \D \omega$.
\end{proof}

\subsubsection{Assumptions}
We prove the stationary error bounds under a relaxation of the assumptions in \S\ref{sec:linear-regr}.

\begin{assumption} \label{asmp:linear-regression}
    The data distribution $\Pdata$ satisfies the following:
    \begin{enumerate}[label=(\textbf{A\arabic*})]
        \item \label[asmp]{asmp:input}
        \textbf{Input Mean and Covariance}: The inputs have mean $\expect[\bfx] = \boldzero$ and covariance $\expect[\bfx \otimes \bfx] =: \bfH$. Further, $L = \lambda_1 \ge \cdots \ge \lambda_d =: \mu > 0$ are the eigenvalues of $\bfH$.
        \item \label[asmp]{asmp:noise}
        \textbf{Noise Mean and Variance}: 
        There exists a $\bftheta_\star \in \reals^d$ such that $y = \inp{\bftheta_\star}{\bfx} + \xi$ where $\xi$ is independent of $\bfx$ with $\expect[\xi] = 0$ and $\expect[\xi^2] \le \sigmasgd^2$.
        \item \label[asmp]{asmp:m4}
        \textbf{Input Kurtosis}: There exists $R^2 < \infty$ such that $\expect\left[ \norm{\bfx}_2^2 \, (\bfx \otimes \bfx) \right] \preceq R^2 \, \bfH$.
        Moreover, for every PSD $\bfP \in \mathbb{S}^d_+$ that commutes with $\bfH$ (i.e., $\bfP \bfH = \bfH \bfP$), we have 
        \ifnum \value{arxiv}= 0 {
         $\expect\left[(\bfx \otimes \bfx) \bfH^{-1/2} \bfP \bfH^{-1/2} (\bfx \otimes \bfx) \right] \preceq \kurt \, \tr{\bfP} \, \bfH$ 
        } 
        \else {
         \[
         \expect\left[(\bfx \otimes \bfx) \bfH^{-1/2} \bfP \bfH^{-1/2} (\bfx \otimes \bfx) \right] \preceq \kurt \, \tr{\bfP} \, \bfH
         \]
        } 
        \fi
       for some $\kurt < \infty$.
    \end{enumerate}
\end{assumption}

These assumptions are fairly standard in the context of linear regression.
\Cref{asmp:input} implies that the Hessian matrix of objective $F(\bftheta)$ is $\bfH \succ 0$. Thus, $F$ is $L$-smooth and $\mu$-strongly convex.
\Cref{asmp:noise} implies that $\bftheta_\star$ is the unique global minimizer of $F$ and that the linear model is well-specified. The upper bounds we prove continue to hold in the case where the linear model is mis-specified (i.e. $\xi$ is not independent of $\bfx$) but we still have $\expect[\xi^2\, (\bfx \otimes \bfx)] \preceq \sigmasgd^2 \bfH$.

\Cref{asmp:m4} is a kurtosis (i.e. 4th moment) assumption on the input distribution; we will momentarily show that it follows with absolute constants when $\bfx \sim \calN(\boldzero, \bfH)$. 
More generally, by taking a trace, we get from Jensen's inequality that $\tr{\bfH} \le R^2$.
The case of $\bfP = \bfI$ of the second part of \Cref{asmp:m4} has a special significance in the literature~\cite[e.g.][]{hsu2014random,jain2018accelerating} as $\kurt \tr{\bfI} = \kurt d$ is the number of samples that allows the spectral concentration of the empirical covariance to the population covariance $\bfH$.

\begin{property}\label{prop:m4-gaussian}
    if $\bfx \sim \calN(\boldzero, \bfH)$, we have that \Cref{asmp:m4} holds with $R^2 \le 3 \, \tr{\bfH}$ and $\kurt \le 3$.
\end{property}
\begin{proof}
    Let $\bfz = \bfH^{-1/2} \bfx$ be element-wise independent and distributed as a standard Gaussian.
    For the first part, denote $\bfM = \bfH^{-1/2}\, \expect[\norm{\bfx}_2^2 \bfx \otimes \bfx] \, \bfH^{-1/2} = \expect[\inp{\bfz}{\bfH\bfz} \bfz \otimes \bfz]$. Elementary properties of the standard Gaussian distribution give
    \[
        \expect[z_k z_l z_j^2] = \begin{cases}
            3\,, & \text{ if } k = l = j \\
            1 \,, & \text{ if } k = l \neq i \\
            0 \,, & \text{ if } k \neq l \,,
        \end{cases}
        \quad\text{and} \quad
        \expect[z_k z_l z_j z_{j'}] = \begin{cases}
            1\,, & \text{ if } k=j \text{ and } l = j' \\
            1\,, & \text{ if } k=j' \text{ and } l = j \\
            0\,, & \text{else}
        \end{cases}
    \]
    for $j\ne j'$.
    Thus, we have $\bfM = 2 \bfH + \tr{\bfH} \bfI$.
    This gives
    \[
        \expect[\norm{\bfx}_2^2 \bfx\otimes \bfx] = 
        \bfH^{1/2} \bfM \bfH^{1/2}
        = 
        2 \bfH^2 + \tr{\bfH} \bfH \preceq 3 \tr{\bfH} \bfH \,.
    \]
    
    For the second part, let $\bfH = \bfU \bfLambda \bfU^\top$ and $\bfP = \bfU \bfSigma \bfU^\top$ be the eigenvalue decomposition of $\bfH, \bfP$ respectively (since they commute, they are simultaneously diagonalized in the same basis given by the columns of $\bfU$). Since $\bfU^\top \bfz$ has the same distribution as $\bfz$ by the spherical invariance of Gaussians, we have,
    \begin{align}
        \bfH^{-1/2} \expect\left[ (\bfx \otimes \bfx) \bfH^{-1/2} \bfP \bfH^{-1/2}  (\bfx \otimes \bfx) \right] \bfH^{-1/2}
        &= \expect \left[ (\bfz \otimes \bfz) \bfP  (\bfz \otimes \bfz) \right]
        = \bfU \, \expect \left[ (\bfz \otimes \bfz) \bfSigma  (\bfz \otimes \bfz) \right] \, \bfU^\top \,.
    \label{eq:pf:gauss:1}
    \end{align}
    Each off-diagonal entry of $\expect\left[ (\bfz \otimes \bfz) \bfSigma  (\bfz \otimes \bfz) \right]$ is zero since it involves expected odd powers of Gaussians. Its $j$\textsuperscript{th} diagonal entry equals (denoting $\sigma_j := [\bfSigma]_{j,j}$) 
    \[
        \expect\left[z_j^2 \sum_{k=1}^d \sigma_k z_k^2\right]
        = \sigma_j \expect[z_j^4] + \sum_{k \neq j} \sigma_k \, \expect[z_j^2 z_k^2]
        = 2 \sigma_j + \tr{\bfSigma} \,.
    \]
    This gives $\expect\left[ (\bfz \otimes \bfz) \bfSigma  (\bfz \otimes \bfz) \right] = 2 \bfSigma + \tr{\bfSigma}\bfI \preceq 3 \tr{\bfSigma} \bfI$ since $\bfSigma \succeq \boldzero$.
    Plugging this back into \eqref{eq:pf:gauss:1} and rearranging completes the proof.
\end{proof}

\subsubsection{Notation}
We set up some notation, that we use throughout this section.
\begin{itemize}
    \item It is convenient to rewrite the \noisyftrl recursion in terms of the difference $\bftheta_t' := \bftheta_t - \bftheta_\star$.
    We can rewrite the \noisyftrl recursion \eqref{eq:dp-ftrl-inf-original} as
    \begin{align} \label{eq:dp-ftrl-inf}
        \bftheta_{t+1}' =  \big( \bfI - \eta (\bfx_t \otimes \bfx_t) \big) \bftheta_t' + \eta\, \xi_t \bfx_t - \eta \sum_{\tau=0}^\infty \beta_\tau \bfw_{t-\tau} \,.
    \end{align}
    We will analyze this recursion.
    \item We describe the asymptotic suboptimality in terms of the self-adjoint linear operator $\bfT: \ell^2 \to \ell^2$ defined by
    \begin{align} \label{eq:def:T}
        [\bfT \bfbeta]_t = \sum_{\tau=0}^\infty \beta_\tau \sum_{j=1}^d ( 1- \eta \lambda_j)^{|t-\tau|} \,.
    \end{align}
    This operator is positive semi-definite, as we show in \Cref{lem:T:psd} below.
    In the finite time setting, we could represent $\bfT$ by the matrix
    \[
        \bfT = \begin{bmatrix}
            d & \sum_{j=1}^d(1 - \eta \lambda_j) &  \sum_{j=1}^d(1 - \eta \lambda_j)^2 & \cdots  \\
            \sum_{j=1}^d(1 - \eta \lambda_j)
            & d & \sum_{j=1}^d(1 - \eta \lambda_j) & \cdots \\
            \sum_{j=1}^d(1 - \eta \lambda_j)^2
            & \sum_{j=1}^d(1 - \eta \lambda_j) & d & \cdots \\
            $\vdots$ & & & $\vdots$ 
        \end{bmatrix}
    \]
    We only consider step-size $0 < \eta < 1/R^2$, which implies that $1 - \eta\lambda_j \in (0, 1)$ for all $j$.

    \item For $j=1, \ldots, d$, define $\bfT_j : \ell^2 \to \ell^2$ as the linear operator
    \begin{align} \label{eq:def_Tj}
        [\bfT_j \bfbeta]_t = \sum_{\tau=0}^\infty \beta_\tau ( 1- \eta\lambda_j)^{|t-\tau|} \,.
    \end{align}
    Note that $[\bfT_j \bfbeta]_t < \infty$ always since
    \[
        \sum_{\tau=0}^\infty \beta_\tau ( 1- \eta\lambda_j)^{|t-\tau|} \le \frac{2 \norm{\bfbeta}_\infty}{\eta \lambda_j} < \infty\,,
    \]
    since $0 < \eta \lambda < 1$.
    Thus, we have that $\bfT = \sum_{j=1}^d \bfT_j$ by the bounded convergence theorem.
    Further, we show in the upcoming \Cref{lem:T:psd} that each $\bfT_j$ is PSD.
    
    \item Define $\bfSigma_\bfbeta, \bfP_\bfbeta \in \mathbb{S}^d$ as
    \begin{align} \label{eq:def:Pbeta}
        \bfSigma_\bfbeta := \diag\left( (\inp{\bfbeta}{\bfT_j \bfbeta})_{j=1}^d \right),
        \quad \text{and}\quad
        \bfP_\bfbeta = \bfU \bfSigma_\bfbeta \bfU^\top \,,
    \end{align}
    where $\bfU$ is the eigen-basis of $\bfH = \bfU \bfLambda \bfU^\top$. By definition, $\bfP_\bfbeta$ commutes with $\bfH$ since $\bfP_\bfbeta\bfH = \bfH \bfP_\bfbeta = \bfU (\Lambda \bfSigma_\bfbeta) \bfU^\top$.
    Further, since each $\bfT_j$ is PSD (\Cref{lem:T:psd}), we have that $\bfSigma_\bfbeta$ and $\bfP_\bfbeta$ are PSD as well.
    We also have
    \begin{align} \label{eq:tr:P_beta}
        \tr{\bfP_\bfbeta} = \tr{\bfSigma_\bfbeta} = \inp{\bfbeta} {\bfT \bfbeta} \,.
    \end{align}
    \item Define the matrix $\bfM_\omega \in \mathbb{C}^{d \times d}$ as
    \begin{align} \label{eq:M_omega}
        \bfM_\omega = 
        \big((1 - \exp(\I \omega)) \bfI - \eta \bfH\big)^{-1} \,.
    \end{align}
\end{itemize}
Throughout, we assume that \Cref{asmp:linear-regression} holds.

\mypar{Preliminary lemmas}
This lemma helps us move back and forth between the time-domain and frequency-domain representations. See \Cref{sec:additional-proofs} for a proof.

\begin{lem} \label{lem:integral-main}
    Consider $\bfbeta \in \ell^2$ and its DTFT $B(\omega)$.
    If $0 < \eta <  1 / \lambda_j$, we have 
    \[
        \frac{1}{2} \inp{\bfbeta}{\bfT_j \bfbeta} 
        \le \frac{\eta\lambda_j}{2\pi} \int_{-\pi}^\pi  \frac{ |B(\omega)|^2 \, \D \omega}{| 1 - \eta \lambda_j - \exp(\I\omega)|^2}
        \le 
        \inp{\bfbeta}{\bfT_j \bfbeta}  \,.
    \]
\end{lem}

Setting $B(\omega) = 1$ and $\bfbeta = (1, 0, \ldots)$ gives the next corollary.
\begin{cor} \label{lem:integral-easy}
    If $0 < \eta < 1/\lambda_j$, we have,
    \[
        \frac{1}{2} 
        \le \frac{\eta\lambda_j}{2\pi} \int_{-\pi}^\pi \frac{\D \omega}{|1 - \eta \lambda_j - \exp(\I \omega)|^2}
        \le 1 \,.
    \]
\end{cor}

\begin{lem} \label{lem:T:psd}
    The operators $\bfT_j$ defined in \eqref{eq:def_Tj} and $\bfT$ defined in \eqref{eq:def:T} are both positive semi-definite for $\eta < 1/ \max_{j \in [d]} \lambda_j$.
\end{lem}
\begin{proof}
    Consider any $\bfbeta \in \ell^2$ and its DTFT $B(\omega)$. We have from \Cref{lem:integral-main} that
    \[
    0 \le \int_{-\pi}^\pi \frac{|B(\omega)|^2 \, \D \omega}{| 1 - \eta \lambda_j - \exp(\I\omega)|^2}
        \le 
        \frac{2\pi}{\eta \lambda_j} \inp{\bfbeta}{\bfT_j \bfbeta} \,,
    \]
    or that $\inp{\bfbeta}{\bfT_j \bfbeta} \ge 0$.
\end{proof}

\subsection{Proof of the Upper Bound on the Asymptotic Suboptimality}
\label{sec:stationary:upper-bound}

The key tool in the warm-up analysis of mean estimation  (\Cref{sec:a:mean}) is the use of linear time-invariant (LTI) input-output systems to relate the output covariance to the input covariance using its {transfer function} (see \Cref{sec:a:lti} for a summary).
The \noisyftrl recursion is not trivial to characterize in this manner because the update \eqref{eq:dp-ftrl-inf-original} is not LTI. Instead, we decompose it into an infinite sequence of LTI systems and carefully analyze the error propagation. 

This consists of the following steps:
\begin{enumerate}[label={Part \arabic*:}, leftmargin=\widthof{Part 4444}]
    \item Decompose the \noisyftrl recursion into a sequence of LTI systems.
    \item Compute the transfer function of each LTI system.
    \item Compute the stationary covariance for each LTI system from the previous one.
    \item Combine the stationary covariances to get the stationary error of the original iterate.
\end{enumerate}

\subsubsection{Part 1: Decomposition into a Sequence of LTI Systems}
\label{sec:proof:decomposition}

A challenge in analyzing the stationary error of \Cref{eq:dp-ftrl-inf} in the frequency domain is that it is not an LTI system. 
Replacing $\bfx_t \otimes \bfx_t$ by $\bfH$ in \Cref{eq:dp-ftrl-inf} results in an LTI update; this system is quite similar to fixed design linear regression. However, this leads to an error in the general case, which satisfies a recursion of the same form as \eqref{eq:dp-ftrl-inf}. We can repeat the same technique of replacing $\bfx_t \otimes \bfx_t$ by $\bfH$ and repeat this process indefinitely.
This proof technique has been used in \cite{aguech2000perturbation} to analyze stochastic tracking algorithms and \cite{bach2013nonstrongly} to analyze iterate-averaged SGD for linear regression.
We adopt this technique to analyze the stationary covariance of DP mechanisms with correlated noise.

We define sequences $(\bftheta_t\pow{r})_{t=-\infty}^\infty$ and $(\bfdelta_t\pow{r})_{t=-\infty}^\infty$ for $r\ge0$  as follows:
\begin{align}
\label{eq:helper-recursions}
\begin{aligned}
    \bftheta_{t+1}\pow{0}
        &= ( \bfI - \eta \bfH) \bftheta_t\pow{0} + \eta \xi_t \bfx_t - \eta \sum_{\tau=0}^\infty \beta_\tau \bfw_{t-k}\,, \\
    \bftheta_{t+1}\pow{r}
        &= (\bfI - \eta \bfH) \bftheta_t\pow{r} + \eta (\bfH - \bfx_t \otimes \bfx_t) \bftheta_t\pow{r-1} \, \text{ for } r > 0\,, \\
    \bfdelta_{t+1}\pow{r} &= (\bfI - \eta \bfx_t \otimes \bfx_t) \bfdelta_t\pow{r} + \eta (\bfH - \bfx_t \otimes \bfx_t) \bftheta_t \pow{r} \,.
\end{aligned}
\end{align}
These recursions are assumed to start at $t=-\infty$ from $\bftheta_{t}\pow{0} = \bftheta'_{t}$, $\bfdelta_t\pow{r} = \boldzero$ for $r\ge 0$ and $\bftheta_t\pow{r} = \boldzero$ for $r > 0$.
These recursions are a decomposition of \eqref{eq:dp-ftrl-inf} as we define below.

\begin{property} \label{prop:recursion-decomposition}
    For each iteration $t$ and any integer $m \ge 0$, we have $\bftheta_t' = \sum_{r=0}^m \bftheta_t\pow{r} + \bfdelta_t\pow{m}$.
\end{property}
\begin{proof}
    We prove this by induction. The base case at $t=-\infty$ holds by definition. Assume that this is true for some integer $t$. Then, we have
    \begin{align*}
        \sum_{r=0}^m \bftheta_{t+1}\pow{r} + \bfdelta_{t+1}\pow{m}
        &= (\bfI - \eta \bfx_t \otimes \bfx_t) \left(\sum_{r=0}^m \bftheta_t\pow{r} + \bfdelta_t\pow{m} \right)
        + \eta \xi_t \bfx_t - \eta \sum_{\tau=0}^\infty \beta_\tau \bfw_{t-\tau}   \\
        &= (\bfI - \eta \bfx_t \otimes \bfx_t) \bftheta_t'
        + \eta \xi_t \bfx_t - \eta \sum_{\tau=0}^\infty \beta_\tau \bfw_{t-\tau}
        = \bftheta_{t+1}' \,.
    \end{align*}
\end{proof}

The idea behind the proof is to show that $\expect\left[\bfdelta_0\pow{m} \otimes \bfdelta_0 \pow{m}\right] \to \boldzero$ as $m \to \infty$. Then, we can use the triangle inequality to bound
\[
    \norm{\bftheta_t'} \le \sum_{r=0}^\infty \norm{\bftheta_t\pow{r}} \,,
\] 
where the stationary error of the right side can be obtained from analyzing the LTI systems defined in \eqref{eq:helper-recursions}.

\subsubsection{Part 2: Characterize the Transfer Function of each LTI System}

There are two LTI systems. First, $\bftheta_t\pow{r}$ for $r > 0$ is an LTI system
\begin{align} \label{eq:lti1}
    \bfz_{t+1} = (\bfI - \eta \bfH) \bfz_t + \eta \bfu_t 
\end{align}
with input $\bfu_t \in \R^d$ and output $\bfz_t \in \R^d$.
Second, $\bftheta_t\pow{0}$ satisfies satisfies an LTI system
\begin{align} \label{eq:lti2}
    \bfz_{t+1} = (\bfI - \eta \bfH) \bfz_t + \eta \bfu_t - \eta \sum_{\tau=0}^\infty \beta_t \bfw_{t-\tau} 
\end{align}
with inputs $(\bfu_t, \bfw_t) \in \reals^d \times \reals^d$ and output $\bfz_t \in \reals^d$ where the weights $\bfbeta \in \ell^2$ are assumed to be given.

We now characterize the transfer functions of these LTI systems; see \Cref{sec:a:lti} for a review.

\begin{property} \label{prop:transfer:lti1}
    The LTI system \eqref{eq:lti1} is $\bfG(\omega) = -\eta \bfM_\omega \in \C^{d\times d}$, where $\bfM_\omega$ is defined in \Cref{eq:M_omega}.
    Moreover, this system is asymptotically stable as long as $\boldzero \prec \eta \bfH \prec \bfI$.
\end{property}
\begin{proof}
    Let $\bfU(\omega) \in \C^d$ and $\bfZ(\omega) \in \C^d$ be the Fourier transforms of $\bfu_t$ and $\bfz_t$ respectively. The transfer function must hold for any input-output sequences, so
    we can choose some sequences and solve for the transfer functions. It is convenient to consider the delta spike on a standard basis (up to scaling), i.e., $\bfU = 2\pi \delta_{\omega} \bfe_j$, where $\delta_\omega$ is the Dirac delta at $\omega$, and $\bfe_j$ is the $j$\textsuperscript{th} standard basis vector in $\R^{d}$. 
    This gives $\bfZ = 2\pi \bfg_j \delta_\omega$ where $\bfg_j(\cdot)$ is the $j$\textsuperscript{th} column of $\bfG(\cdot)$.
    
    To move back to the time domain, we take an inverse Fourier transform to get $\bfu_t = \exp(\I \omega t) \bfe_j$ and $\bfz_t = \bfg_j(\omega) \exp(\I \omega t)$. Plugging this into the update \eqref{eq:lti1} gives and solving for $\bfg_j(\omega)$ gives $\bfg_j(\omega) = -\eta \bfM_\omega \bfe_j$. Stacking these into a matrix gives the expression.
    
    If $\bfu_t \equiv \boldzero$ for all $t$, then $\norm{\bfz_{t+1}}_2 \le \norm{\bfI - \eta \bfH}_2 \norm{\bfz_t}_2 < \norm{\bfz_t}_2$ since $\norm{\bfI - \eta \bfH}_2 < 1$. Hence, $\norm{\bfz_{t}}_2 \to 0$, giving the asymptotic stability of the system.
\end{proof}

\begin{property} \label{prop:transfer:lti2}
    The transfer function of the LTI system \eqref{eq:lti2} is 
    \[
        \tilde \bfG(\omega) = \begin{bmatrix} \bfG(\omega) & \bfG'(\omega) \end{bmatrix} \in \C^{d \times 2d}
    \]
    where $\bfG(\omega) = -\eta \bfM_\omega$ and $\bfG'(\omega) = \eta B(\omega) \bfM_\omega$
    with $B(\omega)$ as the DTFT of $\bfbeta$.
    Moreover, this system is asymptotically stable as long as $\boldzero \prec \eta \bfH \prec \bfI$.
\end{property}
\begin{proof}
    The expression for $\bfG(\omega)$ is the same as in \Cref{prop:transfer:lti1}. To find $\bfG'$, we set the Fourier transforms $\bfU \equiv \boldzero$, $\bfW = 2\pi \delta_\omega \bfe_j$ so that $\bfZ = 2\pi \delta_\omega \bfg'_j$, where $\bfg'_j(\cdot)$ is the $j$\textsuperscript{th} column of $\bfG'(\cdot)$.
    
    An inverse Fourier transform gives the time domain versions $\bfw_t = \exp(\I \omega t)$, 
    $\bfu_t \equiv \boldzero$, $\bfz_t = \exp(\I \omega t) \bfg_j'(\omega)$. Plugging these into \eqref{eq:lti2} and plugging in the definition of $B(\omega)$ gives the expression for the transfer function. Its asymptotic stability holds similar to \Cref{prop:transfer:lti1}.
\end{proof}

\subsubsection{Part 3: Compute the Stationary Covariance of each LTI System}

The stationary covariance of an LTI system driven by white noise can be concisely described in the frequency domain.
A sequence $(\bfu_t)$ is said to be a white noise process if it is mean zero and $\expect[\bfu_t \bfu_\tau] = \boldzero$ for $t\neq \tau$. This is true for both $\bftheta_t\pow{0}$ as well $\bftheta_t\pow{r}$ for $r > 0$.
Since we care about the stationary distribution and we start at $t=-\infty$, we have reached the steady state at $t=0$. So, we compute $\expect[\bftheta_0\pow{r} \otimes \bftheta_0\pow{r}]$.

\mypar{Stationary covariance of the base recursion}
We first start with $\bftheta_t\pow{0}$.

\begin{prop} \label{prop:base-cov}
    We have that $\expect\left[\bftheta_t\pow{0} \otimes \bftheta_t\pow{0} \right]$ is equal for all $t > -\infty$ and is bounded as
    \[
        \expect\left[\bftheta_t\pow{0} \otimes \bftheta_t\pow{0} \right]
        \preceq \eta \sigmasgd^2 \bfI + \eta \sigma^2 \, \bfH^{-1/2} \bfP_\bfbeta \bfH^{-1/2} \,,
    \]
    where $\bfP_\bfbeta$ is defined in \Cref{eq:def:Pbeta} and we denote $\sigma^2 = G^2 \gamma_\infty^2(\bfbeta) / (2\rho)$.
\end{prop}
\begin{proof}
    The input $(\xi_t \bfx_t, \bfw_t)$ forms a white noise sequence, since for $t \neq \tau$, we have
     $\expect[\xi_t \bfx_t \xi_\tau \bfx_\tau] = \expect[\xi_t \bfx_t] \, \expect[\xi_\tau \bfx_\tau] = \boldzero$ (since $\xi_t \bfx_t$ for each $t$ is i.i.d.) and
    $\expect[\bfw_t \bfw_\tau] = \boldzero$.
    The covariance of the input is
    \[
        \expect[(\xi_t\bfx_t, \bfw_t) \otimes (\xi_t \bfx_t, \bfw_t)]
        = \begin{bmatrix}
        \expect[\xi_t^2 \bfx_t \bfx_t] & \boldzero \\
        \boldzero & \expect[\bfw_t \otimes \bfw_t]
        \end{bmatrix}
        = \expect[(\xi_\tau\bfx_\tau, \bfw_\tau) \otimes (\xi_\tau \bfx_\tau, \bfw_\tau)]
    \]
    for all $t, \tau$.
    This is further bounded by \Cref{asmp:input} as 
    \[
        \expect[(\xi_t\bfx_t, \bfw_t) \otimes (\xi_t \bfx_t, \bfw_t)]
        \preceq 
        \begin{bmatrix}
        \sigmasgd^2 \bfH & \boldzero \\
        \boldzero & \sigma^2 \bfI \end{bmatrix}
    \]
    
    The output covariance of the asymptotically stable LTI system \eqref{eq:lti2} can be given in terms of the transfer function
    $\tilde \bfG(\omega) = \begin{bmatrix} \bfG(\omega) & \bfG'(\omega) \end{bmatrix}$
    characterized in \Cref{prop:transfer:lti2} using \Cref{thm:lti-covariance}. This gives that $\expect\left[\bftheta_t\pow{0} \otimes \bftheta_t\pow{0}\right]$ is equal for each $t > -\infty$ and is bounded as 
    \begin{align} \label{eq:pftheta0:1}
        \expect\left[\bftheta_t\pow{0} \otimes \bftheta_t\pow{0}\right]
        &\preceq 
        \frac{1}{2\pi} \int_{-\pi}^\pi \left(\eta^2 \sigmasgd^2 \bfM_\omega \bfH \bfM_\omega^* + \eta^2 \sigma^2 |B(\omega)|^2 \, \bfM_\omega \bfM_\omega^*\right) \, \D\omega \,.
    \end{align}
    With the eigenvalue decomposition $\bfH = \bfU \bfLambda \bfU^\top$, we get
    $\bfM_\omega = \bfU\big( (1 - \exp(\I\omega))\bfI - \eta \bfLambda\big)^{-1} \bfU^\top$. This gives
    \[
        \bfM_\omega \bfH \bfM_\omega^*
        = \bfU \, \diag\left(
        \big( \lambda_j / | 1- \exp(\I \omega) - \eta \lambda_j |^2\big)_{j=1}^d
        \right)
        \bfU^\top \,.
    \]
    We invoke \Cref{lem:integral-easy} to say
    \begin{align*}
        \int_{-\pi}^\pi \bfM_\omega \bfH \bfM_\omega^* \D\omega
        &= \bfU \, \diag\left(
            \left(
            \int_{-\pi}^\pi \D\omega \, \lambda_j / | 1 - \exp(\I \omega) - \eta \lambda_j|^2
            \right)_{j=1}^d
        \right) \, \bfU^\top \\
        &\preceq 
        \bfU \, \diag\left(
            \big(
            2\pi / \eta
            \big)_{j=1}^d
        \right) \, \bfU^\top
        = \frac{2\pi}{\eta} \bfI \,.
        \numberthis
        \label{eq:pftheta0:2}
    \end{align*}
    Similarly, we invoke \Cref{lem:integral-main} to compute
    \begin{align*}
        \int_{-\pi}^\pi |B(\omega)|^2 \bfM_\omega \bfM_\omega^* \D\omega 
        &= 
        \bfU \, \diag\left( \left( \int_{-\pi}^\pi \D\omega\, |B(\omega)|^2 / | 1- \exp(\I\omega) - \eta \lambda_j|^2 \right)_{j=1}^d \right) \, \bfU^\top \\
        &\preceq 
         \bfU \, \diag\left( \big( 
         2\pi \inp{\bfbeta}{\bfT_j \bfbeta} / (\eta \lambda_j)
         \big)_{j=1}^d \right) \, \bfU^\top \\
         &= \frac{2\pi}{\eta} \bfU \bfLambda^{-1/2} \bfSigma_\bfbeta \bfLambda^{-1/2} \bfU^\top = \frac{2\pi}{\eta} \bfH^{-1/2} \bfP_\bfbeta \bfH^{-1/2} \,,
         \numberthis
        \label{eq:pftheta0:3}
    \end{align*}
    where $\bfSigma_\bfbeta$ and $\bfP_\bfbeta$ are defined in \eqref{eq:def:Pbeta}.
    Plugging in \eqref{eq:pftheta0:2} and \eqref{eq:pftheta0:2} into \eqref{eq:pftheta0:1} completes the proof of the upper bound.
\end{proof}

\mypar{Stationary covariance of the higher-order recursion}
Next, we turn to $\bftheta_t\pow{r}$.
\begin{prop} \label{prop:higher-order-cov}
     For any $r\ge 1$, we have
     \[
        \expect\left[\bftheta_0\pow{r} \otimes \bftheta_0\pow{r} \right] \preceq \eta \left(\eta R^2\right)^r \left( \sigmasgd^2 + \frac{\kurt\sigma^2}{R^2} \inp{\bfbeta}{\bfT\bfbeta}\right) \,.
     \]
\end{prop}
\begin{proof}
    Follows from combining \Cref{prop:base-cov} with the more general \Cref{lem:higher-order-cov} below.
\end{proof}

\begin{lem} \label{lem:higher-order-cov}
     For some $r \ge 1$, suppose that 
     $\expect\left[ \bftheta_t\pow{r-1} \otimes \bftheta_t\pow{r-1} \right]$ is equal for each $t$ and is bounded as
     $\expect\left[ \bftheta_t\pow{r-1} \otimes \bftheta_t\pow{r-1} \right] \preceq a \bfI + b \bfH^{-1/2} \bfP_\bfbeta \bfH^{-1/2}$ for some scalars $a, b \ge 0$.
     Then, we have the following.
     \begin{enumerate}[label=(\alph*)]
         \item We have that $\bfzeta_t\pow{r} := (\bfH - \bfx_t \otimes \bfx_t)\, \bftheta_t\pow{r-1}$ is a white-noise process with 
         \[
            \expect\left[ \bfzeta_t\pow{r} \otimes \bfzeta_t\pow{r} \right]
            \preceq \left(a R^2 + b \kurt\, \inp{\bfbeta}{\bfT \bfbeta} \right) \bfH \,.
         \]
        \item We have that $\expect\left[\bftheta_t\pow{r} \otimes \bftheta_t\pow{r}\right]$ is equal for each $t$ and is bounded as
        \[
            \expect\left[\bftheta_t\pow{r} \otimes \bftheta_t\pow{r}\right]
            \preceq \eta \left( a R^2 + b\kurt \, \inp{\bfbeta}{\bfT \bfbeta}\right) \, \bfI \,.
        \]
     \end{enumerate}
\end{lem}
\begin{proof}
    Note that $\expect\left[ \bfzeta_t\pow{r} \otimes \bfzeta_\tau\pow{r} \right] = \boldzero$ for $t \neq \tau$ since $\bfx_t$ is independent of $\bfx_\tau$ and $\expect[\bfx_t \otimes \bfx_t] = \bfH$.
    Since $\bfx_t$ is independent of $\bftheta_t\pow{r-1}$, we get from the tower rule of expectations that
    \begin{align*}
        \expect\left[\bfzeta_t\pow{r} \otimes \bfzeta_t\pow{r}\right]
        &= 
        \expect\left[
        (\bfH - \bfx_t \otimes \bfx_t) \, 
        \left( \bftheta_t\pow{r-1} \otimes \bftheta_t\pow{r-1}\right) \,
        (\bfH - \bfx_t \otimes \bfx_t)
        \right] \\
        &= 
        \expect\left[
        (\bfH - \bfx_t \otimes \bfx_t) \, 
        \expect\left[ \bftheta_t\pow{r-1} \otimes \bftheta_t\pow{r-1}\right] \,
        (\bfH - \bfx_t \otimes \bfx_t)
        \right]\,,
    \end{align*}
    or that $(\bfzeta_t\pow{r})$ is a white noise process.
    Its covariance can further be bounded as
    \begin{align*}
         \expect\left[\bfzeta_t\pow{r} \otimes \bfzeta_t\pow{r}\right]
        &\preceq
        \expect\left[
        (\bfH - \bfx_t \otimes \bfx_t) \, 
        \left( a \bfI + b\bfH^{-1/2} \bfP_\bfbeta \bfH^{-1/2}
        \right) \,
        (\bfH - \bfx_t \otimes \bfx_t)
        \right] \\
        &\preceq a \, \expect\left[ \norm{\bfx_t}_2^2 \, (\bfx_t\otimes \bfx_t)\right]
        + b\, \expect\left[ (\bfx_t \otimes \bfx_t) \bfH^{-1/2} \bfP_\bfbeta \bfH^{-1/2} (\bfx_t \otimes \bfx_t)   \right) \\
        &\preceq aR^2 \bfH + b \kurt \, \tr{\bfP_\bfbeta} \bfH \,,
    \end{align*}
    where the last inequality followed from \Cref{asmp:m4}.
    Further, note that $\tr{\bfP_\bfbeta} = \inp{\bfbeta}{\bfT\bfbeta}$ from \eqref{eq:tr:P_beta}.
    
     The output covariance of the asymptotically stable LTI system \eqref{eq:lti1} can be given in terms of the transfer function
    $\bfG(\omega) = -\eta \bfM_\omega$ using \Cref{thm:lti-covariance} as 
    \[
        \expect\left[ \bftheta_t\pow{r} \otimes \bftheta_t\pow{r} \right]
        \preceq 
        \frac{\eta^2 \left(a R^2 + b\kurt \inp{\bfbeta}{\bfT\bfbeta}\right)}{2\pi} \int_{-\pi}^\pi \bfM_\omega \bfH \bfM_\omega^* \, \D\omega
        \stackrel{\eqref{eq:pftheta0:2}}{\preceq} \eta \left(a R^2 + b\kurt \inp{\bfbeta}{\bfT\bfbeta}\right) \bfI \,.
    \]
\end{proof}

\mypar{Remainder Term}
It remains to show that the remainder term $\bfdelta_t$ can be neglected by taking $m \to \infty$.

\begin{prop} \label{prop:ignore-remainder}
    We have $\lim_{m \to \infty} \expect\left[ \bfdelta_t\pow{m} \otimes \bfdelta_t\pow{m}\right] = \boldzero$.
\end{prop}
\begin{proof}
    Let $\bfzeta_t\pow{m+1} := (\bfH - \bfx_t \otimes \bfx_t) \, \bftheta_t\pow{m}$.
    By \Cref{lem:higher-order-cov} and \Cref{prop:higher-order-cov}, we have $\bfzeta_t$ is a white-noise process with
    \[
        \expect\left[\bfzeta_t\pow{m+1} \otimes \bfzeta_t \pow{m+1} \right]
        \preceq (\eta R^2)^{m+1} \left(\sigmasgd^2 + \frac{\kurt \sigma^2}{R^2} \inp{\bfbeta}{\bfT\bfbeta}\right) \bfH \to \boldzero
    \]
    as $m \to \infty$ since $\eta < 1/R^2$. Note that the update for $\bfdelta_t\pow{m}$ exactly matches that of SGD (without added DP noise), and the noise covariance is $\boldzero$.
    The statement of this result is equivalent to showing that the stationary covariance of SGD with zero residuals is zero. This observation is formalized in Lemma 4 of \cite{jain2017markov} (see also \Cref{thm:fsttcs} of \Cref{sec:a:technical}), which gives for any $t$ that
    \[
        \boldzero \preceq \expect[\bfdelta_t\pow{m} \otimes \bfdelta_t\pow{m}] \preceq \frac{\eta}{1- \eta R^2} \left[ (\eta R^2)^{m+1} \left(\sigmasgd^2 + \frac{\kurt \sigma^2}{R^2} \inp{\bfbeta}{\bfT\bfbeta}\right) \right] \bfI \to \boldzero 
    \]
    as $m \to \infty$.
\end{proof}

\subsubsection{Part 4: Combining the Errors}

\mypar{Time-domain description}
We now state and prove a time-domain description of the upper bound of \Cref{eq:finf:matching}.

\begin{theorem} 
\label{thm:lr-time-domain:appendix}
Suppose \Cref{asmp:linear-regression} holds. Consider the sequence $(\bftheta_t)_{t=-\infty}^\infty$ produced by the \noisyftrl update in \Cref{eq:dp-ftrl-inf-original} with some given weights $\bfbeta \in \ell^2$ and noise variance $\bfw_t \sim \calN(\boldzero, G^2 \gamma_\infty^2(\bfbeta) / (2\rho) \bfI)$. If the learning rate satisfies $\eta < 1 / R^2$, we have
\[
    F_\infty(\bfbeta)
    \le \left( 1 + \left(1 - \sqrt{\eta R^2}\right)^{-2} \right) \eta R^2 \sigmasgd^2
    + 
    \left(1 + \kurt \left(1 - \sqrt{\eta R^2}\right)^{-2} \right) \frac{ \eta G^2 \gamma_\infty^2(\bfbeta)}{2\rho} \inp{\bfbeta} {\bfT\bfbeta} \,.
\]
\end{theorem}
\begin{proof}
    We use shorthand $\sigma^2 = \frac{ G^2 \gamma_\infty^2(\bfbeta)}{2\rho}$.
    First, note that $\eta < 1/R^2$ also implies that $\eta\lambda_j < 1$ for each eigenvalue $\lambda_j$ of $\bfH$.
    The right side is well-defined since \Cref{lem:helper-l2-beta} gives
    \begin{align}\label{eq:pf-time:1}
        |\inp{\bfbeta} {\bfT\bfbeta}|
        \le \sum_{j=1}^d \left| \sum_{t=0}^\infty \sum_{\tau=0}^\infty \beta_t \beta_\tau (1 - \eta\lambda_j)^{|t-\tau|} \right|
        \le \norm{\bfbeta}_2^2\,\, \sum_{j=1}^d \frac{2}{\eta\lambda_j}  < \infty
    \end{align}
    for $\beta \in \ell^2$.
    Next, using \Cref{prop:base-cov}, $\tr{\bfH} \le R^2$, and $\tr{\bfP_\bfbeta}=\inp{\bfbeta}{\bfT\bfbeta}$, we get
    \begin{align} \label{eq:pf-time:2}
        \expect\norm{\bftheta_0\pow{0}}_\bfH^2
        = \tr{\bfH \expect\left[\bftheta_0\pow{0} \otimes \bftheta_0\pow{0}\right]}
        \le \eta R^2 \sigmasgd^2 + \eta \sigma^2 \inp{\bfbeta}{\bfT\bfbeta} \,.
    \end{align}
    Similarly, using \Cref{prop:higher-order-cov}, we get for $r \ge 1$ that 
    \[
        \expect\norm{\bftheta_0\pow{r}}_\bfH^2
        \le (\eta R^2)^{r+1} \left(\sigmasgd^2 + \frac{\kurt \sigma^2}{R^2} \inp{\bfbeta}{\bfT\bfbeta}\right) \,.
    \]
    We can ignore the remainder term since 
    $\expect\norm{\bfdelta_t\pow{m}}_\bfH^2 \to 0$ as $m \to \infty$, from \Cref{prop:ignore-remainder}. Thus, we get using \Cref{prop:recursion-decomposition} and the triangle inequality on the norm $\bfu \mapsto \sqrt{\expect\inp{\bfu}{\bfH\bfu}}$ of a random vector $\bfu$ to get
    \begin{align*}
        \sqrt{\expect\norm{\bftheta_0'}_\bfH^2}
        &\le \sum_{r=0}^\infty \sqrt{\expect\norm{\bftheta_0\pow{r}}_\bfH^2} \,.
    \end{align*}
    To complete the proof, we plug in \Cref{eq:pf-time:1,eq:pf-time:2} and sum up the infinite series. We simplify the result using $\norm{\bfx + \bfy}_\bfH^2 \le 2 \norm{\bfx}_\bfH^2 + 2 \norm{\bfy}_\bfH^2$ and use $F(\bftheta) - F(\bftheta_\star) = (1/2) \norm{\bftheta- \bftheta_\star}_\bfH^2$.
\end{proof}

\mypar{Frequency-domain description}
We now state and prove the frequency domain description of the upper bound \eqref{eq:finf:matching}.

\begin{theorem}\label{thm:lr-frequency-domain:appendix}
   Consider the setting of \Cref{thm:lr-time-domain:appendix}.
    If $B \in L^2$, i.e., $\int_{-\pi}^\pi | B(\omega)|^2 \, \D \omega < \infty$, we have
    \begin{align*}
        F_\infty(B)
        \le & \,\, \left( 1 + \left(1 - \sqrt{\eta R^2}\right)^{-2} \right) \eta R^2 \sigmasgd^2 \\ & + 
         \left(1 + \kurt \left(1 - \sqrt{\eta R^2}\right)^{-2} \right) \frac{\eta^2 G^2 \gamma_\infty^2(B)}{2\pi \rho } \int_{-\pi}^\pi |B(\omega)|^2 \, h(\omega) \, \D \omega \,.
    \end{align*}
\end{theorem}
\begin{proof}
    We again use the shorthand $\sigma^2 = \frac{ G^2 \gamma_\infty^2(\bfbeta)}{2\rho}$.
    First note that
    \begin{align*}
        h(\omega) \le \sum_{j=1}^d \frac{\lambda_j}{1 + (1 - \eta\lambda_j)^2 - 2(1- \eta\lambda_j)} = \sum_{j=1}^d \frac{1}{\eta^2 \lambda_j} = \frac{\tr{\bfH^{-1}}}{\eta^2} \,.
    \end{align*}
    Thus, the right side is well-defined since
    \[
        \int_{-\pi}^\pi |B(\omega)|^2 \, h(\omega) \D \omega 
        \le \frac{\tr{\bfH^{-1}}}{\eta^2} \int_{-\pi}^\pi |B(\omega)|^2\, \D\omega < \infty
    \]
    by assumption.
    We use \Cref{lem:integral-main} to get
    \[
        \inp{\bfbeta}{\bfT\bfbeta} = \sum_{j=1}^d \inp{\bfbeta}{\bfT_j\bfbeta} \le \sum_{j=1}^d \frac{\eta \lambda_j}{\pi} \int_{-\pi}^\pi \frac{|B(\omega)|^2 \D\omega}{| 1- \exp(\I \omega) - \eta \lambda_j|^2} = \frac{\eta}{\pi} \int_{-\pi}^\pi |B(\omega)|^2\, h(\omega) \, \D \omega \,.
    \]
\end{proof}

\begin{remark}[Contribution per eigendirection]
\label{remark:dim_vs_edim}
The expression of \Cref{thm:lr-frequency-domain:appendix} contains a sum over the eigenvalues $\lambda_1,\ldots, \lambda_d$ of the Hessian matrix $\bfH$ through the function $h(\omega)$, defined in Eq.~\eqref{eq:def:h}. Thus, the contribution of eigenvalue $\lambda_j$ to the error is proportional to (ignoring problem-dependent constants)
\begin{align}  \label{eq:error_of_eigval}
\text{Err}_j := \int_{-\pi}^\pi \frac{\lambda_j \,  |B(\omega)|^2 \,\, \D \omega}{|1 - \exp(\I\omega) - \eta\lambda_j|^2}\,.
\end{align}
For \noisysgd, we have that $B(\omega) = 1$, and the error $\text{Err}_j = \Theta(1)$ evaluates to an absolute constant (details in \Cref{lem:integral-easy}). In other words, each eigendirection contributes a constant amount to the error, leading to a $O(d)$ dimension dependence in the asymptotic error.

On the other hand, as we discuss further in \Cref{remark:dim_vs_edim:2} (\Cref{sec:tuned_dp_ftrl_proof}), we have $\text{Err}_j \le \tilde O(\lambda_j)$ for \ournoisyftrl. Thus, the contribution of an eigendirection reduces proportional to the eigenvalues, leading to an effective dimension dependence for \ournoisyftrl.

These quantitative results can be connected intuitively to the signal in the gradients. Let $\lambda_1, \ldots, \lambda_d$ be the eigenvalues of $\bfH$ with $\lambda_1 = 1$.
The negative gradient at each step pushes the iterates back towards the minimizer, thus mitigating the effect of the past noise. However, the signal in the gradient along tail eigen-directions is small, making it ineffective in such directions. This leads to $\text{Err}_j = \Theta(1)$ for \noisysgd, which can be much larger than $\lambda_j$. On the other hand, the anti-correlations of \ourprivftrl ``subtract out'' the previous noise, leading to $\text{Err}_j \propto \lambda_j$ for \ournoisyftrl, i.e., an improved effective dimension dependence.
\end{remark}

\subsection{Proofs of Lower Bounds on the Asymptotic Suboptimality} \label{sec:stationary:lower-bounds-proofs}

We now state and prove the lower bound part of \eqref{eq:finf:matching} on the asymptotic suboptimality.

\begin{assumption} \label{asmp:linear-regression:lower-bound}
    In addition to \Cref{asmp:linear-regression}, the data distribution $\Pdata$ satisfies the following:
    \begin{enumerate}[label=(\textbf{A\arabic*'}),start=2]
        \item \textbf{Worst-Case Residuals}: For $(\bfx, y) \sim \Pdata$, the residual $\xi := y - \inp{\bftheta_\star}{\bfx}$ has variance $\expect[\xi^2]=\sigmasgd^2$.
    \end{enumerate}
\end{assumption}
Note that the variance of $\xi^2$ holds with equality under \Cref{asmp:linear-regression:lower-bound}.

\begin{theorem} \label{thm:lr-lower-bound}
    Suppose \Cref{asmp:linear-regression:lower-bound} holds. Consider the sequence $(\bftheta_t)_{t=-\infty}^\infty$ produced by the \noisyftrl update in \Cref{eq:dp-ftrl-inf-original} with some given weights $\bfbeta \in \ell^1$.
    If the learning rate satisfies $\eta < 1 / R^2$, we have
    \begin{align*}
        F_\infty(\bfbeta)
        \ge \frac{\eta\sigmasgd^2}{2} \tr{\bfH}
            + \frac{\eta^2 G^2 \gamma_\infty^2(B)}{4\pi \rho} \int_{-\pi}^\pi |B(\omega)|^2 \, h(\omega) \, \D\omega
        \ge \frac{\eta \sigmasgd^2}{2} \tr{\bfH}
            + \frac{\eta G^2 \gamma_\infty^2(\bfbeta)}{4\rho} \inp{\bfbeta}{\bfT \bfbeta} \,,
    \end{align*}
    where $h(\omega)$ is defined in \eqref{eq:def:h} and $\bfT$ is defined in \eqref{eq:def:T}.
    Furthermore, the minimal stationary error over all choices of $\bfbeta$ is bounded as
    \[
        \inf_{\bfbeta} \,\, F_\infty(\bfbeta) \ge 
         \frac{1}{4}\left(2 \eta \sigmasgd^2 + \frac{\eta^2 G^2}{2\rho} \right) \tr{\bfH}
    \]
    where the infimum is attained by $\bfbeta_\star$ whose DTFT $B_\star$ verifies $|B_\star(\omega)|^2 = 1 / \sqrt{h(\omega)}$.
\end{theorem}

Note that we assume $\bfbeta \in \ell^1$, i.e., $\norm{\bfbeta}_1 = \sum_{\tau=0}^\infty |\beta_\tau| < \infty$ for technical reasons. This implies that $\bfbeta \in \ell^2$, which we assumed for the upper bounds.

The key idea behind the proof is that the variance of $\bftheta_t'$ is no smaller than that of an LTI system with $\bfx_t\otimes \bfx_t$ replaced by its expectation $\bfH$. We can quantify this latter covariance with equality under \Cref{asmp:linear-regression:lower-bound}.
We set up some notation and develop some preliminary results before proving this theorem.

Formally, consider the sequences $(\bftheta_t\pow{0})_{t=-\infty}^\infty$ and 
$(\bfdelta_t\pow{0})_{t=-\infty}^\infty$ as defined in \eqref{eq:helper-recursions} (cf. \Cref{sec:proof:decomposition}). They start at $t=-\infty$ from $\bftheta_t\pow{0} = \bftheta_t'$ and $\bfdelta_t\pow{0} = \boldzero$.
By \Cref{prop:recursion-decomposition}, we these satisfy $\bftheta_t' = \bftheta_t\pow{0} + \bfdelta_t\pow{0}$. 

We use a technical result that $\bftheta_t\pow{0}$ and $\bfdelta_t$ are uncorrelated. It is proved at the end of this section.
\begin{prop} \label{prop:helper-covariance}
    Consider the setting of \Cref{thm:lr-lower-bound}.
    We have for all $t$ that 
    \[
        \expect\left[ \bftheta_t\pow{0} \otimes \bfdelta_t\pow{0}\right] = \boldzero \,.
    \]
\end{prop}

We now give the proof of \Cref{thm:lr-lower-bound}.

\begin{proof}[Proof of \Cref{thm:lr-lower-bound}]
    We use shorthand $\sigma^2 = \frac{ G^2 \gamma_\infty^2(\bfbeta)}{2\rho}$.
    Since $\bftheta_t' = \bftheta_t\pow{0} + \bfdelta_t\pow{0}$, we have
    \begin{align} \label{eq:pf:lb1}
        \expect\left[\bftheta_t' \otimes \bftheta_t'\right] 
        = \expect\left[\bftheta_t\pow{0} \otimes \bftheta_t\pow{0} \right]
        + \expect\left[ \bfdelta_t\pow{0} \otimes
            \bfdelta_t\pow{0}\right]
        \succeq \expect\left[\bftheta_t\pow{0} \otimes \bftheta_t\pow{0} \right]
    \end{align}
    where the cross terms disappear from \Cref{prop:helper-covariance} for the first equality.
    We can get an expression for this term by following the proof of \Cref{prop:base-cov}: under \Cref{asmp:linear-regression:lower-bound}, we have that \Cref{eq:pftheta0:1} holds with equality. Thus, we get for all $t > -\infty$ that
    \begin{align}
    \nonumber
        F_\infty(B) &=
        \tr{\bfH\,  \expect\left[\bftheta_t' \otimes \bftheta_t'\right]} 
        \succeq
        \tr{\bfH\,  \expect\left[\bftheta_t\pow{0} \otimes \bftheta_t\pow{0}\right]}
        \\
        &= 
        \frac{1}{2\pi} \int_{-\pi}^\pi \left(\eta^2 \sigmasgd^2 \tr{\bfH^{1/2} \bfM_\omega \bfH \bfM_\omega^* \bfH^{1/2}} + \eta^2 \sigma^2 |B(\omega)|^2 \, \tr{\bfH^{1/2} \bfM_\omega \bfM_\omega^* \bfH^{1/2}} \right) \, \D\omega \,.
    \label{eq:pftheta0:LB1}
    \end{align}
    We invoke \Cref{lem:integral-easy} to obtain
    \begin{align*}
        \int_{-\pi}^\pi \tr{\bfH^{1/2} \bfM_\omega \bfH \bfM_\omega^* \bfH^{1/2}} \D\omega
        &= \sum_{j=1}^d \lambda_j^2 
            \int_{-\pi}^\pi \frac{\D\omega}{| 1 - \exp(\I \omega) - \eta \lambda_j|^2}
         \\
        &\ge
        \sum_{j=1}^d
            \frac{\pi \lambda_j}{\eta}
        = \frac{\pi}{\eta} \tr{\bfH} \,.
    \end{align*}
    Similarly, we invoke \Cref{lem:integral-main} to compute
    \begin{align*}
        \int_{-\pi}^\pi |B(\omega)|^2 \,  \tr{\bfH^{1/2} \bfM_\omega \bfM_\omega^* \bfH^{1/2}} \D\omega 
        &= 
         \int_{-\pi}^\pi \left(\sum_{j=1}^d |B(\omega)|^2 \frac{\lambda_j}{| 1- \exp(\I\omega) - \eta \lambda_j|^2} \right) \D\omega \\
         & = \int_{-\pi}^\pi |B(\omega)|^2 \, h(\omega) \, \D\omega
        \ge \frac{\pi}{\eta}  \inp{\bfbeta}{\bfT \bfbeta} \,.
    \end{align*}
This establishes the lower bound for specific choices of $\bfbeta$.

Now, we turn to the universal lower bound.    Using the expression for  $\gamma_\infty(B)$ from \Cref{prop:sensitivity-fourier}, we get that the lower bound from the theorem statement is
\begin{align} \label{eq:pf:lb:1}
    F_\infty(B)
    \ge 
    \frac{\eta \sigmasgd^2}{2} \tr{\bfH} + 
    \frac{\eta^2 G^2}{8\pi^2 \rho} 
    \left( \int_{-\pi}^\pi \frac{\D\omega}{|B(\omega)|^2}\right)
    \left(  \int_{-\pi}^\pi |B(\omega)|^2 h(\omega) 
    \right) \,.
\end{align}
The Cauchy-Schwarz inequality gives us that
\[
    \left( \int_{-\pi}^\pi \frac{\D\omega}{|B(\omega)|^2}\right)
    \left(  \int_{-\pi}^\pi |B(\omega)|^2 h(\omega) 
    \right) \ge 
    \left( \int_{-\pi}^\pi \sqrt{h(\omega)} \, \D\omega \right)^2\,,
\]
with equality attained for 
$|B(\omega)|^2 = 1/\sqrt{h(\omega)}$. This gives the universal lower bound on \eqref{eq:pf:lb:1} over all possible choices of $B$ (or equivalently, all possible choices of $\bfbeta$).
To further lower bound this, we use $\cos(\omega) \ge -1$ to get
\begin{align*}
    h(\omega) &=
    \sum_{j=1}^d \frac{\lambda_j}{1 + (1 - \eta \lambda_j)^2 - 2 (1 - \eta \lambda_j) \cos(\omega)}
    \ge \sum_{j=1}^d \frac{\lambda_j}{(2 - \eta\lambda_j)^2}
    \ge \frac{1}{4} \sum_{j=1}^d \lambda_j
    = \frac{\tr{\bfH}}{4} \,.
\end{align*}
Thus, we get that \eqref{eq:pf:lb:1} can be further lower bounded as
\begin{align*}
    F_\infty(B)
    \ge 
    \frac{\eta \sigmasgd^2}{2} \tr{\bfH} + 
    \frac{\eta^2 G^2}{8 \pi^2 \rho} \left( \int_{-\pi}^\pi \frac{\sqrt{\tr{\bfH}}}{2} \, \D\omega \right)^2
    =  \frac{\eta \sigmasgd^2}{2} \tr{\bfH} + 
        \frac{\eta^2 G^2}{8 \rho} \tr{\bfH} \,.
\end{align*}
\end{proof}

\mypar{Missing technical proofs in the lower bound}
We now give the proof of \Cref{prop:helper-covariance}, which first relies on the following intermediate result.

\begin{prop} \label{prop:helper-covar-2}
    Consider the setting of \Cref{thm:lr-lower-bound}.
    We have for all $t, \tau$ that
    \[
        \expect\left[ \bfw_\tau \otimes \bfdelta_t\pow{0} \right] = \boldzero \,.
    \]
\end{prop}
\begin{proof}
    For this proof, we start the sequences at $t=0$ rather than $t=-\infty$.
    We drop the superscript to write $\bfdelta_t\pow{0}$ as $\bfdelta_t$.
    Define shorthand $\bfQ_t := \bfI - \eta \bfx_t \otimes \bfx_t$ and $\bfR_t := \bfH - \bfx_t \otimes \bfx_t$. We expand out the recursion to get
    \begin{align*}
        \bfdelta_{t} 
        &= \bfQ_{t-1} \bfdelta_{t-1} + \eta \bfR_{t-1} \bftheta_{t-1}\pow{0}  \\
        &= \bfQ_{t-1}( \bfQ_{t-2} \bfdelta_{t-2} + \eta \bfR_{t-2} \bftheta_{t-2}\pow{0}) + \eta \bfR_{t-1} \bftheta_{t-1}\pow{0} \\
        &= \bfQ_{t-1} \bfQ_{t-2} \cdots \bfQ_0 \bfdelta_0
        + \eta\left(
            \bfR_{t-1} \bftheta_{t-1}\pow{0}
            + \bfQ_{t-1} \bfR_{t-2} \bftheta_{t-2}\pow{0}
            + \cdots + \bfQ_{t-1}\cdots \bfQ_1 \bfR_0 \bftheta_0\pow{0} 
        \right) \,.
    \end{align*}
    The first term is zero because $\bfdelta_0 = \boldzero$ at initialization. 
    Since $\bfR_{\tau}$ is mean zero and independent of $\bftheta_{\tau}\pow{0}$ and $\bfR_t$ for $t > \tau$, we have
    \begin{align*}
        \frac{1}{\eta}\expect[\bfdelta_t \otimes \bfw_\tau]
        =&\, \, \expect[\bfR_{t-1}] \expect\left[\bftheta_{t-1}\pow{0} \otimes \bfw_\tau\right] \\
        &\quad + \expect[\bfQ_{t-1}] \, \expect[\bfR_{t-2}]
         \expect\left[\bftheta_{t-2}\pow{0} \otimes \bfw_\tau\right]
         + \cdots + 
         \expect[\bfQ_{t-1} \cdots \bfQ_1] \, \expect[\bfR_{0}]
         \expect\left[\bftheta_{0}\pow{0} \otimes \bfw_\tau\right] \\
         =& \,\,  \boldzero \,,
    \end{align*}
    giving us the desired result.
\end{proof}

\begin{proof}[Proof of \Cref{prop:helper-covariance}]
    We drop the superscript to write $\bfdelta_t\pow{0}$ as $\bfdelta_t$.
    We prove the claim by induction.
    At initialization, we have $\bfdelta_{-\infty} = \boldzero$ so the hypothesis holds.
    Now assume that it holds at time $t$, i.e.,
    $\expect\left[ \bftheta_{t}\pow{0} \otimes \bfdelta_{t}\right] = \boldzero$.
    
    Next, we expand out
    $\expect\left[ \bftheta_{t+1}\pow{0} \otimes \bfdelta_{t+1}\right]$ using their respective recursions. 
    Note that $\bfw_t$, $\bfH - \bfx_t \otimes \bfx_t$ and $\xi_t$ are each zero mean and independent of all quantities appearing up to iteration $t$ (formally, they are independent of the $\sigma$-algebra generated by $(\bftheta_t\pow{0}$ and $\bfdelta_t$). This gives
    \begin{align} \label{eq:lbh-pf:0}
    \begin{aligned}
        \expect\left[ \bftheta_{t+1}\pow{0} \otimes \bfdelta_{t+1}\right]
        =& (\bfI - \eta\bfH) \expect\left[ \bftheta_{t}\pow{0} \otimes \bfdelta_{t}\right] (\bfI - \eta\bfH) 
        - \eta \expect\left[ \sum_{\tau=0}^\infty \beta_\tau \left(\bfw_{t-\tau} \otimes \delta_t\pow{0} \right) \right] (\bfI - \eta \bfH)  \,.
    \end{aligned}
    \end{align}
    The first term is zero by the induction hypothesis. For the second term, we can interchange the expectation and the infinite sum by the Fubini-Tonelli theorem since
    \begin{align*}
        \sum_{\tau=0}^\infty 
        |\beta_\tau| \,\, \expect\left| \inp*{\bfw_{t-\tau}} {\bfdelta_t\pow{0}}  \right|
        \le \norm{\bfbeta}_1 \,  \max_{\tau=0,\ldots,\infty} \expect\left| \inp*{\bfw_{t-\tau}} {\bfdelta_t\pow{0}}  \right| < \infty
    \end{align*}
    since $\bfbeta_1 \in \ell^1$ and 
    $\expect\left| \inp*{\bfw_{t-\tau}} {\bfdelta_t\pow{0}}  \right| < \infty$
    because 
    \[
        \expect\inp*{\bfw_{t-\tau}} {\bfdelta_t\pow{0}}
        = \tr{
        \expect\left[{\bfw_{t-\tau}} \otimes {\bfdelta_t\pow{0}} \right]} = 0 
    \]
    by \Cref{prop:helper-covar-2}. By \Cref{prop:helper-covar-2} again, we thus get
    \[
         \expect\left[ \sum_{\tau=0}^\infty \beta_\tau \left(\bfw_{t-\tau} \otimes \delta_t\pow{0} \right) \right]
         =  \sum_{\tau=0}^\infty \beta_\tau \, \expect\left[  \left(\bfw_{t-\tau} \otimes \delta_t\pow{0} \right) \right] = \boldzero \,.
    \]
\end{proof}

\subsection{Asymptotics of \ournoisyftrl} \label{sec:tuned_dp_ftrl_proof}

We now state and prove the upper bound for \ournoisyftrl.
Note that \ournoisyftrl can be described in the frequency domain as $|\hat B^\nu(\omega)|^2 = |1 - \nu - \exp(\I \omega)|$.

For the proof, we
define $\calI : (0, 1)^2 \to \reals_+$ as the integral
\begin{align} \label{eq:elliptic-nu-ftrl}
    \calI(a, b) := \int_{-\pi}^\pi \frac{|1 - a - \exp(\I\omega)|}{| 1- b - \exp(\I\omega)|^2}  \, \D\omega\,.
\end{align}
The crux of the proof relies on a precise characterization of this integral, as we will shortly see below.
\begin{lem} \label{lem:lem-interal-elliptic-tail}
Consider the integral $\calI$ from \eqref{eq:elliptic-nu-ftrl}.
It satisfies the following properties:
\begin{enumerate}[label=(\roman*),nosep]
    \item \label{part:lem-interal-elliptic-tail:a}
    For all $a \in (0, 1)$, we have 
    \[
        \calI(a, a) \le 5 \log(8 / a) \,.
    \]
    \item \label{part:lem-interal-elliptic-tail:b}
    For all $a \le b \le 1/4$, we have
    \[
        \calI(a, b) \le \frac{128}{49} \log(8/a)\big(1 + O(a)\big) \,.
    \]
\end{enumerate}
\end{lem}
\begin{proof}
    The strategy is to reduce this integral to the standard elliptic integrals and leverage their properties to get the result.
We start with the first part $\calI(a, a)$. We use \Cref{lem:sensitivity-ellip} to rewrite in terms of the elliptic integral of the first kind $K(k) = \int_{0}^{\pi/2} \D\omega / \sqrt{1 - k^2 \sin^2(\omega)}$ (denoted as (a)).
Then, we use \Cref{prop:ellipk:asymp} which says that $K(k) = O(-\log\sqrt{1-k^2})$ (denoted as (b)). This gives,
\begin{align} \label{eq:tuned-dpftrl-pf2}
    \calI(a, a) 
     &\stackrel{\text{(a)}}{=} \frac{4}{2 - a} K\left( \frac{\sqrt{1 - a}}{1 - a / 2} \right)
     \stackrel{\text{(b)}}{\le} 
     \frac{5}{2 - a} \log\left( \frac{4}{a}(2-a) \right) 
     \le  5 \log\left( \frac{8}{a} \right) \,.
\end{align}
Similarly, we can express $\calI(a, b)$ for $a \neq b$ in terms of the elliptic integral of the third kind $\Pi(\alpha^2, k)$, whose definition is given in \eqref{eq:ellippi}. From \Cref{lem:error-ellip}, we have for $a, b \in (0, 1)$ that
\[
    \calI(a, b) = \frac{2a^2}{b^2(1 - a/2)} \Pi(\alpha^2, k)
    \quad \text{where} \quad
    \alpha^2 = \frac{b^2(1 - a) - a^2(1-b)}{b^2 ( 1 - a/2)^2}
\]
and $k = \sqrt{1 - a}/(1 - a/2)$. We invoke \Cref{prop:ellippi:asymp} to bound the behavior of $\Pi(\alpha^2, k)$ as $k \to 1^-$ (i.e. $a \to 0^+$) to get
\begin{align*}
    \calI(a, b) 
    &\le \frac{2a^2}{b^2 (1 - a/2)} \,
    \frac{1}{\sqrt{1-\alpha^2}} \log \frac{4}{\sqrt{1-k^2}}
    \left(1 + O(a)\right) \\
    &= \frac{2(1-a/2)}{(1 - b/2)^2} \, \log\left(\frac{4}{a}(2-a) \right) \left(1+ O(a)\right)
    \le \frac{128}{49} \log(8 / a) \left(1 + O(a)\right) \,,
\end{align*}
where the last inequality holds for $a \le b \le 1/4$.
\end{proof}

We are now ready to prove the bounds for \ournoisyftrl.
\begin{prop} \label{prop:dp-ftrl-bound}
    Consider the setting of \Cref{thm:lr-frequency-domain:appendix} with $\sigmasgd^2 = 0$. Then, \ournoisyftrl with  $\nu \le \eta\mu$ satisfies
    \[
        F_\infty(\hat \bfbeta^\nu)
        \le C \, \max\{1, \kurt\} \, \eta^2 G^2 \rho^{-1} \, \tr{\bfH} \, \log^2\left( \frac{8}{\nu} \right) + \tilde O(\eta^3 R^2 \mu G^2 \rho^{-1}) \,,
    \]
    for a universal constant $C > 0$, and $\tilde O(\cdot)$ suppresses polylogarithmic terms in the problem parameters.
\end{prop}
\begin{proof}
We use $C$ to denote a universal constant that can change from line to line.
We can express the bound of \Cref{thm:lr-frequency-domain:appendix} with our specific choice of $B(\omega)$ as
\begin{align} \label{eq:tuned-dpftrl-pf1}
    F_\infty(\hat B^\nu)
    \le C \, \max\{1, \kurt\} \, \calI(\nu, \nu) \sum_{j=1}^d \lambda_j \calI(\nu, \eta\lambda_j) \,.
\end{align}
For the $\calI(\nu, \nu)$ term, we plug in \Cref{lem:lem-interal-elliptic-tail}\ref{part:lem-interal-elliptic-tail:a}.
We plug $a = \nu$ and $b = \eta\lambda_j$ into \Cref{lem:lem-interal-elliptic-tail}\ref{part:lem-interal-elliptic-tail:b} to get (note that its conditions are satisfied)
\begin{align} \label{eq:tuned-dpftrl-pf3}
    \calI(\nu, \eta\lambda_j) \le C\, \log\left( \frac{8}{\nu} \right)  \left(1 + O(\nu)\right) \,.
\end{align}
The last term is $O(\nu) \le O(\eta\mu)$.
Plugging in \eqref{eq:tuned-dpftrl-pf2} and \eqref{eq:tuned-dpftrl-pf3} into \eqref{eq:tuned-dpftrl-pf1}
and using $\tr{\bfH} = \sum_{j=1}^n \lambda_j \le R^2$ completes the proof.
\end{proof}

\begin{remark}[Contribution per eigendirection]
\label{remark:dim_vs_edim:2}
    We continue the discussion of \Cref{remark:dim_vs_edim}. 
    The proof of \Cref{prop:dp-ftrl-bound} shows that the contribution of the $j$\textsuperscript{th} eigendirection to the asymptotic suboptimality is proportional to 
    \[
        \text{Err}_j = \lambda_j \calI(\nu, \eta \lambda_j) \,.
    \]
    As long as $\nu \le \eta \mu$, we get from \Cref{lem:lem-interal-elliptic-tail} that $\text{Err}_j \le O\big(\lambda_j \, \log(1/\nu)\big)$. Thus, the error contributed drops proportional to $\lambda_j$, leading to an effective dimension dependence for \ournoisyftrl.
\end{remark}

\subsection{Asymptotics of Anti-PGD}
\label{sec:two-step-noise}

As we discussed in \Cref{tab:ComparisonNoisy}, anti-PGD~\cite{orvieto2022anticorrelated} is a special case of \noisyftrl with $\bfbeta = (1, -1, 0, \ldots)$. Then, we have that $(\toeplitz(\bfbeta))^{-1}$ is the lower triangular matrix of all ones, so we have $\gamma_T(\bfbeta) = T$, or that its limiting sensitivity is infinite.

We can circumvent the infinity by damping $\bfbeta = (1, -(1- \nu), 0, \ldots)$ for some $0 < \nu < 1$ to be decided later. In this case, we have $B(\omega) = 1 - (1 - \nu) \exp(-\I \omega)$, so that 
$|B(\omega)|^2 = |1- \nu - \exp(\I \omega)|^2$, which is the analogue of \ournoisyftrl with a square.

\begin{prop} \label{prop:dp-ftrl-bound-2step}
    Consider the setting of \Cref{thm:lr-frequency-domain:appendix} with $\sigmasgd^2 = 0$ and $\bfbeta = (1, -(1 - \eta\lambda), 0, \ldots)$ for some $\lambda \in (0, 1/\eta]$. Then, we have,
    \[
        F_\infty(\bfbeta)
        = \Theta\left(
        \eta G^2 \rho^{-1}\left(\nu d + \frac{\eta \tr{\bfH}}{\nu} \right)
        \right) \,.
    \]
    Further, if the learning rate satisfies $\eta = c/\tr{\bfH}$ and we take $\bfbeta = (1, -(1 - \sqrt{1/d}), \ldots)$, we get
    \[
         F_\infty(\bfbeta)
         = \Theta\left((c^{1/2} + c^{-1/2}) \eta^{3/2} \sigma^2 \sqrt{d \, \tr{\bfH}} \right) \,.
    \]
\end{prop}
\begin{proof}
    Let $\sigma^2 = G^2 / (2\rho)$.
    From \Cref{thm:lr-frequency-domain:appendix,thm:lr-lower-bound}, we get that 
    \begin{align} \label{eq:two-step:pf:1}
       F_\infty(\bfbeta)
        = \Theta\left(
            \eta^2 \sigma^2 
            \left(\int_{-\pi}^\pi
                \frac{\D\omega}{|1 - \nu - \exp(\I\omega)|^2}
            \right)
            \left(
            \sum_{j=1}^d \lambda_j
                \int_{-\pi}^\pi 
                \frac{|1 - \nu - \exp(\I\omega)|^2}{|1 - \eta \lambda_j - \exp(\I\omega)|^2}\, \D\omega
            \right)
        \right) \,.
    \end{align}
    Using \Cref{lem:cos-integral}, we have
    \[
        \int_{-\pi}^\pi
                \frac{\D\omega}{|1 - \nu - \exp(\I\omega)|^2}
        = \frac{2\pi}{\nu(2 - \nu)} = \Theta\left( \frac{1}{\nu} \right) \,.
    \]
    For the second integral, we expand out the numerator and invoke \Cref{lem:cos-integral} again to get
    \begin{align*}
         \frac{1}{2\pi}\int_{-\pi}^\pi 
                \frac{|1 - \nu - \exp(\I\omega)|^2}{|1 - \eta \lambda_j - \exp(\I\omega)|^2}\, \D\omega
        &= \frac{1 + (1 - \nu)^2}{\eta\lambda_j(2 - \eta\lambda_j)} - 2(1 - \nu) \frac{1 - \eta\lambda_j}{\eta\lambda_j(2 - \eta\lambda_j)} \\
        &= \Theta\left( \frac{\nu^2}{\eta\lambda_j} + 1 \right) \,,
    \end{align*}
    where we use $1 \le 2 - \nu \le 2$ and the same for $\lambda_j$ instead of $\lambda$.
    Plugging the two integrals back into \eqref{eq:two-step:pf:1} completes the proof.
\end{proof}

\subsection{Effective Dimension and the Stable Rank}
\label{sec:a:edim}

The stable/numerical rank $\srank(\bfA)$ of a matrix $\bfA$ is defined as
\[
	\srank(\bfA) = \frac{\|\bfA\|_F^2}{\sigma_{\max}(\bfA)^2}\,,
\]
i.e., the squared ratio of the Frobenius norm of a matrix to its largest singular value~\cite{rudelson2007sampling}.
By comparing this to our definition of the effective dimension, we find that $\edim(\bfH) = \srank(\bfH^{1/2})$.
Note that the effective dimension is also called the ``intrinsic dimension'' by \citet{martinsson2020randomized}.

The stable rank of a matrix is a continuous function while the true rank is discontinuous.
Thus, it is highly desirable for the error of a numerical algorithm to scale with the stable rank of its matrix input rather than the true rank~\cite{rudelson2007sampling,martinsson2020randomized}.
The stable rank is thus a fundamental quantity appearing in various fields such as randomized linear algebra~\cite{cohen2016optimal,martinsson2020randomized} and matrix concentration~\cite{hsu2011dimension,minsker2017some}.

Our results show that \ourprivftrl's error has the desirable property of scaling with the stable rank (i.e. effective dimension) of the Hessian $\bfH$ rather than its true rank (i.e. the problem's dimension).

\subsection{Proofs of Technical Lemmas} \label{sec:additional-proofs}

We now prove \Cref{lem:integral-main}.

\begin{proof}[Proof of \Cref{lem:integral-main}]
    Denote 
    \[
        I = \int_{-\pi}^\pi \frac{|B(\omega)|^2 \, \D \omega}{| 1 - \eta \lambda_j - \exp(\I\omega)|^2} \,.
    \] 
The denominator is simply
\begin{align}
       \left| {1- \exp(\I \omega) - \eta \lambda_j} \right|^{2} = {1 + (1 - \eta \lambda_j)^2 - 2 (1-\eta\lambda_j) \cos \omega} \,.
    \label{eq:pf:1}
    \end{align}
We expand the numerator as
    \begin{align*}
        |B(\omega)|^2
        &= \sum_{t=0}^\infty \beta_t^2 + \sum_{t=0}^\infty \sum_{\tau=0}^{t-1} \beta_t \beta_\tau \big( \exp(\I \omega(t-\tau)) +  \exp(-\I \omega(\tau - t))\big) \\
        &= \sum_{t=0}^\infty \beta_t^2 + 2\sum_{t=0}^\infty \sum_{\tau=0}^{t-1} \beta_t \beta_\tau \cos(\omega(t-\tau)) \\
        &= \sum_{t=0}^{\infty} \sum_{\tau=0}^\infty \beta_t \beta_{\tau} \cos(\omega(t - \tau)) \,.
        \numberthis \label{eq:pf:2}
    \end{align*}
    This is bounded since the Cauchy-Schwarz inequality gives
    \[
        |B(\omega)|^2 \le \norm{\bfbeta}_2^2 < \infty \,.
    \]
    Thus, we can apply Fubini's theorem to exchange the sum and integral to give
    \begin{align*}
        I &= \sum_{t=0}^\infty \sum_{\tau=0}^\infty \beta_t \beta_\tau \int_{-\pi}^\pi \frac{\cos(\omega(t-\tau)) \D\omega}{1 + (1 - \eta\lambda_j)^2 - 2 (1 - \eta\lambda_j) \cos(\omega)} \\
        &= \sum_{t=0}^\infty \sum_{\tau=0}^\infty 
        \frac{2\pi}{1 - (1 - \eta\lambda_j)^2} (1 - \eta \lambda_j)^{|t - \tau|}
        = \frac{2 \pi \inp{\bfbeta}{\bfT_j \bfbeta}}{\eta \lambda_j (2 - \eta\lambda_j)} \,,
    \end{align*}
    where we evaluated the integral using \Cref{lem:cos-integral}.
    We use $1 \le 2 - \eta \lambda_j \le 2$ to complete the proof.
\end{proof} 
\section{Finite-Time Privacy-Utility Tradeoffs for Linear Regression}\label{app:finite_lr}

The goal of this section is to establish the finite time convergence of \dpmf. The key idea of the proof is to establish high probability bounds on the $\ell_2$ norm of the iterates of \noisyftrl and use that to deduce a clip norm that does not clip any gradients with high probability. 

The outline of this section is as follows:
\begin{itemize}
    \item \textbf{\Cref{sec:finite-time:setup}}: Preliminaries, including setup, notation and assumptions.
    \item \textbf{\Cref{sec:finite-time:high-probability}}: High probability bounds the iterates of \noisyftrl.
    \item \textbf{\Cref{sec:finite-time:expected}}: Expected bounds on the iterates of \noisyftrl.
    \item \textbf{\Cref{sec:a:dp-guarantee}}: 
    Connecting \dpftrl to \noisyftrl for the final bound privacy-utility bounds (\Cref{cor:dpsgd} for \privsgd and \Cref{cor:dpftrl} for \dpftrl).
\end{itemize}

\subsection{Setup, Assumptions, and Notation}
\label{sec:finite-time:setup}

In this section, we fix the precise notation and assumptions. We also give some preliminary results.

\subsubsection{Assumptions}
We make the following assumptions throughout this section.

\begin{assumption} \label{asmp:linear-regression:conc}
    The data distribution $\Pdata$ satisfies the following:
    \begin{enumerate}[label=(\textbf{B\arabic*})]
        \item \label[asmp]{asmp:input:conc}
        \textbf{Input Distribution}: The inputs have mean $\expect[\bfx] = \boldzero$ and covariance $\expect[\bfx \otimes \bfx] =: \bfH$. We have $\mu \bfI  \preceq \bfH \preceq L \bfI$ for $\mu, L > 0$. Further, $\bfH^{-1/2}\bfx$ is element-wise independent and sub-Gaussian with variance proxy $1$, e.g. $\bfH^{-1/2}\bfx \sim \calN(0, \bfI)$.
        \item \label[asmp]{asmp:noise:conc}
        \textbf{Noise Distribution}: There exists a $\bftheta_\star \in \reals^d$ such that $y = \inp{\bftheta_\star}{\bfx} + \xi$, where $\xi$ is independent of $\bfx$ and is zero-mean sub-Gaussian with variance proxy $\sigmasgd^2$, e.g. $\xi \sim \calN(0, \sigmasgd^2)$.
    \end{enumerate}
\end{assumption}

These assumptions are a strengthening of \Cref{asmp:linear-regression} which are necessitated by concentration arguments to follow below.

\subsubsection{Notation}

\begin{itemize}
    \item As in \Cref{asmp:linear-regression}, we denote $R^2$ as the smallest number such that the fourth moment of $\bfx$ is bounded as
    \begin{align} \label{eq:R2:def}
        \expect\left[\norm{\bfx}_2^2 \, \bfx \otimes \bfx\right] \preceq R^2 \bfH \,.    
    \end{align} 
    Under \Cref{asmp:input:conc}, we have $R^2 = \Theta(\tr{\bfH})$ always. 
        While $\tr{\bfH} \le R^2$ directly follows from \eqref{eq:R2:def} using Jensen's inequality, we show that $R^2 \le 3 \tr{\bfH}$ in \Cref{prop:m4-gaussian} in \Cref{sec:stationary-setup}.
    \item It is convenient to rewrite the \noisyftrl recursion \eqref{eq:dp-ftrl-finite} in terms of the difference $\bftheta_t' := \bftheta_t - \bftheta_\star$ as 
    \begin{align} \label{eq:dp-ftrl-finite-a}
        \bftheta_{t+1}' =  \big( \bfI - \eta (\bfx_t \otimes \bfx_t) \big) \bftheta_t' + \eta\, \xi_t \bfx_t - \eta \sum_{\tau=0}^t \beta_\tau \bfw_{t-\tau} \,.
    \end{align}
    We will show in the upcoming \Cref{prop:bvd} that $\bftheta_t' = \hat\bftheta_t + \thetasgd + \thetadp$, where $\hat\bftheta_t$ captures the effect of the initial iterate, $\thetasgd$ captures the effect of the SGD noise, and $\thetadp$ captures the effect of the additive DP noise. We will define these quantities now and state and prove \Cref{prop:bvd} later.
    Note that these recursions are defined for the same sequences of input realizations $(\bfx_0, \bfx_1, \ldots)$ drawn from $\Pdata$, linear model noise realizations $(\xi_0, \xi_1, \ldots)$, and DP noise realizations $(\bfw_0, \bfw_1, \ldots)$.
    \item We define the noise-free version of the \dpftrl recursion as $\hat \bftheta_0 = \bftheta'_0$ and
    \begin{align} \label{eq:dp-ftrl-finite-bias}
        \hat \bftheta_{t+1} =  \big( \bfI - \eta (\bfx_t \otimes \bfx_t) \big) \hat \bftheta_t \,.
    \end{align}
    \item The effect of the SGD noise in the \noisyftrl process can be quantified by creating a process starting from $\thetasgd_0 = \boldzero$ with no DP noise (i.e. $\bfw_\tau\equiv \boldzero$):
    \begin{align} \label{eq:dp-ftrl-finite:sgd-noise}
        \thetasgd_{t+1} =  \big( \bfI - \eta (\bfx_t \otimes \bfx_t) \big) \thetasgd_t + \eta\, \xi_t \bfx_t \,.
    \end{align}
    \item The effect of the DP noise in the \noisyftrl process can be quantified by creating a process starting from $\thetadp_0 = \boldzero$ with no SGD noise (i.e., $\xi_t \equiv 0$):
    \begin{align} \label{eq:dp-ftrl-finite:dp-noise}
        \thetadp_{t+1} =  \big( \bfI - \eta (\bfx_t \otimes \bfx_t) \big) \thetadp_t  - \eta \sum_{\tau=0}^t \beta_\tau \bfw_{t-\tau} \,.
    \end{align}
    \item For an input $\bfx_t$ drawn from $\Pdata$ We define the matrix
    \begin{align} \label{eq:Q-def}
        \bfQ_t := \bfI - \eta \bfx_t \otimes \bfx_t \,.
    \end{align}
    Note that $\expect[\bfQ_t] = \bfI - \eta\bfH$.
    
    \item Define the linear operator $\calP : \mathbb{S}^d_+ \to \mathbb{S}^d_+$ that operates on the cone of PSD matrices given by 
    \begin{align} \label{eq:P-def}
        \calP \bfM = \expect[(\bfI - \eta \bfx \otimes \bfx) \bfM (\bfI - \eta \bfx \otimes \bfx)]  \,,
    \end{align}
    where $\bfx$ is an input drawn from $\Pdata$.
    By definition, we have $\expect[\bfQ_t \bfM \bfQ_t] = \calP \bfM$ and by independence,
    \begin{align} \label{eq:P-powers}
        \expect[\bfQ_{t} \bfQ_{t-1} \bfM \bfQ_{t-1}\bfQ_t] = \calP (\calP \bfM) = \calP^2 \bfM \,.
    \end{align}
    This extends to higher powers of $\calP$ as well. Finally, we will heavily use the fact that $\tr{\calP \bfM} \le (1 - \eta \mu)\tr{\bfM}$ for PSD matrices $\bfM$ (see \Cref{lem:fourth-order-contraction} for a proof).
    
    \item For each iteration $t$, we define the PSD matrix $\bfSigma_t^\SGD$ as 
    \begin{align} 
    \label{eq:SGDcovar:def}
        \bfSigma_t^\SGD &= \bfx_{t-1} \otimes \bfx_{t-1} + \bfQ_{t-1} (\bfx_{t-2} \otimes \bfx_{t-2}) \bfQ_{t-1} + \cdots + \bfQ_{t-1} \cdots \bfQ_1 (\bfx_0 \otimes \bfx_0) \bfQ_1 \cdots \bfQ_{t-1} \,,
    \end{align}
    \item For each iteration $t$, we define the PSD matrix $\bfSigma_t^\DP$ as 
    \begin{align} 
    \label{eq:DPcovar:def}
    \begin{aligned}
        \bfSigma_t^\DP &= \sum_{\tau=0}^{t-1} \bfV_{t, \tau} \bfV_{t, \tau}\T 
        \quad \text{where} \quad \\
        \bfV_{t, \tau} &= \begin{cases}
            \beta_\tau \bfI + \beta_{\tau-1} \bfQ_{t-1} + \cdots + \beta_0 \bfQ_{t-1} \cdots \bfQ_{t-\tau} \,, & 
            \text{ if } 1 \le \tau \le t-1 \,, \\
            \beta_0 \bfI\,, & \text{ if } \tau = 0 \,.
        \end{cases}
    \end{aligned}
    \end{align}
\end{itemize}

\subsubsection{Preliminary Results}

The first result is a decomposition of the \noisyftrl process into three processes: (a) gradient descent without additive noise, (b) a noise process with only noise from the linear model, and (c) a noise process with only the DP noise.
\begin{property}\label{prop:bvd}
    For the sequences $\bftheta_t', \hat\bftheta_t, \thetasgd_t, \thetadp_t$ defined in \Cref{eq:dp-ftrl-finite-a,eq:dp-ftrl-finite-bias,eq:dp-ftrl-finite:sgd-noise,eq:dp-ftrl-finite:dp-noise}, we have the following:
    \begin{gather}
    \label{eq:bvd}
        \bftheta_t' = \hat\bftheta_t + \thetasgd_t +  \thetadp_t  \\
    \label{eq:bvd:bias}
        \hat \bftheta_t = \,\, \bfQ_t \cdots \bfQ_0 \bftheta_0' \\
     \label{eq:bvd:sgd-var}
        \thetasgd_t = \eta \left(
        \bfx_t \xi_t + \bfQ_t \bfx_{t-1}\xi_{t-1} + \cdots + \bfQ_t \cdots \bfQ_1 \bfx_0 \xi_0
        \right)  \\
    \label{eq:bvd:dp-var}
        \begin{aligned}
         \thetadp_t &= -\eta \left(
            \sum_{\tau=0}^t \beta_\tau \bfw_{t-\tau} + \bfQ_t \sum_{\tau=0}^{t-1} \beta_\tau \bfw_{t-1-\tau} + \cdots + 
            \bfQ_t \cdots \bfQ_1 (\beta_0 \bfw_0)
        \right)\\
            &=  -\eta \Big(
                \beta_0 \bfw_{t-1} + (\beta_1 \bfI  +  \beta_0\bfQ_{t-1}) \bfw_{t-2}
                + \cdots + (\beta_{t-1} \bfI  + \beta_{t-2} \bfQ_{t-1} + \cdots +  \beta_0 \bfQ_{t-1}\cdots \bfQ_1) \bfw_0
            \Big) \,.
        \end{aligned}
    \end{gather}
\end{property}
\begin{proof}
    The expressions follow from unrolling their respective updates.
    By unrolling the \dpftrl update \eqref{eq:dp-ftrl-finite-a}, we get,
    \begin{align*}
        \bftheta_{t+1}' &= \bfQ_t \bftheta_t' + \eta \bfx_t \xi_t - \eta \sum_{\tau=0}^t \beta_\tau \bfw_{t-\tau} \\
        &= \bfQ_t \bfQ_{t-1} \bftheta_{t-1}' + \eta
        \left(\bfx_t \xi_t + \bfQ_t \bfx_{t-1}\xi_{t-1} \right)
            - \eta\left(\sum_{\tau=0}^t \beta_\tau \bfw_{t-\tau} + \bfQ_t \sum_{\tau=0}^{t-1} \beta_\tau \bfw_{t-1-\tau}
        \right) \\
        &= \bfQ_t \cdots \bfQ_0 \bftheta_0' + 
         \eta \left(
        \bfx_t \xi_t + \bfQ_t \bfx_{t-1}\xi_{t-1} + \cdots + \bfQ_t \cdots \bfQ_1 \bfx_0 \xi_0
        \right) \\
        &\qquad -\eta \left(
            \sum_{\tau=0}^t \beta_\tau \bfw_{t-\tau} + \bfQ_t \sum_{\tau=0}^{t-1} \beta_\tau \bfw_{t-1-\tau} + \cdots + 
            \bfQ_t \cdots \bfQ_1 (\beta_0 \bfw_0)
            \right)
        \,.
    \end{align*}
    Unrolling \Cref{eq:dp-ftrl-finite-bias,eq:dp-ftrl-finite:sgd-noise,eq:dp-ftrl-finite:dp-noise} respectively gives \Cref{eq:bvd:bias,eq:bvd:sgd-var,eq:bvd:dp-var}, and comparing them with the expression above gives \Cref{eq:bvd}.
\end{proof}

\subsection{High-Probability Bounds on \noisyftrl} \label{sec:finite-time:high-probability}

The goal of this subsection is to prove a high probability bound on norms of the iterates of \noisyftrl. We require a technical convergence condition on the weights $\bfbeta$.

\begin{defn} \label{def:hed}
    A sequence $\bfbeta = (\beta_0, \beta_1, \ldots)$ is said to satisfy \hed with parameter $\nu \in (0, 1)$ if for all nonnegative integers $\tau$, we have
    \begin{align} \label{eq:hed:def}
        |\beta_0| (1 - \nu)^{\tau / 2}
        + |\beta_1| (1 - \nu)^{(\tau-1)/2}
        + \cdots + | \beta_\tau| \le C ( 1- \nu)^{\tau/2}
    \end{align}
    for a universal constant $C > 0$.
\end{defn}

\begin{thm} \label{thm:high-prob-bounds}
    Fix a constant $0 < \failprob < 1$ and suppose the \Cref{asmp:linear-regression:conc} holds.
    Consider the sequence $(\bftheta_t)_{t=0}^{T-1}$ of iterates and the sequence $(\bfg_t)_{t=0}^{T-1}$ of gradients when running \noisyftrl for $T$ iterations with noise coefficients $\bfbeta = (\beta_0, \ldots, \beta_{T-1})$, DP noise $\bfw_{t} \sim \calN(\boldzero, \sigma^2 \bfI)$ of a given variance\footnote{
        In the context of this paper, we have 
        $\sigma^2 = G^2 \gamma(\bfbeta)^2 / (2\rho)$.
    } $\sigma^2$, a learning rate $\eta \le \big(c R^2 \log(T/\failprob)\big)$ for a universal constant $c \ge 1$. Further, suppose that  $\bfbeta$ satisfies \hed with parameter $\nu$ for some $\nu \le \eta\mu$.
    Then, with probability at least $1-\failprob$, we have
    \begin{align*}
        \norm{\bftheta_t'}_2^2 &\le C \left( 
            \norm{\bftheta_0'}_2^2
            + \frac{\eta R^2 \sigmasgd^2}{\mu} 
            + \frac{\eta^2 \sigma^2 d \, \norm{\bfbeta}_1^2}{\nu}
        \right) \log^3\left(\frac{T}{\failprob} \right)
        \quad \text{and} \quad \\
        \norm{\bfg_t}_2^2 &\le C R^4 \left( \norm{\bftheta_0'}_2^2
            + \frac{\eta R^2 \sigmasgd^2}{\mu} + \frac{\sigmasgd^2}{R^2} + \frac{\eta^2 \sigma^2 d \norm{\bfbeta}_1^2}{\nu}
        \right) \log^5\left( \frac{T}{\failprob} \right) \,.
    \end{align*}
    for a universal constant $C$.
\end{thm}
We prove this theorem over a sequence of intermediate results.

\subsubsection{Proof Setup: Definition of Events}

The proof strategy relies on defining some events (that hold with high probability from concentration of measure) and proving the required boundedness under those events.
Consider $0 < \failprob < 1$ and a universal constant $C$ from statement of \Cref{thm:high-prob-bounds}.
We define the following events.
\begin{itemize}
    \item Define the event where the inputs are bounded in norm as:
    \begin{align} \label{eq:E1:def}
        \calE_1 := \bigcap_{t=0}^{T-1} \left\{ 
            \norm{\bfx_t}_2^2 \le C R^2 \log\left(\frac{T}{\failprob} \right)
        \right\} \,.
    \end{align}
    \item Define an event where the noise in the linear model is bounded as:
    \begin{align} \label{eq:E2:def}
        \calE_2 := \bigcap_{t=0}^{T-1} \left\{
            |\xi_t|^2 \le 2 \sigmasgd^2 \log\left( \frac{2T}{\failprob}\right)
        \right\} \,.
    \end{align}
    \item Define the event where the norm of $\thetasgd$ defined in \eqref{eq:dp-ftrl-finite:sgd-noise} is bounded
    \begin{align} \label{eq:E1sgd:def}
        \calE_1^\SGD := \bigcap_{t=0}^{T-1} \left\{
            \norm{\thetasgd}_2^2 \le C \eta^2 \sigmasgd^2 \, \tr{\bfSigma_t^\SGD} \log\left(\frac{T}{\failprob}\right)
        \right\} \,,
    \end{align}
    where we define the random matrix
    $\bfSigma_t^\SGD = \bfx_{t-1} \otimes \bfx_{t-1} + \bfQ_{t-1} (\bfx_{t-2} \otimes \bfx_{t-2}) \bfQ_{t-1} + \cdots + \bfQ_{t-1} \cdots \bfQ_1 (\bfx_0 \otimes \bfx_0) \bfQ_1 \cdots \bfQ_{t-1}$ (see also \eqref{eq:SGDcovar:def}). When this event holds, we have that $\boldzero \preceq \bfQ_t \preceq \bfI$ for $t = 0, \ldots, T-1$ as long as $\eta \le 1 / \left(C R^2 \log(T/\failprob)\right)$. Indeed, in this case, we have
    \begin{align} \label{eq:Qt-psd}
        \bfI - \eta \bfx_t \otimes \bfx_t
        \succeq \left(1 - \eta \norm{\bfx_t}_2^2\right) \bfI  \succeq \boldzero \,.
    \end{align}
    \item The components of the sum defining $\bfSigma_t^\SGD$ are the PSD matrices $\bfW_{t, \tau}$, defined for $\tau \le t-1$ as
    \begin{align} \label{eq:W-t-tau:def}
        \bfW_{t, \tau} = \begin{cases}
            \bfQ_{t-1}\cdots Q_{\tau+1} (\bfx_{\tau} \otimes \bfx_\tau) \bfQ_{\tau+1} \cdots \bfQ_{t-1} \,, & \text{ if } \tau < t-1, \\
            \bfx_{t-1} \otimes \bfx_{t-1} \,, & \text{ if } \tau = t-1 \,.
        \end{cases}
    \end{align}
    Define the event where these are bounded in trace as
    \begin{align} \label{eq:E2sgd:def}
        \calE_2^\SGD := \bigcap_{t=0}^{T-1} \bigcap_{\tau=0}^{t-1}
        \left\{
            \tr{\bfW_{t, \tau}} \le 
            \frac{T^2 R^2}{\failprob} (1 - \eta \mu)^{t- 1 - \tau}
        \right\} \,.
    \end{align}
    \item Define the event where the norm of $\thetadp$ defined in \eqref{eq:dp-ftrl-finite:dp-noise} is bounded as
    \begin{align} \label{eq:E1dp:def}
        \calE_1^\DP := \bigcap_{t=0}^{T-1}
        \left\{
            \norm{\thetadp_t}_2^2 \le C \eta^2 \sigma^2 \, \tr{\bfSigma_t^\DP}\, \log\left(\frac{T}{\failprob}\right) 
        \right\} \,,
    \end{align}
    where $\bfSigma_t^\DP$ is defined in \eqref{eq:DPcovar:def}.
    \item Define the event where the matrix $\bfV_{t, \tau}$ defined in \eqref{eq:DPcovar:def} is bounded in trace:
    \begin{align} \label{eq:E2dp:def}
        \calE_2^\DP :=  \bigcap_{t=0}^{T-1} \bigcap_{\tau=0}^{t-1}
        \left\{
            \tr{\bfV_{t, \tau} \bfV_{t, \tau}\T}
            \le \frac{T^2 d}{\failprob} \left(
                \sum_{k=0}^\tau |\beta_k| (1 - \eta\mu)^{(\tau - k) / 2}
            \right)
        \right\} \,.
    \end{align}
\end{itemize}

We show that all these events hold with high probability.
\begin{prop} \label{prop:high-prob-events}
    Consider the setting of \Cref{thm:high-prob-bounds}. We have,
    \[
        \mathbb{P}\left(\calE_1 \cap \calE_2 \cap \calE_1^\SGD \cap \calE_2^\SGD \cap \calE_1^\DP \cap \calE_2^\DP\right)) \ge 1 - 6\failprob \,.
    \]
\end{prop}
\begin{proof}
    We will show that each of the events holds with probability at least $1-\failprob$ and a union bound gives the desired result.
    
    \mypar{Event $\calE_1$}
    Since $\bfz_t = \bfH^{-1/2} \bfx_t$ is element-wise independent and 1-sub-Gaussian, we have from the Hanson-Wright inequality (\Cref{lem:hanson-wright}) that
    \[
        \mathbb{P}(\norm{\bfx_t}_2^2 > C \tr{\bfH} \log(1/\failprob))
        = \mathbb{P}(\inp{\bfz_t}{\bfH \bfz_t} > C \tr{\bfH} \log(1/\failprob)) \le \failprob \,.
    \]
    Taking a union bound over $t = 0, 1, \ldots, T-1$ gives that $\mathbb{P}(\calE_1) \ge 1- \failprob$.
    
    \mypar{Event $\calE_2$} Since $\xi_t$ is sub-Gaussian with mean zero and variance proxy $\sigmasgd^2$, we have,
    \[
        \mathbb{P}(|\xi_t| > s) \le 2 \exp\left(-\frac{s^2}{2 \sigmasgd^2} \right) \,.
    \]
    Setting the right side equal to $\failprob / T$ and taking a union bound over $t = 0, 1, \ldots, T-1$ gives $\mathbb{P}(\calE_2) \ge 1- \failprob$.
    
    \mypar{Event $\calE_1^\SGD$}
    From the expression for $\thetasgd_t$ from \eqref{eq:bvd:sgd-var}, we can say that 
    $\thetasgd_t$ conditioned on $\bfx_0, \ldots, \bfx_{t-1}$ is mean zero and satisfies 
    \[
        \thetasgd_t = \eta 
        \underbrace{\begin{bmatrix}
            \bfx_{t-1}& \bfQ_{t-1} \bfx_{t-1} & \cdots & (\bfQ_{t-1} \cdots \bfQ_1 \bfx_0)
        \end{bmatrix}}_{=:\bfM_t}
        \begin{bmatrix}
            \xi_{t-1} \\
            \vdots \\
            \xi_0
        \end{bmatrix} \,.
    \]
    Using the assumption that each $\xi_\tau$ is independent and sub-Gaussian with variance proxy $\sigmasgd^2$, we get from the 
    Hanson-Wright inequality (\Cref{lem:hanson-wright}) again that 
    \[
        \mathbb{P}\left( \norm{\thetasgd_t}_2^2 > C \eta^2 \sigmasgd^2 \, \tr{\bfM_t \bfM_t\T} \log(1/\failprob) \right) = 
        \mathbb{P}\left( \inp*{\bfxi_{:t}}{\bfM_t\bfM_t\T \bfxi_{:t}} > C \eta^2 \sigmasgd^2 \, \tr{\bfM_t \bfM_t\T} \log(1/\failprob) \right) \le \failprob \,.
    \]
    Next, we confirm that 
    \[
        \tr{\bfM_t \bfM_t\T}
        = \norm{\bfx_{t-1}}_2^2 
        + \norm{\bfQ_{t-1}\bfx_{t-1}}_2^2 + \cdots + \norm{\bfQ_{t-1}\cdots \bfQ_1 \bfx_0}_2^2
        = \tr{\bfSigma_t^\SGD} \,.
    \]
    Finally, a union bound over $t = 0, 1, \ldots, T-1$ gives that $\mathbb{P}(\calE_1^\SGD) \ge 1- \failprob$.
    
    \mypar{Event $\calE_2^\SGD$}
    Markov's inequality gives
    \[
        \mathbb{P}\left(
            \tr{\bfW_{t, \tau}} > s
        \right)
        \le \frac{1}{s} \expect\left[\bfW_{t, \tau} \right] \le (1 - \eta \mu)^{t-1-\tau}\frac{R^2}{s}
    \]
    where the calculations for the expected bound are deferred to \Cref{lem:expect:w-t-tau}.
    Taking a union bound over all $T(T+1)/2 \le T^2$ choices of $(t, \tau)$ gives $\mathbb{P}(\calE_2^\SGD) \ge 1-\failprob$.
    
    \mypar{Event $\calE_1^\DP$}
    From the expression for $\thetadp_t$ from \eqref{eq:bvd:dp-var}, we deduce that
    \[
        \thetadp_t \, | \, \bfx_0, \ldots, \bfx_{t-1} \, \sim \, 
        \calN(\boldzero, \eta^2 \sigma^2 \bfSigma_t^\DP) \,.
    \]
    Invoking the Hanson-Wright inequality (\Cref{lem:hanson-wright}) and union bounding over $t=0, \ldots, T-1$ gives $\mathbb{P}(\calE_1^\DP) \ge 1- \failprob$.
    
    \mypar{Event $\calE_2^\DP$}
    Markov's inequality gives
    \[
        \mathbb{P}\left(
            \tr{\bfV_{t, \tau} \bfV_{t, \tau}\T} > s
        \right)
        \le \frac{1}{s} \expect\left[\bfV_{t, \tau} \bfV_{t, \tau}\T \right] \le 
        \left( \sum_{k=0}^\tau |\beta_k| (1 - \eta\mu)^{(\tau-k)/2} \right)\frac{d}{s} 
    \]
    where we defer the technical calculations involved in bounding the expectation above to \Cref{lem:v-t-tau}.
    Taking a union bound over all $T(T+1)/2 \le T^2$ choices of $(t, \tau)$ gives $\mathbb{P}(\calE_2^\DP) \ge 1-\failprob$.
\end{proof}

\subsubsection{High Probability Bounds on Component Recursions}

\mypar{Bound on the noise-less iterates} We start with $\hat \bftheta_t$ from \eqref{eq:dp-ftrl-finite-bias}.

\begin{prop} \label{prop:high-prob-bias}
    Under event $\calE_1$ and if $\eta \le (C R^2 \log(T/\failprob))^{-1}$, we have that $\norm{\hat \bftheta_t}_2 \le \norm{\bftheta_0'}_2$.
\end{prop}
\begin{proof}
    Using the fact that $\boldzero \preceq \bfQ_t \preceq \bfI$ under $\calE_1$ (cf. \Cref{eq:Qt-psd}), we get
    \[
        \norm{\hat \bftheta_t}_2 = 
        \norm{\bfQ_{t-1}\cdots \bfQ_0 \bftheta_0'}_2
        \le \norm{\bfQ_{t-1}}_2 \cdots \norm{\bfQ_0}_2 \norm{\bftheta_0'}_2 \le \norm{\bftheta_0'}_2 \,.
    \]
\end{proof}

\mypar{Bound on $\thetasgd_t$} We turn to $\thetasgd_t$ from \eqref{eq:dp-ftrl-finite:sgd-noise}.

\begin{prop} \label{prop:high-prob-sgd-noise}
    Under events $\calE_1, \calE_1^\SGD, \calE_2^\SGD$, and $\eta \le (CR^2 \log(T/\failprob))^{-1}$, we have
    \[
        \norm{\thetasgd_t}_2^2 \le C \left(\frac{\eta R^2}{\mu}\right) \log^3 \left(\frac{T}{\failprob}\right) \,.
    \]
\end{prop}
\begin{proof}
    Under $\calE_1^\SGD$, we have
    \begin{align} \label{eq:proof:hp-bound:sgd-1}
         \norm{\thetasgd}_2^2 \le C \eta^2 \sigmasgd^2 \, \tr{\bfSigma_t^\SGD} \log\left(\frac{T}{\failprob}\right) \,.
    \end{align}
    We bound $\tr{\bfSigma_t} = \sum_{\tau=0}^{t-1} \tr{\bfW_{t, \tau}}$ for $\bfW_{t, \tau}$ defined in \eqref{eq:W-t-tau:def}.
    We have two bounds for $\tr{\bfW_{t, \tau}}$:
    \begin{enumerate}[label=(\alph*)]
        \item Using $\boldzero \preceq \bfQ_t \preceq \bfI$ under $\calE_1$  (cf. \Cref{eq:Qt-psd}), we bound
        \begin{align*}
            \tr{\bfW_{t, \tau}} = \norm{\bfQ_{t-1} \cdots \bfQ_{\tau+1} \bfx_\tau}_2^2
            \le \norm{\bfQ_{t-1}}_2^2 \cdots \norm{\bfQ_{\tau+1}}_2^2 \norm{\bfx_\tau}_2^2
            \le C R^2 \log(T/\failprob) \,.
        \end{align*}
        \item Under event $\calE_2^\SGD$, we have the bound
        \[
            \tr{\bfW_{t, \tau}} \le \frac{T^2 R^2}{\failprob}  (1 - \eta\mu)^{t-1-\tau} \,.
        \]
    \end{enumerate}
    Using the first bound for the last $\tau \le t-1$ iterations and the second bound for the rest, we get
    \begin{align*}
        \tr{\bfSigma_t^\SGD}
        &\le \sum_{k=0}^{t-\tau-1} \frac{T^2 R^2}{\failprob}  (1 - \eta\mu)^{t-1-\tau} \indi{\tau < t-1}
        + \tau \left( C R^2 \log(T/\failprob)\right) \\
        &\le  \frac{T^2 R^2}{\failprob}   (1 - \eta\mu)^\tau \sum_{k=0}^{t-\tau-1}(1 - \eta\mu)^{k} \indi{\tau < t-1}
        + \tau \left( C R^2 \log(T/\failprob)\right) \\
        &\le  \frac{T^2 R^2}{\failprob}   \frac{\exp(- \eta\mu\tau)}{\eta \mu}  \indi{\tau < t-1}
        + \tau \left( C R^2 \log(T/\failprob)\right) \,.
    \end{align*}
    Choosing $\tau = \min\left\{t-1, \frac{1}{\eta\mu} \log\left(\frac{T^2}{C \failprob \log(T/\failprob)}\right) \right\}$ as per \Cref{lem:tau-opt} gives
    \[
        \tr{\bfSigma_t^\SGD} 
        \le \frac{CR^2 \log(T/\failprob)}{\eta \mu}
            \left(
                1 + \log\left( \frac{T^2}{\failprob \log(T/\failprob)} \right)
            \right) 
        \le \frac{C' R^2}{\eta \mu} \log^2(T/\failprob)
    \]
    for some absolute constants $C, C'$.
    Plugging this back into \eqref{eq:proof:hp-bound:sgd-1} completes the proof.
\end{proof}

\mypar{Bound on $\thetadp_t$} We turn to $\thetadp_t$ from \eqref{eq:dp-ftrl-finite:dp-noise}.

\begin{prop} \label{prop:high-prob-dp-noise}
    Consider the setting of \Cref{thm:high-prob-bounds}.
    Under events $\calE_1, \calE_1^\DP, \calE_2^\DP$, and $\eta \le (CR^2 \log(T/\failprob))^{-1}$, we have
    \[
        \norm{\thetasgd_t}_2^2 \le C \left(\frac{\eta R^2}{\mu}\right) \log^3 \left(\frac{T}{\failprob}\right) \,.
    \]
\end{prop}
\begin{proof}
    Based on the bound on $\norm{\thetadp_t}_2$ from $\calE_1^\DP$, we bound $\tr{\bfSigma_t^\DP} = \sum_{\tau=0}^{t-1} \tr{\bfV_{t, \tau} \bfV_{t, \tau}\T}$. We bound each trace on the right side in two ways:
    \begin{enumerate}[label=(\alph*)]
        \item  We have $\tr{\bfV_{t, \tau} \bfV_{t, \tau}\T} \le \norm{\bfbeta}_1^2 d$ from \Cref{lem:v-t-tau}.
        \item Under $\calE_2^\DP$ and the assumption $(*)$ of \hed of $\beta$ with parameter $\nu \le \eta\mu$, we also have
        \begin{align*}
            \tr{\bfV_{t, \tau} \bfV_{t, \tau}\T}
            &\le \frac{T^2 d}{\failprob} \left( \sum_{\tau=0}^\tau |\beta_k| (1 - \eta\mu)^{(\tau -k) / 2} \right)^2 \\
            &\le \frac{T^2 d}{\failprob} \left( \sum_{\tau=0}^\tau |\beta_k| (1 - \nu)^{(\tau -k) / 2} \right)^2 \\
            &\stackrel{(*)}{\le} \frac{C T^2 d}{\failprob} (1 - \nu)^{\tau} \,.
        \end{align*}
    \end{enumerate}
    Using the first bound for the first $\tau$ iterations and the second bound for the rest, we get
    \begin{align*}
        \tr{\bfSigma_t^\DP} 
        &\le \tau \left( \norm{\bfbeta}_1^2 d \right) + \sum_{k=\tau}^{t-1} \frac{C T^2 d}{\failprob} (1 - \nu)^{k} \indi{\tau > t-1} \\
        &\le \tau \left( \norm{\bfbeta}_1^2 d \right) +  \frac{C T^2 d}{\failprob} (1 - \nu)^{\tau} \sum_{k=0}^{\infty}  (1 -  \nu)^{k} \indi{\tau > t-1} \\
        &\le \tau \left( \norm{\bfbeta}_1^2 d \right) +  \frac{C T^2 d \exp(- \nu \tau)}{\failprob \nu} \indi{\tau > t-1} \,.
    \end{align*}
    Choosing $\tau \le \left\{ t-1, \frac{1}{\nu} \log(CT^2 / \failprob\norm{\bfbeta}_1^2) \right\}$ as per \Cref{lem:tau-opt}, we get,
    \begin{align*}
        \tr{\bfSigma_t^\DP} \le \frac{\norm{\bfbeta}_1^2 d}{ \nu} 
        \left( 1 + \log\left( \frac{CT^2}{\failprob \norm{\bfbeta}_1^2} \right) \right) 
        \le C' \frac{\norm{\bfbeta}_1^2  d}{\nu} \log\left( \frac{T}{\failprob} \right) \,,
    \end{align*}
    where we used $\norm{\bfbeta}_1 \ge |\beta_0| = 1$ and $C, C'$ are some universal constants.
    Combining this with the bound on $\norm{\thetadp_t}_2$ asserted by $\calE_1^\DP$ completes the proof.
\end{proof}

\subsubsection{Completing the Proof of the High Probability Bounds}
We are now ready to prove \Cref{thm:high-prob-bounds}.

\begin{proof}[Proof of \Cref{thm:high-prob-bounds}]
    Under events $\calE_1, \calE_1^\SGD, \calE_2^\SGD, \calE_1^\DP, \calE_2^\DP$, we have bounds on the norms of $\hat\bftheta_t, \thetasgd_t, \thetadp_t$ respectively from Propositions \ref{prop:high-prob-bias} to
    \ref{prop:high-prob-dp-noise}.
    We combine them with the triangle inequality and \Cref{eq:bvd} of \Cref{prop:bvd} to the claimed bound on $\norm{\bftheta_t'}_2$.
    
    Next, for the gradients, we use the triangle and Cauchy-Schwarz inequalities on the definition $\bfg_t = \bfx_t \inp{\bfx_t}{\bftheta_t'} - \bfx_t \xi_t$ to get
    \[
        \norm{\bfg_t}_2^2 \le 2 \, \norm{\bfx_t}_2^4 \norm{\bftheta_t'}_2^2 + 2 \norm{\bfx_t}_2^2 |\xi_t|_2^2 \,.
    \]
    Plugging in the bounds on $\norm{\bfx_t}_2$ and $|\xi|_t$ from $\calE_1$ and $\calE_2$ respectively gives the claimed bound on $\norm{\bfg_t}_2^2$.
    
    Finally, all the events above hold with probability at least $1- 6\failprob$ from \Cref{prop:high-prob-events}. Substituting $\failprob/6$ for $\failprob$ and adjusting the constants completes the proof.
\end{proof}

\subsubsection{Helper Lemmas}

\begin{lem} \label{lem:expect:w-t-tau}
    Consider the setting of \Cref{thm:high-prob-bounds} and consider the PSD matrices $\bfW_{t, \tau}$, defined for $\tau \le t-1$ as
    \[
        \bfW_{t, \tau} = \begin{cases}
            \bfQ_{t-1}\cdots Q_{\tau+1} (\bfx_{\tau} \otimes \bfx_\tau) \bfQ_{\tau+1} \cdots \bfQ_{t-1} \,, & \text{ if } \tau < t-1, \\
            \bfx_{t-1} \otimes \bfx_{t-1} \,, & \text{ if } \tau = t-1 \,.
        \end{cases}
    \]
    We have that $\expect[\tr{\bfW_{t, \tau}}] \le R^2 (1 - \eta\mu)^{t-1-\tau}$.
\end{lem}
\begin{proof}
    For $\tau = t-1$, we have $\expect[\bfW_{t, t-1}] = \tr{\bfH} \le R^2$. For $\tau < t-1$, we have by independence of each $\bfx_t$ that
    \begin{align*}
        \tr{\expect[\bfW_{t, \tau}]}
        &= \tr{\expect[\bfQ_{t-1} \cdots \bfQ_{\tau+1} \bfH \bfQ_{\tau+1} \cdots \bfQ_{t-1}]}
        = \tr{\expect[\bfQ_{t-1} \cdots \bfQ_{\tau} (\calP\bfH) \bfQ_{\tau} \cdots \bfQ_{t-1}]}
        = \cdots 
        \\ &= \tr{\calP^{t-1-\tau} \bfH } \,.
    \end{align*}
    Recursively bounding $\tr{\calP^\tau H} = \tr{\calP (\calP^{\tau-1} \bfH)} \le (1 - \eta\mu) \tr{\calP^{\tau-1} \bfH}$ from \Cref{lem:fourth-order-contraction} completes the proof. 
\end{proof}

\begin{lem} \label{lem:v-t-tau}
    Consider $\bfV_{t, \tau}$ as defined in \eqref{eq:DPcovar:def}. We have that
    \[
        \expect\left[ \tr{\bfV_{t, \tau} \bfV_{t, \tau}\T} \right]
        \le d \left( \sum_{k=0}^\tau |\beta_k| (1 - \eta\mu)^{(\tau-k)/2} \right) \,.
    \]
    Further, if the event $\calE = \cap_{\tau=1}^t \{\bfQ_t \succeq \boldzero\}$ holds, then we also have
    \[
        \tr{\bfV_{t, \tau}\bfV_{t, \tau}\T} \le d \left( \sum_{k=0}^\tau |\beta_k| \right)^2 \,.
    \]
\end{lem}
\begin{proof}
    Since $t$ is fixed throughout, we simply write $\bfV_{t, \tau}$ as $\bfV_\tau$.
    We define a sequence of matrices $\bfA_0, \ldots, \bfA_\tau$ as $\bfA_0 = \beta_0 \bfI$ and 
    \[
        \bfA_{k+1} = \beta_{k+1} \bfI + \bfQ_{t-\tau + k} \bfA_k
    \]
    for $k= 0, \ldots, \tau-1$.
    We first prove the expected bound followed by the absolute bound.
    
    \mypar{Expected bound}
    Then, we successively deduce the following.
    \begin{enumerate}[label=(\alph*)]
        \item We have $\bfA_k = \beta_k \bfI + \beta_{k-1} \bfQ_{t-\tau + k-1} + \dots + \beta_0 \bfQ_{t-\tau+k-1} \ldots \bfQ_{t-\tau}$ by simply unrolling the recursions.
        \item We immediately recognize that $\bfV_{\tau} = \bfA_{\tau}$.
        \item By independence of each $\bfQ_t$, taking an expectation of the expression in (a) gives
        \[
            \expect[\bfA_k] = \sum_{l=0}^k \beta_l (\bfI - \eta\bfH)^{k-l} \,.
        \]
        \item We establish a recursion
        \[
            \expect \tr{\bfA_{k+1}\bfA_{k+1}\T }
            \le d \beta_{k+1}^2  + 2 d |\beta_{k+1}| \sum_{l=0}^k |\beta_l| (1 - \eta\mu)^{k-l+1} + (1-\eta\mu) \expect \tr{\bfA_k \bfA_k\T}\,.
        \]
        
        Indeed, by expanding out the square of the recursion and using the independence of the $\bfx_t$'s, we get
    \begin{align*}
        \expect\tr{\bfA_{k+1}\bfA_{k+1}\T}
        &= \beta_{k+1}^2 \tr{\bfI} + 2 \beta_{k+1} \tr{(\bfI - \eta\bfH) \expect[\bfA_k]} + \tr{\calP(\expect[\bfA_k\bfA_k\T])} \\
        &\le
            d \beta_{k+1}^2  + 2 |\beta_{k+1}| \sum_{l=0}^k |\beta_l| \,\, \tr{(\bfI - \eta\bfH)^{k-l+1}} + (1 - \eta\mu) \expect\tr{\bfA_k\bfA_k\T} \,,
    \end{align*}
    where we plugged in the expression for $\expect[\bfA_k]$ from item (c) and used \Cref{lem:fourth-order-contraction} to bound the last term. Using $\boldzero \preceq \bfI - \eta\bfH \preceq (1- \eta\mu) \bfI$ gives the claimed expression.
        \item Using induction and the recursion from part (d), we prove that
        \[
            \expect \tr{\bfA_k \bfA_k\T}
            \le d\left(\sum_{l=0}^k |\beta_l| (1 - \eta\mu)^{(k-l)/2}\right)^2 \,.
        \]
        Together with part (b), this gives the desired result.
        
        Indeed, the base case holds because $\expect\tr{\bfA_0 \bfA_0\T} = \beta_0^2 d$. Supposing the induction hypothesis holds for some $k < \tau - 1$, we use the recursion of item (d) to get
        \begin{align*}
            \frac{1}{d}\,\, \expect \tr{\bfA_{k+1}\bfA_{k+1}\T }
            &\le \beta_{k+1}^2 + 2 |\beta_{k+1}| \sum_{l=0}^k |\beta_l| (1 - \eta\mu)^{k-l+1} + \left( \sum_{l=0}^k |\beta_l| (1 - \eta\mu)^{\frac{k-l+1}{2}} \right)^2 \\
            &\le \beta_{k+1}^2 + 2 |\beta_{k+1}| \sum_{l=0}^k |\beta_l| (1 - \eta\mu)^{\frac{k-l+1}{2}} + \left( \sum_{l=0}^k |\beta_l| (1 - \eta\mu)^{\frac{k-l+1}{2}} \right)^2 \\
            &= \left( \sum_{l=0}^{k+1} |\beta_l| (1 - \eta\mu)^{\frac{k-l+1}{2}}\right)^2 \,,
        \end{align*}
        where the second inequality used $1-\eta\mu \le 1$.
    \end{enumerate}
    
    \mypar{Absolute bound} Next, we prove the absolute bound, assuming that $\calE$ holds.
    Again, we successively deduce:
     \begin{enumerate}[label=(\alph*)]
        \item We starting with $\bfA_k = \beta_k \bfI + \beta_{k-1} \bfQ_{t-\tau + k-1} + \dots + \beta_0 \bfQ_{t-\tau+k-1} \ldots \bfQ_{t-\tau}$.
        \item Then, we get
        \[
            |\tr{\bfA_k}| 
            \le |\beta_k| d + |\beta_{k-1}| \, \left|\tr{\bfQ_{t-\tau+k-1}} \right| + \cdots + |\beta_0| \, \left|\tr{\bfQ_{t-\tau+k-1} \cdots \bfQ_{t-\tau}}\right|
            \le d \sum_{l=0}^k |\beta_l| \,,
        \]
        where we bound each of the traces by $d$ using \Cref{lem:holder-schatten} (since we have $\bfQ_k \preceq \bfI$ under $\calE$).
    \item By a similar logic, we get
    \begin{align*}
        \Big| & \tr{\bfQ_{t-\tau+k} \bfA_k + \bfA_k\T \bfQ_{t-\tau+k}} \Big| \\
        &\le 2 |\beta_k| \tr{\bfQ_{t-\tau+k}}
        + 2 |\beta_1| \, |\tr{\bfQ_{t-\tau + k} \bfQ_{t-\tau+k-1}}| + \cdots + 2 |\beta_0| \, |\tr{\bfQ_{t-\tau+k} \cdots \bfQ_{t-\tau}}| \\
        &\le 2d \sum_{l=0}^k |\beta_l| \,.
    \end{align*}
    \item We  prove by induction that $\tr{\bfA_k\bfA_k\T} \le d \left( \sum_{l=0}^k |\beta_l| \right)^2$.
    
    The base case holds since $\tr{\bfA_0\bfA_0\T} = d \beta_0^2$. Supposing the induction hypothesis holds for some integer $1\le k < t-1$, we use the recursion of $\bfA_{k+1}$ to calculate
    \begin{align*}
        \tr{\bfA_{k+1}\bfA_{k+1}\T}
        &= d \beta_{k+1}^2 + \beta_{k+1} \tr{\bfQ_{t-\tau+k}\bfA_k + \bfA_k\T Q_{t-\tau+k}} + \tr{\bfQ_{t-\tau+k} \bfA_k \bfA_k\T \bfQ_{t-\tau+k}} \\
        &\le d \beta_{k+1}^2 + 2d |\beta_{k+1}| \sum_{l=0}^k |\beta_l| + \tr{\bfA_k\bfA_k\T} \le d \left( \sum_{l=0}^{k+1} |\beta_l|\right)^2 \,.
    \end{align*}
    \end{enumerate}
    Finally, item (d) together with $\bfA_\tau = \bfV_{t, \tau}$ completes the proof.
\end{proof}

\subsection{Expected Bounds on \noisyftrl}
\label{sec:finite-time:expected}

Our goal in this section is to prove the following finite-time convergence guarantee of \noisyftrl in terms of the asymptotic suboptimality.

\begin{thm} \label{thm:lr-finite-noclip}
    Consider problem \eqref{eq:linear-regression} and suppose \Cref{asmp:linear-regression} holds.
    For a given a starting iterate $\bftheta_0 \in \reals^d$, weights $\bfbeta \in \ell^2$, learning rate $\eta < 1/R^2$, 
    consider the sequence $(\bftheta_t)_{t=0}^\infty$ produced by the iteration \eqref{eq:dp-ftrl-finite} where $\bfw_t \sim \calN(\boldzero, \sigma^2 \bfI)$ with $\sigma^2 = G^2 \gamma_\infty^2(\bfbeta) / (2\rho)$. Then, for any $t \ge 0$, we have,
    \[
        \expect\left[ 
            F(\bftheta_t) - F(\bftheta_\star)
        \right]
        \le \left(\sqrt{\tfrac{L}{\mu} \exp( - \eta \mu t) \left( F(\bftheta_0) - F(\bftheta_\star) \right)} +   \sqrt{F_\infty(\bfbeta)} \right)^2 \,.
    \]
\end{thm}

We start with some preliminary lemmas.
The first lemma is about the covariance of the noise process and is a generalization of \cite[Lemma 3]{jain2017markov} to linearly correlated additive noise.

\begin{lem} \label{lem:noise-cov-finite}
    Consider the sequence $(\tilde \bftheta_t)_{t=0}^\infty$ generated by \noisyftrl starting from $\tilde \bftheta_t = \bftheta_\star$ with noise coefficients $\bfbeta \in \ell^2$ and learning rate $\eta \le 1/R^2$. Under \Cref{asmp:linear-regression}, we have that its covariance
    \[
        \bfS_t := \expect\left[ \left(\tilde\bftheta_t - \bftheta_\star\right) \otimes \left(\tilde\bftheta_t - \bftheta_\star\right) \right] 
    \]
    satisfies: (a) $\bfS_t \preceq \bfS_{t+1}$ for all $t \ge 0$, and (b) the sequence $(\bfS_t)_{t=0}^\infty$ converges element-wise as $t \to \infty$.
\end{lem}
\begin{proof}
    Recall the notation $\bfQ_t = \bfI - \eta \bfx\otimes \bfx_t$ and $\calP \bfM = \expect[\bfQ_t \bfM \bfQ_t]$.
    We use the shorthand $\tilde \bftheta_t' := \tilde \bftheta_t - \bftheta_\star $.
    We first prove that the covariance is increasing in a PSD sense and argue that its limit exists.

    \mypar{Part 1: Non-decreasing noise}
    By unrolling the update equation and using $\tilde\bftheta_t' = \boldzero$, we get
    \begin{align} \label{eq:noise-cov-pf-1}
    \begin{aligned}
        \tilde \bftheta_t' =& \,\, 
        \eta\left( \bfx_{t-1} \xi_{t-1} + \bfQ_{t-1} \bfx_{t-2} \xi_{t-2} + \cdots + \bfQ_{t-1}\cdots \bfQ_1 \bfx_0 \xi_0 \right)  \\
        &\quad -
        \eta \Bigg(
        \beta_0 \bfw_{t-1} + (\beta_1 \bfI  +  \beta_0\bfQ_{t-1}) \bfw_{t-2}
        + \cdots + (\beta_{t-1} \bfI  + \beta_{t-2} \bfQ_{t-1} + \cdots +  \beta_0 \bfQ_{t-1}\cdots \bfQ_1) \bfw_0
        \Bigg) \,.
    \end{aligned}
    \end{align}
    Next, we calculate $\expect\left[\tilde \bftheta_t' \otimes \tilde \bftheta_t'\right]$. By independence, all the cross terms cancel out, so it suffices to write out the second moment of each of the terms above.
    For the SGD noise terms that contain $\bfx_\tau \xi_\tau$, we get for $\tau = 0, \ldots, t-1$ that 
    \begin{align} \label{eq:noise-cov-pf-2}
        \expect\left[ 
            (\bfQ_{t-1} \cdots \bfQ_{t-\tau+1} \bfx_{t-\tau}  \xi_{t-\tau})
            \otimes
            (\bfQ_{t-1} \cdots \bfQ_{t-\tau+1} \bfx_{t-\tau}  \xi_{t-\tau})
        \right]
         = \calP^\tau \left( \expect[\xi^2 \bfx \otimes \bfx] \right) =: \calT_\tau \,.
    \end{align}
    Since it is a second-moment term, we have $\calT_\tau \succeq \boldzero$.
    For the DP noise terms, denote $\bfx^{\otimes 2} = \bfx\otimes \bfx = \bfx \bfx^\top$. Then, we have for $\tau = 0$ to $t-1$ that
    \begin{align*}
        \frac{1}{\sigma^2} & \expect\left( (\beta_\tau \bfI + \beta_{\tau-1} \bfQ_{t-1} + \beta_{\tau-2} \bfQ_{t-1}\bfQ_{t-2}
        + \cdots + \beta_0 \bfQ_{t-1}\cdots \bfQ_{t-\tau}) \bfw_{t-\tau-1}\right)^{\otimes 2}  \\
        & = \expect\left(\beta_\tau \bfI + \beta_{\tau-1} \bfQ_{t-1} + \beta_{\tau-2} \bfQ_{t-1}\bfQ_{t-2}
        + \cdots + \beta_0 \bfQ_{t-1}\cdots \bfQ_{t-\tau} \right)^{\otimes 2} \\
        &= \beta_\tau^2 \bfI + 2 \beta_\tau \sum_{k=0}^{\tau-1} \beta_k (\bfI - \eta \bfH)^{\tau - k}
        + \sum_{k=0}^{\tau-1} \sum_{l=0}^{\tau-1} \beta_k \beta_l \,\, \expect\left[ 
            \bfQ_{t-1} \cdots \bfQ_{t-\tau + k} \bfQ_{t-\tau+l} \cdots \bfQ_{t-1}
        \right] \\
        &= \beta_\tau^2 \bfI  + 2 \beta_\tau \sum_{k=0}^{\tau-1} \beta_k (\bfI - \eta \bfH)^{\tau - k}
        + 2 \sum_{k=0}^{\tau-1} \sum_{l=0}^{k} \beta_k \beta_l \,\, \expect\left[ 
            \bfQ_{t-1} \cdots \bfQ_{t-\tau + l} (\bfI - \eta \bfH)^{k-l} \bfQ_{t-\tau+l} \cdots \bfQ_{t-1}
        \right] \\
        &= \beta_\tau^2 \bfI  + 2 \beta_\tau \sum_{k=0}^{\tau-1} \beta_k (\bfI - \eta \bfH)^{\tau - k} +  
            2  \sum_{k=0}^{\tau-1} \sum_{l=0}^{k} \beta_k \beta_l \,\, \calP^{\tau - k}\left( (\bfI - \eta \bfH)^{k-l} \right) =: \calT_{\tau}' \,.
        \numberthis 
        \label{eq:noise-cov-pf-3}
    \end{align*}
    By this being a second moment, we have that $\calT_\tau' \succeq \boldzero$. Plugging in \eqref{eq:noise-cov-pf-2} and \eqref{eq:noise-cov-pf-3} into the second moment of \eqref{eq:noise-cov-pf-1}, we get,
    \begin{align*}
        \expect\left[ \tilde \bftheta_{t+1}' \otimes \tilde \bftheta_{t+1}' \right] &= \eta^2 \sum_{\tau=0}^{t} (\calT_\tau +  \sigma^2 \calT_\tau') \\ 
        &= \expect\left[ \tilde \bftheta_{t}' \otimes \tilde \bftheta_{t}' \right] + \eta^2 (\calT_t + \sigma^2 \calT_t')
        \succeq  \expect\left[ \tilde \bftheta_{t}' \otimes \tilde \bftheta_{t}' \right] \,.
    \end{align*}
    This shows that the noise is non-decreasing in a PSD sense.
    
    \mypar{Part 2: Convergence of the covariance}
    Next, we show that the noise sequence converges. From the update equation $\tilde \bftheta_{t+1}' = \bfQ_t \tilde \bftheta_t' + \eta \bfx_t \xi_t - \eta \sum_{\tau=0}^t \beta_\tau \bfw_{t-\tau}$, we get
    \begin{align*}
        \bfS_{t+1} =& \,\, \calP \bfS_t + \eta^2 \expect[\xi^2 \bfx\otimes\bfx]
            + \eta^2 \sigma^2 \sum_{\tau=0}^t \beta_\tau^2 \bfI \\
            &\quad- \eta (\bfI - \eta \bfH) \sum_{\tau=0}^t \beta_\tau \expect\left[ \tilde\bftheta_t' \otimes \bfw_{t-\tau}\right]
            - \eta \sum_{\tau=0}^t \beta_\tau \expect\left[  \bfw_{t-\tau} \otimes \tilde\bftheta_t'  \right]  (\bfI - \eta \bfH) \,.
    \end{align*}
    For $\tau = 0$, the term $\expect[ \tilde\bftheta_t' \otimes \bfw_{t-\tau}]$ and its transpose are both $\boldzero$. For $\tau > 0$, we have from \eqref{eq:noise-cov-pf-1} that 
    \begin{align*}
        - \expect\left[ \tilde\bftheta_t' \otimes \bfw_{t-\tau}\right]
        &= \eta \expect\left[ 
            \beta_{\tau-1} \bfI + \beta_{\tau-2}\bfQ_{t-1} + \cdots + 
            \beta_0 \bfQ_{t-1} \cdots \bfQ_{t-\tau+1} 
        \right] \,\, \expect[\bfw_{t-\tau} \otimes \bfw_{t-\tau}] \\
        & = \eta \sigma^2 \left(
            \beta_{\tau-1} \bfI + \beta_{\tau-2}(\bfI - \eta\bfH) + \cdots + \beta_0 (\bfI - \eta \bfH)^{\tau-1}
        \right) \,.
    \end{align*}
    Plugging this back in gives
    \begin{align*}
        \bfS_{t+1} &= \calP \bfS_t + \eta^2 \expect[\xi^2 \bfx \otimes \bfx] + \eta^2 \sigma^2 \sum_{\tau=0}^t \beta_\tau^2 \bfI 
        + 2 \eta^2 \sigma^2 \sum_{\tau=1}^t \sum_{k=0}^{\tau-1} \beta_\tau\beta_k (\bfI - \eta \bfH)^{\tau-k} \\
        &= \calP \bfS_t + \eta^2 \expect[\xi^2 \bfx \otimes \bfx] + \eta^2 \sigma^2 \sum_{\tau=0}^t \sum_{k=0}^t \beta_\tau \beta_k (\bfI - \eta \bfH)^{|\tau-k|} \,.
    \numberthis \label{eq:noise-cov-pf-4}
    \end{align*}
    Next, we take a trace of \eqref{eq:noise-cov-pf-4}. For the first term, we get
    \begin{align*}
        \tr{\calP \bfS_t}
        &= \tr{\bfS_t} - 2\eta \tr{\bfH \bfS_t} + \eta^2 \tr{\bfS_t \expect[\norm{\bfx_t}_2^2 \bfx_t \otimes \bfx_t]} \\
        &\le \tr{\bfS_t} - \eta \tr{\bfH \bfS_t}(2 - \eta R^2) \\
        &\le (1 - \eta\mu) \tr{\bfS_t} \,,
    \end{align*}
    where we use (a) $\expect[\norm{\bfx_t}_2^2 \bfx_t\otimes \bfx_t] \preceq R^2 \bfH$, (b) $\eta \le 1/R^2$, and (c) $\bfH \succeq \mu \bfI$. By assumption, we also get that $\tr{\expect[\xi^2 \bfx \otimes \bfx]} \le \sigmasgd^2 \tr{\bfH} \le \sigmasgd^2 R^2$. Finally, we have using 
    \Cref{lem:helper-l2-beta} that 
    \[
        \sum_{\tau=0}^t \sum_{k=0}^t \beta_\tau \beta_k \sum_{j=1}^d (1 - \eta\lambda_j)^{|\tau - k|}
        \le \norm{\bfbeta}_2^2 \sum_{j=1}^d \left(\frac{2 - \eta\lambda_j}{\eta \lambda_j}\right) \le \frac{2 \norm{\bfbeta}_2^2 \tr{\bfH^{-1}}}{\eta} \,.
    \]
    Thus, we get
    \[
        \tr{\bfS_{t+1}} \le (1 - \eta \mu) \tr{\bfS_t} + 2 \eta \sigma^2 \norm{\bfbeta}_2^2 \, \tr{\bfH^{-1}} + \eta^2 R^2 \sigmasgd^2 \,.
    \]
    By unrolling this out, we get a uniform bound for all $t$:
    \[
        \tr{\bfS_t} \le \frac{1}{\mu}\left( 2 \sigma^2 \norm{\bfbeta}_2^2 \, \tr{\bfH^{-1}} + \eta R^2 \sigmasgd^2  \right) < \infty 
    \]
    since $\bfbeta \in \ell^2$.
    For any fixed vector $\bfv$, $\inp{\bfv}{\bfS_t\bfv}$ thus has a limit from the monotone convergence theorem. From this, it follows that every diagonal entry of $\bfS_t$ converges (take $\bfv$ as a standard basis vector) and then every off-diagonal entry of $\bfS_t$ also converges (take $\bfv$ as the sum of two standard basis vectors). This shows that $\bfS_t$ converges element-wise.
\end{proof}

We are now ready to prove \Cref{thm:lr-finite-noclip}.

\begin{proof}[Proof of \Cref{thm:lr-finite-noclip}]
Define $F_\infty^\star(\bfbeta)$ as the asymptotic suboptimality of a process that starts from $\bftheta_0 = \bftheta_\star$. 
We will prove the desired result with $F_\infty^\star(\bfbeta)$ in the place of $F_\infty(\bfbeta)$. Finally,  we will show that $F_\infty(\bfbeta)$ is independent of its starting iterate so $F_\infty(\bfbeta) = F_\infty^\star(\bfbeta)$.

    We first separate the effects of the noise and the initial iterate using \Cref{prop:bvd}. We invoke \Cref{lem:noise-cov-finite} for the former and directly bound the latter. Lastly, we combine them both with a triangle inequality. Recall that use the shorthand $\bftheta_t' := \bftheta_t - \bftheta_\star$ and
    $\bfQ_t := \bfI - \eta \bfx_t \otimes \bfx_t$.

    \mypar{Effect of the initialization} 
    We first calculate
    \begin{align*}
        \expect[\bfQ_t^2] = \bfI - 2\eta \bfH + \eta^2 \expect\left[\norm{\bfx_t}_2^2 \bfx_t \otimes \bfx_t \right]
        \preceq \bfI - 2\eta \bfH + \eta^2 R^2 \bfH
        \preceq \bfI - \eta \bfH
        \preceq (1 - \eta \mu) \bfI \,,
    \end{align*}
    where the first inequality follows from \eqref{eq:R2:def}, the second since $\eta \le 1/R^2$, and the third since $\bfH \succeq \mu \bfI$.
    Letting $\calF_t$ denote the sigma algebra generated by $\bfx_0, \ldots, \bfx_{t-1}$, we get
    \begin{align*}
        \expect\left[ \norm{\hat \bftheta_{t+1}}_2^2 \middle| \calF_t \right]
        &= \inp*{\hat \bftheta_t}{\expect[\bfQ_t^2] \hat\bftheta_t} \le (1 - \eta \mu) \norm{\hat \bftheta_t}_2^2
        \le \exp(-\eta \mu) \norm{\hat \bftheta_t}_2^2 \,. 
    \end{align*}
    Taking an unconditional expectation and unrolling this and using $\mu \bfI \preceq \bfH \preceq L \bfI$ (\Cref{asmp:input:conc}) gives
    \begin{align} \label{eq:pf-finite-time-1}
        \expect\norm{\hat \bftheta_t}_\bfH^2
        \le L \, \expect\norm{\hat \bftheta_t}_2^2 
        \le L \, \exp(-\eta\mu t) \, \norm{\bftheta_0'}_2^2 
        \le \frac{L}{\mu} \, \exp(-\eta\mu t) \, \norm{ \bftheta_0'}_\bfH^2 \,.
    \end{align}
    
    \mypar{Effect of the noise}
    Define $\tilde \bftheta_t' := \thetasgd_t + \thetadp_t$.
    We get from \Cref{lem:noise-cov-finite} that there exists a PSD matrix $\bfS_\infty$ such that 
    \[
        \boldzero = \expect\left[ \tilde \bftheta_0' \otimes  \tilde \bftheta_0' \right] \preceq \expect\left[ \tilde \bftheta_1' \otimes  \tilde \bftheta_1' \right] \preceq \cdots \preceq \lim_{t \to \infty} \expect\left[ \tilde \bftheta_t' \otimes  \tilde \bftheta_t' \right] =: \bfS_\infty \,.
    \]
    Multiplying by $\bfH$ and taking a trace, we get,
    \begin{align} \label{eq:pf-finite-time-2a}
        0 \le \expect\norm{\tilde \bftheta_0'}_\bfH^2 \le  \expect\norm{\tilde \bftheta_1'}_\bfH^2 
        \le \cdots 
        \le \lim_{t \to \infty}  \expect\norm{\tilde \bftheta_t'}_\bfH^2 = \tr{\bfH \bfS_\infty} \,.
    \end{align}
    Thus, $\tilde \bftheta_t = \tilde \bftheta_t' + \bftheta_\star$ 
    is a process that starts from $\tilde \bftheta_0 = \bftheta_\star$ and satisfies the conditions of \Cref{lem:noise-cov-finite}.
    This in turn gives
    \begin{align} \label{eq:pf-finite-time-2}
        0 \le \expect\left[ F(\tilde\bftheta_0) - F(\bftheta_\star)\right]
        \le \expect\left[ F(\tilde\bftheta_1) - F(\bftheta_\star)\right]
        \le \cdots \le  \lim_{t\to\infty} \expect\left[ F(\tilde\bftheta_t) - F(\bftheta_\star)\right] = \frac{1}{2}\tr{\bfH \bfS_\infty}\,,
    \end{align}
    which equals $F_\infty^\star(\bfbeta)$ by definition.
    
    \mypar{Combining both processes}
     From the triangle inequality of the norm $\bfu \mapsto \sqrt{\expect \norm{\bfu}_\bfH^2}$, we get
    \begin{align*}
        \sqrt{\expect{\norm{\bftheta_t'}_\bfH^2}}
        \le \sqrt{\expect{\norm{\hat \bftheta_t}_\bfH^2}} + \sqrt{\expect{\norm{\tilde \bftheta_t'}_\bfH^2}} \,.
    \end{align*}
    Plugging in \eqref{eq:pf-finite-time-1} and \eqref{eq:pf-finite-time-2a} gives
    \begin{align*}
        \sqrt{\expect\left[F(\bftheta_t) - F(\bftheta_\star)\right]}
        &\le \sqrt{\frac{L}{2\mu} \exp(-\eta\mu t) \norm{\hat \bftheta_0'}_\bfH^2 }
            + \sqrt{\frac{1}{2}\tr{\bfH \bfS_\infty}} \\
        &= \sqrt{\frac{L}{\mu} \exp(-\eta\mu t) \left( F(\bftheta_0) - F(\bftheta_\star)\right)}
            + \sqrt{F_\infty^\star(\bfbeta)} \,,
    \end{align*}
    where the last equality followed from \eqref{eq:pf-finite-time-2}.
    This establishes the required statement with $F_\infty^\star$ in place of $F^\infty$. Taking $t \to \infty$, we see that
    \[
        \sqrt{F_\infty(\bfbeta)}
        = \lim_{t \to \infty} \sqrt{\expect\left[F(\bftheta_t) - F(\bftheta_\star)\right]}
        = \sqrt{F_\infty^\star(\bfbeta)}\,,
    \]
    for any fixed $\eta$ or that $F_\infty = F_\infty^\star$ irrespective of $\bftheta_0$.
\end{proof}

\subsection{Privacy-Utility Guarantees of \dpftrl}
\label{sec:a:dp-guarantee}

We now state a general privacy-utility bound for \dpftrl in terms of the asymptotics of \noisyftrl run with the same parameters.

\begin{thm} \label{thm:dp-main}
    Fix a constant $0 < \failprob < 1$ and suppose the \Cref{asmp:linear-regression:conc} holds.
    Fix some
    noise coefficients $\bfbeta = (\beta_0, \ldots, \beta_{T-1})$
    that satisfy \hed with parameter $\eta\tilde\nu$ for some $\tilde\nu \le \mu$.
    Consider the sequence $(\bftheta_t)_{t=0}^{T-1}$ of iterates and the sequence $(\bfg_t)_{t=0}^{T-1}$ of gradients when running \dpftrl for $T$ iterations with noise coefficients $\bfbeta$,
    gradient clip norm
    $G = cR^2 \max\left\{ \norm{\bftheta_0 - \bftheta_\star}_2, \sqrt{\eta R^2 \sigmasgd^2 / \mu}, \sigmasgd/R\right\} \log^{5/2}\left( \frac{T}{\failprob}\right)$, and a learning rate
    \[
        \eta \le \min\left\{
        \frac{1}{CR^2 \log(T/\failprob)},
        \frac{\tilde \nu\rho}{8C^2 R^4 d \gamma_\infty^2  (\bfbeta) \norm{\bfbeta}_1^2 \log^5(T/\failprob)}
        \right\}\,,
    \]
    and
    DP noise $\bfw_{t} \sim \calN(\boldzero, \sigmadp^2 G^2 \bfI)$ with squared noise multiplier $\sigmadp^2 = \gamma(\bfbeta)^2 / (2\rho)$. Then, we have the following:
    \begin{enumerate}[label=(\alph*)]
        \item $(\bftheta_t)_{t=0}^T$ is $\rho$-zCDP.
        \item Let $\calE$ denote the event where no gradients are clipped, i.e, $\calE = \cap_{t=0}^{T-1}\{\norm{\bfg_t}_2 \le G\}$. We have, $\mathbb{P}(\calE) \ge 1-\failprob$.
        \item We have,
        \[
            \expect\left[ 
            \left( F(\bftheta_t) - F(\bftheta_\star) \right) \cdot \indi{\calE}
        \right]
        \le \frac{2L}{\mu} \exp( - \eta \mu t) \left( F(\bftheta_0) - F(\bftheta_\star) \right) + 2 \, \hat F_\infty(\bfbeta) \,,
        \]
        where $\hat F_\infty(\bfbeta)$ is the asymptotic suboptimality of \noisyftrl run with the same parameters.
    \end{enumerate}
\end{thm}
\begin{proof}
    Part (a) follows from \Cref{thm:dp}.
    For part (b), we bound the gradient norms from \Cref{thm:high-prob-bounds} as
    \begin{align*}
        \norm{\bfg_t}_2 &\le CR^2 \left(  \norm{\bftheta_0'}_2
            + \sqrt{\frac{\eta R^2 \sigmasgd^2}{\mu}} + \frac{\sigmasgd}{R} + G \sqrt{\frac{\eta \sigma^2 d \norm{\bfbeta}_1^2}{\tilde \nu}}
        \right) \log^{5/2}\left( \frac{T}{\failprob} \right) \\
        &\le C R^2\left( \norm{\bftheta_0'}_2
            + \sqrt{\frac{\eta R^2 \sigmasgd^2}{\mu}} + \frac{\sigmasgd}{R}
        \right) \log^{5/2}\left( \frac{T}{\failprob} \right) + \frac{G}{4}  \\
    &\le 4 \max\left\{
        CR^2  \max\left\{ 
            \norm{\bftheta_0'}_2, \sqrt{\frac{\eta R^2 \sigmasgd^2}{\mu}},
            \frac{\sigmasgd}{R}
        \right\} \log^{5/2} \left( \frac{T}{\failprob} \right), 
        \frac{G}{4} \right\} \le G
    \end{align*}
    where the second inequality follows from the condition on the learning rate and we take $c = 4C$ in the definition of $G$ for the last inequality.
    Thus, $\calE$ holds whenever the bound of \Cref{thm:high-prob-bounds} holds, so we have $\mathbb{P}(\calE) \ge 1- \failprob$.
    
    For part (c), consider the sequence $(\bfphi_t)_{t=0}^T$ produced by running \noisyftrl with $\bfphi_0 = \bftheta_0$ and the same realizations $(\bfx_t, \xi_t, \bfw_t)$ of random inputs, linear model noise, and DP noise. On $\calE$, we have that $\bfphi_t = \bftheta_t$ for all $t$. Thus, we have,
    \begin{align*}
         \expect\left[ 
            \left( F(\bftheta_t) - F(\bftheta_\star) \right) \cdot \indi{\calE}
        \right]
        = \expect\left[ 
            \left( F(\bfphi_t) - F(\bftheta_\star) \right) \cdot \indi{\calE}
        \right]
        \le \expect\left[F(\bfphi_t) - F(\bftheta_\star)
        \right] \,,
    \end{align*}
    since $\indi{\calE} \le 1$. This can now be bounded using \Cref{thm:lr-finite-noclip} to complete the proof.
\end{proof}

We can instantiate these rates for \privsgd and \dpftrl. Recall that we have $\kappa = L/\mu$, $\edim = \tr{\bfH} / L$, and $R^2 = \Theta(\tr{\bfH})$.

\begin{cor} \label{cor:dpsgd}
    Consider the setting of \Cref{thm:dp-main} with $T$ large enough that $T / \log^5(T/\failprob) \ge c \kappa^2 \edim^2 d / \rho$. The final suboptimality of DP-SGD at an appropriate choice of the learning rate is (ignoring absolute constants),
    \begin{align*}
    \begin{aligned}
        \expect\left[ \left(F(\bftheta_T) - F(\bftheta_\star)\right) \cdot \indi{\calE} \right]
        \le& \, \frac{L}{\mu} \exp\left(
            - \frac{\rho T}{c \kappa^2 \edim^2 d \log^5(T/\failprob)}
        \right) \\
        &+
        \kappa \, \edim  \left(
            \frac{d \tr{\bfH} \norm{\bftheta_0 - \bftheta_\star}_2^2}{\rho T}
             + \frac{ d \sigmasgd^2}{\rho T} + \frac{\sigmasgd^2}{T}
        \right)  \polylog{T} \,.
    \end{aligned}
    \end{align*}
\end{cor}
\begin{proof}
We plug in the asymptotic suboptimality bound of \noisysgd into the bound of \Cref{thm:dp-main}. We get two terms depending on the learning rate $\eta$: the first $\exp(-\eta \mu T)$ term and the second $O(\eta)$ term coming from the asymptotic suboptimality.
We balance both the terms subject to the maximum bound on $\eta$ using \Cref{lem:tune-lr} to get
    \begin{align*}
    \begin{aligned}
        \expect\left[ \left(F(\bftheta_T) - F(\bftheta_\star)\right) \cdot \indi{\calE} \right]
        \le& \, \frac{L}{\mu} \exp\left(
            - \frac{\rho \mu^2 T}{c R^4 d \log^5(T/\failprob)}
        \right) \\
        &+ \frac{\polylog{T}}{\mu T}
        \left(
            \frac{dR^4 \norm{\bftheta_0 - \bftheta_\star}_2^2}{\rho}
             + \frac{d \sigmasgd^2 R^2}{\rho} + \sigmasgd^2 R^2
        \right) \,.
    \end{aligned}
    \end{align*}
Rearranging the constants completes the proof.
\end{proof}

\begin{cor} \label{cor:dpftrl}
    Consider the setting of \Cref{thm:dp-main} with $T$ large enough that $T / \log^7(T/\failprob) \ge \frac{c \kappa^2 \edim^2 d}{\rho} \log\left(\frac{c \kappa^2 \edim^2 d}{\rho}\right)$. For \ourprivftrl with an appropriate choice of the parameter $\nu$ and learning rate $\eta$, we have (ignoring absolute constants),
    \begin{align*}
    \begin{aligned}
        \expect\left[ \left(F(\bftheta_T) - F(\bftheta_\star)\right) \cdot \indi{\calE} \right]
        \le& \, \frac{L}{\mu} \exp\left(
            - \frac{\rho T}{c \kappa^2 \edim^2 \, d \log^7(T/\failprob) \log(\kappa^2 \edim^2 \, d/{\rho})}
        \right) \\
        &+ 
        \kappa \edim  \left(
            \frac{\kappa \edim \tr{\bfH} \norm{\bftheta_0 - \bftheta_\star}_2^2}{\rho T^2}
             + \frac{\kappa \edim \sigmasgd^2 }{\rho T^2} + \frac{\sigmasgd^2}{T}
        \right) \polylog{T} \,.
    \end{aligned}
    \end{align*}
\end{cor}
\begin{proof}
We plug in the asymptotic error for \ournoisyftrl from \Cref{prop:dp-ftrl-bound} into \Cref{thm:dp-main} to get that
\begin{align} \label{eq:pf:dpftrl-dp:1}
        \expect\left[ \left(F(\bftheta_T) - F(\bftheta_\star)\right) \cdot \indi{\calE} \right]
        \le 
    \frac{L}{\mu} \exp(-\mu \eta T)
    + \eta \sigmasgd^2 R^2
    + \eta^2  \frac{R^2 G^2}{\rho} \log^2\frac{1}{\eta \mu} \,,
\end{align}
where $G^2$ is as given in the statement of \Cref{thm:dp-main}.
For our choice of $\bfbeta$, we have $\norm{\bfbeta}_1^2 \le 4$ always and $\gamma(\bfbeta)^2 \le 5 \log(1/\eta\mu)$ from \Cref{eq:tuned-dpftrl-pf2} (from the proof of \Cref{prop:dp-ftrl-bound}).
Thus, the largest learning rate permitted must satisfy
\[
    \eta \log^2\frac{1}{\eta\mu} \le \frac{\eta \rho}{c R^2 d \log^5(T/\failprob)} \,.
\]
From \Cref{lem:xlogsqx}, we can ensure with a more stringent condition
\[
    \eta \le \frac{\mu \rho}{cR^4 d \log^5(T/\failprob) \log^2(cR^4 d \log^(T/p) / (\mu^2 \rho))} \,.
\]
Finally, this is implied by imposing the requirement
\[
     \eta \le \frac{\mu \rho}{cR^4 d \log^7(T/\failprob) \log\left(\frac{R^4 d}{\mu^2 \rho}\right)} =: \eta_{\max} \,.
\]
We now tune $\eta$ to minimize the bound \eqref{eq:pf:dpftrl-dp:1} subject to $\eta \le \eta_{\max}$ using \Cref{lem:tune-lr}. Thus gives,
\begin{align*}
\begin{aligned}
    \expect\left[ \left(F(\bftheta_T) - F(\bftheta_\star)\right) \cdot \indi{\calE} \right]
    \le& \, \frac{L}{\mu} \exp\left(
        - \frac{\rho \mu^2 T}{c R^4 d \log^7(T/\failprob) \log\frac{R^4 d}{\rho \mu^2}}
    \right) \\
    &+ \frac{\polylog{T}}{\mu T}
    \left(
        \frac{R^6 \norm{\bftheta_0 - \bftheta_\star}_2^2}{\rho\mu T}
         + \frac{R^4 \sigmasgd^2 }{\rho \mu^2 T^2} + \sigmasgd^2 R^2
    \right) \,.
\end{aligned}
\end{align*}
Rewriting the constants completes the proof.
\end{proof}

\section{Proofs for General Strongly Convex Functions}\label{app:IQC}
We prove the results from \Cref{thm:iqc}. Under the assumptions of the theorem, clipping does not occur in \privftrl so the updates can be written as
\begin{align}
\bftheta_{t+1} = \bftheta_t - \eta\br{\br{\bfB\bfw}_t+\br{\bfg_t +  \hat{\bfw}_t}}    \label{eq:iqc_dynamics}
\end{align}
where 
\[
\bfg_t = \nabla F\br{\bftheta_t},
\quad 
\hat{\bfw}_t = \grad f\br{\bftheta_t;\bfz_t}-\ExP{\bfz \sim \Pdata}{\grad f\br{\bftheta_t;\bfz}}
\]
and $\hat{\bfw}_t$ is a random variable that, conditioned on $\bftheta_t$, is bounded by $\sigmasgd$ with probability 1. 
Below, $\id{d}$ denotes the $d \times d$ identity matrix.

\begin{theorem}\label{thm:iqc_app}

$\bfmulth=\{\multh_t\}_{t=-\infty}^{\infty}$ be such that $\multh_t \geq 0 \quad  \forall t \in \mathbb{Z}$, 
\[\sum_{t=-\infty}^{\infty} \multh_t \leq 2\multh_0\]
and let $\multH$ denote the Discrete-time Fourier transform (DTFT) of $\bfmulth$. Let
\begin{subequations}
\begin{align}
 &M_{\multh}\br{\omega} = \tran{\herm{A\br{\omega}}} \tilde{M}_{\multh}\br{\omega}A\br{\omega} 
 \label{eq:iqc_M_def}   \\
 &A\br{\omega} = \begin{pmatrix} \eta\id{d} &  0 \\
\br{1-\exp\br{i\omega}}\id{d} & -\eta\id{d}  \end{pmatrix}  \\
& \tilde{M}_{\multh}\br{\omega} = \begin{pmatrix} -\strongconvex\smoothconvex\br{\multH\br{\omega}+\herm{\multH\br{\omega}}}\id{d} & \strongconvex\multH\br{\omega}\id{d} + \smoothconvex\herm{\multH\br{\omega}}\id{d} \\
\strongconvex\herm{\multH}\br{\omega}\id{d} + \smoothconvex\multH\br{\omega}\id{d} &  -\br{\multH\br{\omega} + \herm{\multH\br{\omega}}}\id{d}\end{pmatrix}
\end{align}
\end{subequations}

Then, for any non-negative valued function $\kappafunc: [-\pi, \pi] \mapsto \mathbb{R}_+$  such that
\begin{align}
 M_{\multh}\br{\omega} \preceq \begin{pmatrix} -\eta^2 \id{d} & 0 \\ 0 & \kappafunc\br{\omega}\id{d} \end{pmatrix} \quad \forall \omega \in [-\pi, \pi] \label{eq:IQC_SDP}    
\end{align}
We have that
\[\lim_{t \to \infty} \ExP{}{\frac{\sum_{t=-T}^T \norm{\bftheta_t-\bftheta^\star}_2^2}{2T+1}} \leq  \frac{2d}{2\pi \eta^2}\int_{-\pi}^{\pi} \br{|B\br{\omega}|^2 G^2 \rho^{-1}\gamma_\infty^2\br{B} + \sigmasgd^2}\kappafunc\br{\omega}\D\omega \]
where $\Ssgd$ is the power spectral density of $\tilde{\bfw}$. In particular, if the density of $\bftheta_t$ converges to a stationary distribution, the expected value of 
\[\lim_{t \to \infty} \ExP{}{\norm{\bftheta_t-\bftheta^\star}_2^2}\] under  the stationary distribution is bounded as above.
\end{theorem}
\begin{proof}
We assume without loss of generality that $\grad F\br{0}=0$ so that the origin is the global optimum of $F$ (else we can translate the origin to achieve this). 
Since $\bfg=\grad F\br{\bftheta}$ satisfies
\[
\inner{\bfg - \smoothconvex \bftheta}{\strongconvex \bftheta - \bfg} \geq 0 \quad \forall \bftheta, \bfg \,.
\]

 Then, we can write down the following family of integral quadratic constraints relating $\bfg=\br{\ldots, \bfg_0, \bfg_1, \bfg_2, \ldots}$ and $\bftheta=\br{\ldots, \bftheta_0, \bftheta_1, \bftheta_2, \ldots}$ in terms of their Fourier transforms $\Theta\br{\omega}, G\br{\omega}$ (\citet{heath2005zames} Eq. 27-29):
\begin{align}\int_{-\pi}^{\pi} \herm{\begin{pmatrix} \Theta\br{\omega} \\ G\br{\omega}\end{pmatrix}} \begin{pmatrix} -\strongconvex\smoothconvex\br{\multH\br{\omega}+\herm{\multH\br{\omega}}}\id{d} & \strongconvex\br{\multH\br{\omega}}\id{d} + \smoothconvex\br{\herm{\multH\br{\omega}}}\id{d} \\
\strongconvex\br{\herm{\multH}\br{\omega}}\id{d} + \smoothconvex\br{\multH\br{\omega}}\id{d} &  -\br{\multH\br{\omega} + \herm{\multH\br{\omega}}}\id{d}\end{pmatrix}\begin{pmatrix} \Theta\br{\omega} \\ G\br{\omega}\end{pmatrix}\D\omega \geq 0 \,. \label{eq:iqc}
\end{align}
Noting that from \eqref{eq:iqc_dynamics}, we have that 
\[{\small \Theta\br{\omega}\br{\exp\br{i\omega}-1} = -\eta\br{G\br{\omega} + Z\br{\omega}} \implies G\br{\omega} = \br{\frac{1-\exp\br{\bfi\omega}}{\eta}}\Theta\br{\omega}-Z\br{\omega}}\]
where $Z$ denotes the DTFT of $\bfzeta=\bfB\bfw + \hat{\bfw}$. Plugging this into the above quadratic constraint and multiplying by $\eta^2$, we obtain
\begin{align}\int_{-\pi}^{\pi} \herm{\begin{pmatrix} \Theta\br{\omega} \\ Z\br{\omega}\end{pmatrix}}  M_{\multh}\br{\omega} \begin{pmatrix} \Theta\br{\omega} \\ Z\br{\omega}\end{pmatrix}\D\omega \geq 0\,. \label{eq:iqc_deriv}
\end{align}
Since 
$M_{\multh}\br{\omega} \preceq  \begin{pmatrix}  -\eta^2\id{d} & 0 \\ 0 & \kappafunc\br{\omega}\id{d} \end{pmatrix}$ 
we obtain that 
\begin{align}
& \int_{-\pi}^{\pi} \herm{\begin{pmatrix} \Theta\br{\omega} \\  Z\br{\omega}\end{pmatrix}} \begin{pmatrix} - \eta^2\id{d} & 0 \\ 0 & \kappafunc\br{\omega} \end{pmatrix} \begin{pmatrix} \Theta\br{\omega} \\ Z\br{\omega}\end{pmatrix}\D\omega \geq 0  \implies \frac{\ExP{}{\int_{-\pi}^{\pi} \norm{\Theta\br{\omega}}^2}}{\ExP{}{\int_{-\pi}^{\pi} \norm{\sqrt{\kappafunc\br{\omega}}Z\br{\omega}}^2}} \leq 1 \nonumber\\
& \implies \qquad \frac{\lim_{T \to \infty} \ExP{}{\frac{\sum_{t=-T}^{T} \norm{\bftheta_t}^2}{2T+1} }}{\lim_{T \to \infty} \ExP{}{ \frac{\sum_{t=-T}^{T} \norm{\sqrt{\kappafunc}[\bfzeta]\br{t}}^2}{2T+1}}} \leq \frac{1}{\eta^2} \nonumber
\end{align}
where $\sqrt{\bfzeta}[z]$ denotes the LTI operator with transfer function $\sqrt{\bfzeta\br{\omega}}$ applied to the signal $\bfzeta$. 

The denominator of the final line above is the power spectral density of $\sqrt{\kappa}[\bfzeta]$ (since $\sqrt{\kappa}[\bfzeta]$ is a wide-sense stationary stochastic process). By the Cauchy-Schwarz inequality for random variables, this is bounded above by
\[2d \br{|B\br{\omega}|^2\rho^{-1}\gamma_\infty^2\br{B}  + \sigmasgd^2}\kappafunc\br{\omega}\]
where the first term in brackets is the power spectral density of the Gaussian random process $\bfB\bfw$ and the second term is an upper bound on the power spectral density of $\hat{\bfw}$. Hence, by Theorem \ref{thm:lti-covariance}, we have the desired result.
\end{proof}

\subsection{Proof of \Cref{thm:iqc}}
Given the above theorem and smooth convexity parameter $\smooth$, we know that the asymptotic suboptimality $\Fobj$ is bounded above by 
\[ 
\frac{2Ld}{2\pi \eta^2}\int_{-\pi}^{\pi} \br{|B\br{\omega}|^2 \rho^{-1}\gamma_\infty^2\br{B}G^2 + \sigmasgd^2}\kappafunc\br{\omega}\D\omega \,.
\]
Now, the constraint \eqref{eq:IQC_SDP} can be rewritten as

\begin{align}& \begin{pmatrix} -\eta^2 & 0 \\ 0 & \kappafunc\br{\omega}\end{pmatrix} - \nonumber\\
& \quad \tran{\herm{\begin{pmatrix} \eta &  0 \\
1-\exp\br{i\omega} & -\eta  \end{pmatrix}}}\begin{pmatrix} -\strongconvex\smoothconvex\br{\multH\br{\omega}+\herm{\multH\br{\omega}}} & \strongconvex\multH\br{\omega} + \smoothconvex\herm{\multH\br{\omega}} \\
\strongconvex\herm{\multH}\br{\omega} + \smoothconvex\multH\br{\omega} &  -\br{\multH\br{\omega} + \herm{\multH\br{\omega}}}\end{pmatrix}\begin{pmatrix} \eta &  0 \\
1-\exp\br{i\omega} & -\eta  \end{pmatrix} \succeq 0 
\label{eq:int_psd}
\end{align}
since all the matrices involved are Hadamard products of the $2 \times 2$ matrices above and the identity matrix. 

Thus, for each $\omega$, $\kappafunc\br{\omega}$ must satisfy a $2 \times 2$ PSD constraint which can be rewritten as a Second Order Cone Program (SOCP) constraint. Furthermore, the constraint on $\multh$ from theorem \ref{thm:iqc_app} is a linear constraint. Since the projection of a convex set in $\kappafunc, \multh$ to $\kappafunc$ is convex, $\kappafunc$ belongs to a convex set. Furthermore, if we take $\multh$ such that $\multh_\tau=0$ for $|\tau| > T_{\max}$ for some $T_{\max} > 0$, the constraint on $\multh$ can be written as
\[
2\multh_0 \ge \sum_{\tau=-T_{\max}}^{T_{\max}} \multh_t \,.
\]

Further, if we discretize $\omega$ to a uniform grid on $[-\pi, \pi]$, the constraints \eqref{eq:int_psd} can be written as a finite collection of SOCP constraints linking $\kappafunc\br{\omega}$ and $\multh$. 

\section{Technical Definitions and Lemmas}
\label{sec:a:technical}

We review several relevant technical definitions and lemmas here:
\begin{itemize}
    \item \textbf{\Cref{sec:a:lti}}: Fourier Analysis of Linear Time-Invariant Systems.
    \item \textbf{\Cref{sec:a:sgd-cov}}: Stationary covariance of SGD.
    \item \textbf{\Cref{sec:a:concentration}}: {Concentration of Measure}.
    \item \textbf{\Cref{sec:a:ellip-def}}: Review of definitions and useful properties of elliptic integrals.
\end{itemize}

\subsection{Linear Time-Invariant (LTI) Systems}
\label{sec:a:lti}

We first review the definition and some useful properties of discrete-time Linear Time-Invariant (LTI) systems.
We refer to the textbook \cite{oppenheim1997signals} for a more detailed description.

\begin{defn}
An input-output system $\bfy_t = \calA_t(\bfx)$ with an input sequence $\bfx = (\bfx_t)_{t=-\infty}^\infty$ in some input space $\calX$ and an output sequence $(\bfy_t)_{t=-\infty}^\infty$ in an output space $\calY$ is said to be LTI if it satisfies two properties:
\begin{itemize}
    \item \textbf{Linearity}: For any $\calX$-valued sequences $\bfx^{(1)}, \bfx^{(2)}, \ldots$ and scalars $\alpha_1, \alpha_2, \ldots$, we have 
    \[
    \calA_t\left(\sum_{j=1}^\infty \alpha_j \bfx^{(j)}\right) = \sum_{j=1}^\infty \alpha_j \calA_t(\bfx^{(j)})\,.
    \]
    \item \textbf{Time-Invariance}: For any $t_0 \in \mathbb{Z}$, the sequence $\bfx'$ defined as $\bfx'_t := \bfx_{t-t_0}$ satisfies
    $\calA_t(\bfx') = \calA_{t-t_0}(\bfx)$.
\end{itemize}
\end{defn}
Throughout this paper, we consider LTI systems in the Euclidean space $\mathcal{X} = \reals^d$.

LTI systems can be viewed as linear operators defined on the Hilbert space of \emph{signals} in $\reals^d$:
\[
\ell_{2e}^d = \left\{ \br{\bfx_t}_{t=-\infty}^{\infty} \,: \, \quad
\bfx_t \in \reals^d \quad \text{and}\quad
\sum_{\tau=-t}^{t} \norm{\bfx_\tau}_2^2 < \infty \qquad \forall t \in \mathbb{Z} \right\} \,.
\]
We use the notation $\overrightarrow{\bfx} = \br{\bfx_t}_{t=-\infty}^{\infty} \in \ell_{2e}^d$ to denote an entire sequence.
The Hilbert space $\ell^d_{2e}$ is endowed with the inner product
$\inner{\overrightarrow{\bfx}}{\overrightarrow{\bfy}} = \sum_{t=-\infty}^{\infty} \tran{\bfx_t}{\bfy_t}$. 

\mypar{Asymptotic stability}
An LTI system is said to be asymptotically stable if its output decays to zero for any input sequence that is bounded, i.e., for which there exists $T > -\infty$ such that $\bfx_t = 0 \quad \forall t > T$. 

\mypar{LTI systems in 1D}
We highlight some key properties of LTI systems in $d=1$ dimension, i.e. $\mathcal{X} = \reals$. This conveys the key ideas before we describe the extension in higher dimensions.
LTI systems can be described in linear algebraic notation by the action of an infinite Toeplitz matrix $\bfH = \toeplitz(\bfh)$ (i.e., the first column of $\bfH$ is $\bfh$) on an element $\overrightarrow \bfx \in \ell_{2e}$:
\[ \overrightarrow \bfy = \bfH \overrightarrow \bfx \iff y_t = \sum_{\tau=-\infty}^{\infty} \bfH_{t, \tau} x_\tau = 
\br{\bfh \star \overrightarrow\bfx}_t \quad \forall t \in \mathbb{Z} 
\]
where $\star$ denotes the convolution operator.
This property is represented more elegantly in the Fourier domain. Consider the discrete-time Fourier transform (DTFT) $X : [-\pi, \pi] \to \mathbb{C}$ of $\overrightarrow \bfx$, defined by
\[
    X(\omega) = \sum_{t=-\infty}^\infty x_t \, \exp(-\I \omega t) \,.
\]
Similarly, let $Y(\omega)$ denote the DTFT
of $\overrightarrow \bfy$ 
and $G(\omega)$\footnote{
    The transfer function $G(\omega)$ here is not to be confused with the clip norm $G$ used in the rest of the manuscript; this section is a self-contained technical reference.
} 
denote the DTFT of $\bfh$. Then, we have $Y(\omega) = G(\omega) X(\omega)$.
Here, $\bfh$ is known as the \textbf{impulse response} and $G(\omega)$ is known as the \textbf{transfer function}.

\mypar{Multivariate LTI systems}
The previous concepts can be directly extended to higher dimensions and multivariate LTI systems admit a clean representation in the Fourier domain. 

Let $\bfx_t \in \reals^d$ be the input and $\bfy_t \in \reals^p$ be the output of an LTI system. 
The DTFT $\bfX(\omega) = \sum_{t=-\infty}^\infty \bfx_t \exp(-\I \omega t) \in \mathbb{C}^d$ outputs a $d$-dimensional complex vector, and $\bfY(\omega) \in \mathbb{C}^p$ similarly.

The transfer function $\bfG(\omega)$ in this case can be represented as a complex matrix in $\mathbb{C}^{p \times d}$. Similar to the scalar case, the Fourier domain description of this LTI system is given as $\bfY(\omega) = \bfG(\omega) \bfX(\omega)$, where the latter product is the standard matrix-vector product over complex numbers.

\mypar{Variance of LTI systems driven by white noise}
The Fourier-domain analysis of an LTI system (particularly its transfer function) helps us characterize the covariance of the output $\bfy_t$ as a function of the covariance of the input $\bfx_t$.
The following theorem presents the result for multivariate LTI systems driven by white noise.

\begin{theorem} \label{thm:lti-covariance}
    Consider an asymptotically-stable LTI system with $\reals^d$-valued inputs $(\bfx_t)_{t=-\infty}^\infty$ and $\reals^p$-valued outputs $(\bfy_t)_{-\infty}^\infty$  and a transfer function $\bfG(\omega) \in \mathbb{C}^{p \times d}$. Suppose that $\bfx_t$ is a stationary white noise sequence with covariance matrix $\bfSigma \in \reals^{d \times d}$, i.e., 
    $\expect[\bfx_t] = \boldzero$ and $\expect[\bfx_t \otimes \bfx_\tau] = \bfSigma$ if $t=\tau$ and $\boldzero_{d \times d}$ otherwise for all $t, \tau$. Then, we have for all $t > -\infty$ that
    \[
        \expect[\bfy_t \otimes \bfy_t]
        = \frac{1}{2\pi} \int_{-\pi}^\pi \bfG(\omega) \, \bfSigma \, \bfG(\omega)^* \, \D \omega
        \,.
    \]
\end{theorem}

\subsection{Stationary Covariance of Stochastic Gradient Descent for Linear Regression}
\label{sec:a:sgd-cov}

We now give a result characterizing the stationary covariance of SGD for linear regression~\cite{bach2013nonstrongly,defossez2015averaged,jain2017parallelizing,jain2017markov}.

\begin{theorem}[Lemma 5 of \cite{jain2017markov}]
\label{thm:fsttcs}
    Consider the recursion $\bfdelta_0 = \boldzero$ and
    \[
        \bfdelta_{t+1} = (\bfI - \eta\bfx_t \otimes \bfx_t) \, \bfdelta_t + \eta \bfzeta_t \,,
    \]
    for all $t \ge 0$ where
    \begin{itemize}
        \item $\bfx_t$ are i.i.d. with mean $\boldzero$, covariance $\bfH$, and
        \item $\bfzeta_t$ are i.i.d. with mean $\boldzero$, covariance $\expect[\bfzeta_t \otimes \bfzeta_t] \preceq \sigma^2 \bfH$.
    \end{itemize}  
    Further, if $\expect\left[\norm{\bfx_t}_2^2 \, (\bfx_t \otimes \bfx_t)\right] \preceq R^2 \bfH$ and $\eta < 1/R^2$, then we have for all $t\ge 0$.
    \[
        \expect[\bfdelta_t \otimes \bfdelta_t] \preceq \frac{\eta \sigma^2}{1 - \eta R^2} \, \bfI \,.
    \]
\end{theorem}

\subsection{Concentration of Measure}
\label{sec:a:concentration}

We recall the definition of sub-Gaussian random variables and list some useful concentration inequalities.

\begin{defn} \label{def:subgaussian}
    A real-valued random variable $X$ is said to be sub-Gaussian with variance proxy $\sigma^2$ if for all $\lambda \in \reals$, we have
    \[
        \expect[\exp(\lambda (X-\mu))] \le \exp(\lambda^2 \sigma^2/ 2) \,,
    \]
    where $\mu = \expect[X]$.
    If in addition, the variance of $X$ exactly equals $\sigma^2$, it
    is said to be \emph{strictly sub-Gaussian}.
\end{defn}

The cumulants of strict sub-Gaussian random variables are closely related to those of a Gaussian~\cite[Prop. 3.2]{arbel2020strict}.

\begin{property} \label{prop:strict-subgauss}
    If $X$ is strictly sub-Gaussian with mean zero and variance $\sigma^2$, we have
    $\expect[X^3] = 0$ and $\expect[X^4] \le 3\sigma^4 = \expect[Y^4]$ for $Y \sim \calN(0, \sigma^2)$.
\end{property}

Next, we state the \textbf{Hanson-Wright inequality} for the concentration of quadratic forms; see e.g. \cite{rudelson2013hansonwright}.

\begin{lem} \label{lem:hanson-wright}
    Let $\bfxi = (\xi_1, \ldots, \xi_d)$ be such that each $\xi_j$ is independent and sub-Gaussian with mean zero and variance proxy $\sigma^2$. Then, we have for any matrix $\bfA \in \reals^{d \times d}$, 
    \[
        \mathbb{P}( \inp{\bfxi}{\bfA\bfxi} - \expect[\inp{\bfxi}{\bfA\bfxi}] > t)
        \le \exp\left(
            -c \min\left\{ \frac{t^2}{\sigma^4 \norm{\bfA}_F^2}, \frac{t}{\sigma^2 \norm{\bfA}_2} \right\}
        \right) \,,
    \]
    for a universal constant $c$.
    Consequently, for any $\rho < 1/3$ and symmetric PSD matrix $\bfA$, we have with probability $1-\rho$ that
    \[
        \inp{\bfxi}{\bfA \bfxi} \le C \sigma^2 \left(\tr{\bfA} \sqrt{\log\frac{1}{\rho}}
            + \norm{\bfA}_2 \log\frac{1}{\rho} \right)
        \le C' \sigma^2 \tr{\bfA} \log\frac{1}{\rho} \,,
    \]
    for universal constants $C, C'$.
\end{lem}
The second part follows from the first one under the simplifications $\norm{\bfA}_2 \le \norm{\bfA}_F \le \tr{\bfA}$ and $\expect[\inp{\bfxi}{\bfA\bfxi}] \le \sigma^2 \tr{\bfA}$ for $\bfA$ PSD.

\begin{remark}
    Explicit values for the constant $c$ in \Cref{lem:hanson-wright} (and thus for $C, C'$) are known for the case when $\xi_1, \ldots, \xi_d \sim \calN(0, \sigma^2)$:
    $c \approx 0.1457 \ge 1/8$, $C \le 8$, $C' \le 16$~\cite{moshksar2021absolute}.
\end{remark}

\subsection{Review of Elliptic Integrals}
\label{sec:a:ellip-def}

We recall some definitions and useful properties of elliptic integrals.
We refer to \cite[\S19]{NIST:DLMF} and \cite{byrd2013handbook} for details.

The three canonical elliptic integral forms are:
\begin{enumerate}[label=(\roman*), leftmargin=\widthof{(abcd)}]
    \item The complete elliptic integral of the first kind $K : (0, 1) \to [0, \infty)$ is
        \begin{align} \label{eq:ellipk}
            K(k) := \int_{0}^{\pi / 2} \frac{ \D \omega}{\sqrt{1 - k^2 \sin^2(\omega)}} \,.
        \end{align}
    \item The complete elliptic integral of the second kind $E : (0, 1) \to [0, \infty)$ is
        \begin{align} \label{eq:ellipe}
            E(k) := \int_{0}^{\pi / 2} {\sqrt{1 - k^2 \sin^2(\omega)}} \,\, { \D \omega} \,.
        \end{align}
    \item The complete elliptic integral of the third kind $\Pi: (\reals \setminus \{\pm 1\}) \times (0, 1) \to \reals$\ is denoted conventionally as $\Pi(\alpha^2, k)$ where $\alpha^2$ is allowed to take negative values. It is defined as
    \begin{align} \label{eq:ellippi}
        \Pi(\alpha^2, k) := \int_{0}^{\pi/2} \frac{\D \omega}{(1 - \alpha^2\sin^2(\omega)) \sqrt{1 - k^2 \sin^2(\omega)}} \,.
    \end{align}
\end{enumerate}

The corresponding integrals where $1 - k^2 \sin^2(\omega)$ is replaced with $1 + k^2 \sin^2(\omega)$ can also be expressed using the elliptic integrals~\cite[Eq. (19.7.2), (19.7.5)]{NIST:DLMF}.

\begin{property}\label{prop:ellipk:im}
    For any $m \in (0, 1)$, we have
    \begin{align}
        \int_{0}^{\pi/2} \frac{\D \omega}{\sqrt{1 + m \sin^2(\omega)}} 
        = \frac{1}{\sqrt{1 + m}} \,\, K \left( \sqrt{\frac{m}{1+m}} \right) \,.
    \end{align}
\end{property}

\begin{property}\label{prop:ellippi:im}
    For any $m \in (0, 1)$ and any $\alpha^2 \in \reals \setminus \{\pm 1\}$ such that $\alpha^2 + m \neq 0$, we have
    \begin{align}
    \begin{aligned}
        \int_{0}^{\pi/2} & \frac{\D \omega}{(1 - \alpha^2 \sin^2(\omega)) \sqrt{1 + m \sin^2(\omega)}}  \\
        &= \frac{m}{(m + \alpha^2) \sqrt{1 + m}} \,\, K \left( \sqrt{\frac{m}{1+m}} \right) 
        + \frac{\alpha^2}{(m + \alpha^2) \sqrt{1 + m}}
        \Pi\left( \frac{m + \alpha^2}{1 + m} \, ,  \sqrt{\frac{m}{1+m}} \right)
        \,.
    \end{aligned}
    \end{align}
\end{property}

The next few properties are about the asymptotics of the elliptic integrals; see e.g. \cite[Eq. (19.9.1)]{NIST:DLMF} for $K(\cdot)$ and \cite[Eq. (19.12.4)]{NIST:DLMF} for $\Pi$.

\begin{property} \label{prop:ellipk:asymp}
    For all $k \in (0, 1)$, we have
    \[
        \log \left( \frac{4}{\sqrt{1-k^2}} \right)
        \le K(k)
        \le 
        \left(1 + \frac{1-k^2}{4}\right) \, \log \left( \frac{4}{\sqrt{1-k^2}} \right) \le \frac{5}{4} \log \left( \frac{4}{\sqrt{1-k^2}} \right) \,.
    \]
\end{property}

\begin{property} \label{prop:ellippi:asymp}
    For all $k, \alpha^2 \in (0, 1)$, we have
    \[
        \Pi(\alpha^2, k) \le \frac{1}{1-\alpha^2} \, \log\left( \frac{4}{\sqrt{1-k^2}}\right)
        \left(1 + O\left(\sqrt{1-k^2}\right) \right)
         \,.
    \]
\end{property}

\subsection{Useful Integrals}

We list several useful definite integrals in this section.

\mypar{Direct Evaluation}
The first one is a cosine integral divided by a quadratic form.\footnote{See \url{https://math.stackexchange.com/a/816253}.}

\begin{lem} \label{lem:cos-integral}
    For reals $0 < |b| < a$ and an integer $l$, we have
    \[
        \int_{-\pi}^\pi  \frac{ \cos(l \omega) \D \omega}{a^2 + b^2 - 2ab \cos\omega}
        = \frac{2\pi}{a^2 - b^2}  \left(\frac{b}{a}\right)^{|l|} \,.
    \]
\end{lem}

The next lemma is also about rational cosine functions.\footnote{See \url{https://math.stackexchange.com/a/1235309}.}

\begin{lem} \label{lem:cosine-integral-2}
    For scalar $a$, we have
    \[
        \int_{-\pi}^\pi \frac{\D \omega}{1 + a \cos(\omega)}
        = \begin{cases}
            \frac{2\pi}{1- a^2}, & \text{ if } |a| < 1, \\
            +\infty, & \text{ if } |a| = 1 \,.
        \end{cases}
    \]
\end{lem}

The next one is similar to the previous one.
\begin{lem} \label{lem:cosine-integral-3}
    We have that 
    \[
        \int_{-\pi}^\pi \frac{\D\omega}{\sqrt{1- \cos(\omega)}} = + \infty \,.
    \]
\end{lem}
\begin{proof}
    We successively deduce
    \begin{align*}
        \int_{-\pi}^\pi \frac{\D\omega}{\sqrt{1- \cos(\omega)}}
        &= 
        \frac{1}{\sqrt{2}} \int_{-\pi}^\pi \frac{\D\omega}{|\sin(\omega/2)|}
        = 2\sqrt{2} \, \int_{0}^{\pi/2} \frac{\D\omega}{\sin(\omega)} = +\infty \,,
    \end{align*}
    where we used that $\int \D\omega / \sin(\omega) = -\log | \csc(\omega) + \cot(\omega)| + C$.
\end{proof}

\mypar{Reductions to Elliptic Integrals}
We now list several cosine integrals that can be reduced to elliptic integrals (see \Cref{sec:a:ellip-def} for their definitions).

\begin{lem} \label{lem:sensitivity-ellip}
    For any $a \in (0, 1)$, we have
    \begin{align}
        \int_{-\pi}^\pi \frac{\D \omega}{|1 - a - \exp(\I \omega)|}
        = \frac{4}{2 - a} \, K\left( \frac{\sqrt{1-a}}{1 - a/2}  \right) \,,
    \end{align}
    where $K(\cdot)$ is the complete elliptic integral of the first kind, cf. \eqref{eq:ellipk}.
\end{lem}
\begin{proof}
    Using $\cos(\omega) = 1 - 2\sin^2(\omega/2)$ and the substitution $\omega' = \omega/2$, we successively deduce
    \begin{align*}
        \int_{-\pi}^\pi \frac{\D \omega}{|1 - a - \exp(\I \omega)|}
        &= 
        2 \, \int_{0}^\pi \frac{\D \omega}{\sqrt{1 + (1 - a)^2 - 2(1-a) \cos(\omega)}}  \\
        &= 
        2 \, \int_{0}^\pi \frac{\D \omega}{\sqrt{a^2 + 4(1-a) \sin^2(\omega/2)}} \\
        &= \frac{4}{a} \int_{0}^{\pi/2} \frac{\D \omega'}{\sqrt{1 + 4 \left(\frac{1-a}{a^2}\right) \sin^2(\omega')}} \,.
    \end{align*}
    Applying \Cref{prop:ellipk:im} to reduce this to the standard elliptic integral completes the proof.
\end{proof}

The next lemma handles a more general case.
Note that it recovers \Cref{lem:sensitivity-ellip} when $a=b$ since $\Pi(0, k) = K(k)$ by definition.
\begin{lem} \label{lem:error-ellip}
    For any $a, b \in (0, 1)$, we have
    \begin{align}
        \int_{-\pi}^\pi \frac{|1 - a - \exp(\I \omega)|}{|1 - b - \exp(\I \omega)|^2} \, \D\omega 
        = \frac{2 a^2}{b^2(1 - a/2)} \,\,
        \Pi\left(
        \frac{b^2(1-a) - a^2(1-b)}{b^2(1 - a/2)^2}\,,\,\,
        \frac{\sqrt{1-a}}{1 - a/2}
        \right) \,,
    \end{align}
    where $\Pi$ is the complete elliptic integral of the third kind, cf. \eqref{eq:ellippi}.
\end{lem}
\begin{proof}
    We assume that $a \neq b$ to begin and handle the case of $a=b$ by continuity. Denote $h(a, \omega) = \sqrt{1 + (1 - a)^2 - 2(1-a) \cos(\omega)}$
    \begin{align*}
        \int_{-\pi}^\pi \frac{|1 - a - \exp(\I \omega)|}{|1 - b - \exp(\I \omega)|^2} \, \D\omega 
        &= \int_{-\pi}^\pi \frac{|1 - a - \exp(\I \omega)|^2}{|1 - a - \exp(\I \omega)| \, |1 - b - \exp(\I \omega)|^2} \, \D\omega  \\
        & = \frac{1 + (1 - a)^2}{h(a, \omega) \, h(b, \omega)^2} - 2(1-a) \frac{\cos(\omega)}{h(a, \omega) \, h(b, \omega)^2} \,.
    \end{align*}
    We next add and subtract terms to make the numerator of the second term read $h(b, \omega)^2$ to give
    \begin{align} \label{eq:pf-Pi-int2-1}
        \int_{-\pi}^\pi \frac{|1 - a - \exp(\I \omega)|}{|1 - b - \exp(\I \omega)|^2} \, \D\omega 
        &= \int_{-\pi}^\pi \frac{1 + (1-a)^2 - \frac{1-a}{1-b}\left( 1 + (1 -b)^2\right)}{h(a, \omega)\, h(b, \omega)^2} \, \D\omega
        + \frac{1-a}{1-b} \int_{-\pi}^\pi \frac{\D\omega}{h(a, \omega)} \,.
    \end{align}
    From \Cref{lem:sensitivity-ellip}, the second term above can be written as
    \begin{align} \label{eq:pf-Pi-int2-2}
        \frac{1-a}{1-b} \int_{-\pi}^\pi \frac{\D\omega}{h(a, \omega)}
        = \frac{4(1-a)}{(1 - b)(2-a)} \, K\left( \frac{\sqrt{1-a}}{1 - a/2} \right) \,.
    \end{align}
    The first term of \eqref{eq:pf-Pi-int2-1} can similarly be reduced to the elliptic integral form with
    $\cos(\omega) = 1 - 2 \sin^2(\omega/2)$ and the substitution $\omega' = \omega/2$ as
    \begin{align*}
         \int_{-\pi}^\pi \frac{\D \omega}{h(a, \omega)\, h(b, \omega)^2} 
         &= \frac{2}{ab^2} \int_{0}^\pi \frac{\D \omega}{\sqrt{1 + \frac{4(1-a)}{a^2} \sin^2(\omega/2)} \left(
         1 + \frac{4(1-b)}{b^2} \sin^2(\omega/2)
         \right)} \\
         &= \frac{4}{ab^2} \int_{0}^{\pi/2} \frac{\D \omega'}{\sqrt{1 + \frac{4(1-a)}{a^2} \sin^2(\omega')} \left(
         1 + \frac{4(1-b)}{b^2} \sin^2(\omega')
         \right)} \,.
    \end{align*}
    This can be written in terms of elliptic integrals using \Cref{prop:ellippi:im} as
    \begin{align} \label{eq:pf-Pi-int2-3}
    \begin{aligned}
        \int_{0}^{\pi/2} &\frac{\D \omega'}{\sqrt{1 + \frac{4(1-a)}{a^2} \sin^2(\omega')} \left(
         1 + \frac{4(1-b)}{b^2} \sin^2(\omega')
         \right)} \\
        &= \frac{a}{2-a}\left(\frac{b^2(1 -a)}{b^2(1-a) - a^2(1-b)} \right) K(k)
        - \frac{a^3(1-b)}{(2-a)(b^2(1-a) - a^2(1-b))} \,
        \Pi(\alpha^2, k) \,,
    \end{aligned}
    \end{align}
    with $k = \sqrt{1-a}/ (1 - a/2)$ and 
    \[
        \alpha^2 = \frac{b^2(1-a) - a^2(1-b)}{b^2(1 - a/2)^2} \,.
    \]
    Plugging in \eqref{eq:pf-Pi-int2-2} and \eqref{eq:pf-Pi-int2-3} into \eqref{eq:pf-Pi-int2-1}, we find that the $K(\cdot)$ term cancels out, completing the proof.
\end{proof}

\subsection{Other Helper Results}
\label{sec:a:misc}

We list several other miscellaneous useful results.

\begin{lem} \label{lem:helper-l2-beta}
    For a sequence $\bfbeta = (\beta_0, \beta_1, \ldots) \in \ell^2$ and a constant $0 \le c < 1$, we have
    \[
        \sum_{t=0}^\infty \sum_{\tau = 0}^\infty
        \beta_t \beta_\tau c^{|t - \tau|}
        \le \left( \frac{1+c}{1-c}\right) \norm{\bfbeta}_2^2 \,.
    \]
\end{lem}
\begin{proof}
    We break the sum into powers of $c$ and use the Cauchy-Schwarz inequality $(*)$ to get
    \begin{align*}
         \sum_{t=0}^\infty \sum_{\tau = 0}^\infty
        \beta_t \beta_\tau c^{|t - \tau|}
        &= \norm{\bfbeta}_2^2
        + 2 \sum_{k=1}^\infty c^k \left( \sum_{t=0}^\infty  \beta_t \beta_{t+k} \right) \\
        &\stackrel{(*)}{\le} \norm{\bfbeta}_2^2 + 2 \sum_{k=1}^\infty c^k \norm{\bfbeta}_2^2  \,.
    \end{align*}
    Summing up the geometric series with a multiplier $0 \le c < 1$ completes the proof.
\end{proof}

\begin{lem} \label{lem:fourth-order-contraction}
    Consider a random vector $\bfx$ that satisfies $\expect[\bfx]=0$, $\expect[\bfx \otimes \bfx] = \bfH \succeq \mu I$ for some $\mu > 0$ and $\expect\left[\norm{\bfx}_2^2 \bfx\otimes \bfx \right] \preceq R^2 \bfH$. Then, we have for all $\eta \le 1/R^2$ and all PSD matrices $\bfM$ that
    \[
        \tr{(\bfI - \eta \bfx \otimes \bfx) \bfM (\bfI - \eta \bfx \otimes \bfx)} \le (1 - \eta \mu) \tr{\bfM} \,.
    \]
\end{lem}
\begin{proof}
    The left side above (call it ``LHS'') is bounded by
    \begin{align*}
        \text{LHS}
        &= \tr{\bfM} - 2 \eta \tr{\bfM \bfM}
         + \eta^2 \tr{\expect\left[\norm{\bfx}_2^2 \bfx \otimes \bfx\right] \bfM} \\
        &\le \tr{\bfM} - 2 \eta \tr{\bfH \bfM}
         + \eta^2 R^2 \tr{\bfH \bfM}  \\
        &\le \tr{\bfM} - \eta \tr{\bfH \bfM} \\
        &\le (1 - \eta \mu) \tr{\bfM} \,,
    \end{align*}
    where we used (a) $\expect\left[\norm{\bfx}_2^2 \bfx \otimes \bfx\right] \preceq R^2 \bfH$, (b) $\eta \le 1/R^2$, and (c) $\bfH \succeq \mu \bfI$.
\end{proof}

\begin{lem} \label{lem:holder-schatten}
    For PSD matrices $\boldzero \preceq \bfA_1, \ldots, \bfA_k \preceq \bfI$ of shape $d \times d$, we have $|\tr{\bfA_1 \cdots \bfA_k}| \le d$.
\end{lem}
\begin{proof}
    Recall the inner product $\inp{\bfA}{\bfB} = \tr{\bfA \bfB\T}$ on the space of real $d \times d$ matrices. Using H\"older's inequality on the Schatten $p$-norms, we get
    \begin{align*}
        |\tr{\bfA_1\ldots \bfA_k}|
        = |\inp{\bfA_1}{\bfA_k\cdots \bfA_2}|
        \le \norm{\bfA_1}_{S_1} \, \norm{\bfA_k\cdots, \bfA_2}_{S_\infty} \,.
    \end{align*}
    Here, the Schatten 1-norm $\norm{\cdot}_{S_1}$ is the $\ell_1$ norm of the singular values (i.e. the nuclear norm); this is just the trace for a PSD matrix. Thus,
    \[
        \norm{\bfA_1}_{S_1} = \tr{\bfA_1} \le \tr{\bfI} = 1 \,.
    \]
    The $\norm{\cdot}_{S_\infty}$ is the $\ell_\infty$ norm of the singular values, i.e. the operator norm $\norm{\cdot}_2$. We get,
    \[
        \norm{\bfA_k\cdots \bfA_2}_{2} \le \norm{\bfA_k}_2 \cdots \norm{\bfA_2}_2 \le 1\,.
    \]
\end{proof}

\begin{lem} \label{lem:tau-opt}
    For some fixed integer $t \ge 1$ and constants $a > 0$, $\rho \in (0, 1)$, define the function
    \[
        f(\tau) = \tau + \frac{1}{\rho a} \exp(-a \tau) \, \indi{\tau < t-1} \,.
    \]
    For $\hat\tau = \min\{t-1, a^{-1}\log(1/\rho)\}$, we have,
    \[
        f(\hat \tau) = \min\left\{ t-1, \frac{1}{a}(1 + \log(1/\rho)) \right\} 
        \le \frac{1}{a}(1 + \log(1/\rho)) \,.
    \]
\end{lem}
\begin{proof}
    The convex function $\tau \mapsto \tau + \frac{1}{\rho a} \exp(-a \tau)$ is minimized at $\tau_\star = a^{-1} \log(1/\rho) > 0$ with a minimum value of $a^{-1}(1 + \log(1/\rho))$. If $t-1 \le \hat \tau_\star$, we take $\hat \tau = t-1$ and $f(\hat \tau) = t-1 \le \hat\tau \le a^{-1}(1 + \log(1/\rho))$.
\end{proof}

The next lemma is from \cite[Lemma 13]{pillutla2023federated}.

\begin{lem} \label{lem:tune-lr}
    Consider a function $\varphi: [0, \eta_{\max}] \to \reals_+$ given by
    \[
        \varphi(\eta) = A \exp(- \mu \eta T) + B \eta + C \eta^2 \log^2\left( \frac{1}{\eta\mu}\right) \,,
    \]
    given some constants $\eta_{\max}, \mu, A, B, C > 0$. If $T \ge (\mu \eta_{\max})^{-1}$, then we have
    \[ 
    \varphi(\eta_\star)
        \le 
        A \exp(-\mu \eta_{\max} T)
        + \frac{3 B}{\mu T}\left(1 \vee \log \frac{A \mu T}{B}\right)
        + \frac{3C}{\mu^2 T^2} \left(1 \vee \log\frac{A \mu^2 T^2}{C}\right)^2 \log^2(T) \,,
    \]
    for some $\eta_\star \le \eta_{\max}$ depending on $A, B, C, \mu, T$.
\end{lem}

\begin{lem} \label{lem:xlogsqx}
    For $0 < c < 1/4$, we have,
    \[
    0 < x \le \frac{c}{9 \log^2(9/c)}
    \quad\implies \quad
    x \log^2(1/x) \le c \,.
    \]
\end{lem} 
\section{Empirical Details}
\label{sec:b:empirical-setup}

We train image-classification models using the CIFAR10 dataset and language models using the Stack Overflow Next Word Prediction (SONWP) dataset available on \texttt{tensorflow-datasets}.

\subsection{Image classification} 
Image classification has long been studied in DP ML. For example, the original DP-SGD work of~\citet{DP-DL} focused on this task. We use CIFAR10 which has 50,000 training and 10,000 test examples. We evaluate and compute test accuracies on the entire test set, following the open-sourced code of \citet{kairouz2021practical}. We reuse the network architecture, dataset processing, and initialization strategies presented in \citet{kairouz2021practical}; in particular, the architecture we use can be found in their Table 2 (b).

\mypar{Setup and Tuning}
We train all mechanisms for 2000 steps using a batch size of 500 and a clip norm of 1. This leads to ML training dynamics of 20 epochs and 100 steps per epoch. We performed some initial small grid searches which showed nearly ubiquitously that momentum of 0.95 (searched over the grid $0, 0.85, 0.9, 0.95$) and a linear learning rate cooldown $0.05\times$ the initial learning rate over the last 500 steps of training improved model utility for all privacy levels. Thus, we fix these settings for all mechanisms except DP-SGD, for which no momentum performed best.
For each mechanism, we then run a tuning grid search for the learning rate on coefficients in \{1, 2, 5\} on powers in [-2, 3], selecting the best mechanism for each privacy level from this interval.
Final experiments are repeated 12 times in each setting and show 95\% bootstrapped confidence intervals. 

Some mechanisms include additional hyperparameters that specify the exact mechanism's structure. For example, ME is specified by both the number of steps $n$ and the max number of participations $k$. We include such parameters in the mechanism name. For all mechanisms, $n=2000$. 

\subsection{Language modeling} 

Language modeling has been prominently studied in user-level DP contexts, usually in conjunction with federated learning~\cite[e.g.][]{mcmahan2017learning}.
DP training is important for real-world applications of language models trained on user data as these models can memorize their training data if appropriate mitigations are not applied~\cite{carlini2019secret,carlini2021extracting,carlini2022quantifying,ippolito2022preventing,anil2023palm,kudugunta2023madlad}. Indeed, DP already plays an important role in this application, as evidenced by Google's use of DP for training on-device language models~\citep{mcmahan2022federated,xu2023federated}. StackOverflow Next Word Prediction contains over $10^8$ examples contributed non-identically from 342,477 users. The goal of this task is to predict the next word given a sequence of words. We use the same setup as~\citet{choquette2023multi}.

\mypar{Setup and Tuning}
We consider a version of \dpftrl that works with ``generalized gradients'', i.e., the client update resulting from multiple local gradient steps on a client's data; this is a common strategy to ``lift'' learning algorithms to the federated learning setting~\cite{FLO}. We refer to \cite{reddi20adaptive} for details.
All mechanisms use an $\ell_2$ clip norm of 1, a server momentum of 0.95, and a client learning rate of 1.0. They also use a server learning rate cool-down over the last 25\% rounds. Initial tuning showed that these were favorable parameter settings. We train all mechanisms for 2052 steps and report the final evaluation accuracy of the model as reported on a held-out set of $10,000$ examples. We zero out large updates whose $\ell_\infty$ norm exceeds $100$. We use the tuned server learning rates from \citet{choquette2023multi} for all existing mechanisms. For the proposed \ourprivftrl mechanisms, we do not perform extensive tuning due to computational costs and instead 
tune the parameter to minimize the $\ell_2$ error \eqref{eq:prefix-error} of the total noise added due to $\bfB$~\citep[cf.][Figure 11]{choquette2023amplified}.

\end{document}